\documentclass{article}

\PassOptionsToPackage{numbers,compress}{natbib}
\usepackage[preprint]{neurips_2026}

\usepackage[utf8]{inputenc}
\usepackage[T1]{fontenc}

\usepackage{pifont}
\usepackage[table]{xcolor}
\usepackage{hyperref}
\hypersetup{
    colorlinks=true,
    linkcolor=blue,
    citecolor=blue,
    urlcolor=blue
}
\newcommand{\bluelink}[2]{\hyperref[#1]{#2}}
\usepackage{url}
\usepackage{amsmath}
\usepackage{amssymb}
\usepackage{amsfonts}
\usepackage{amsthm}
\usepackage{nicefrac}

\usepackage{graphicx}
\usepackage{array}
\usepackage{float}
\usepackage{booktabs}
\usepackage{makecell}
\usepackage{tabularx}
\usepackage{caption}
\usepackage{subcaption}
\usepackage{colortbl}

\usepackage{algorithm}
\usepackage{algpseudocode}

\usepackage{tikz}
\usetikzlibrary{positioning,arrows.meta}

\usepackage{microtype}
\usepackage{pifont}
\usepackage{soul}
\usepackage{blindtext}
\usepackage{etoc}

\sethlcolor{blue!15}

\definecolor{myblue}{RGB}{0,70,190}
\definecolor{lightrow}{RGB}{236,240,248}
\definecolor{bestrow}{RGB}{210,225,248}
\definecolor{tableblue}{RGB}{220,235,250}
\definecolor{headerblue}{RGB}{180,210,240}
\definecolor{best}{RGB}{220,255,220}
\definecolor{alggray}{RGB}{235,235,235}

\newcommand{\cmark}{\textcolor{green!60!black}{\ding{51}}}
\newcommand{\xmark}{\textcolor{red!75!black}{\ding{55}}}
\newcommand{\pmark}{\textcolor{orange!85!black}{\textbf{P}}}
\newcommand{\contrib}[1]{\textcolor{myblue}{\textbf{#1}}}

\newtheorem{theorem}{Theorem}

\newtheorem{definition}{Definition}

\newtheorem{assumption}{Assumption}

\title{Controlla: Learning Controllability via Graph-Constrained Latent Geometry}

\author{%
  Jamuna S.~Murthy \\
  Ramaiah Institute of Technology \\
  \And
  Amin Karimi Monsefi \\
  The Ohio State University \\
  \And
  Rajiv Ramnath \\
  The Ohio State University \\
}

\begin{document}

\maketitle

\begin{abstract}
Controllable multimodal generation is commonly formulated as an inference-time conditioning problem using prompts, guidance, or auxiliary modules. While effective, such approaches do not explicitly structure how semantic attributes evolve, which can lead to identity drift and inconsistent cross-modal behavior. We propose \textbf{Controlla}, a modular factorized-control framework that treats controllability as a property of structured latent geometry. Controlla learns identity and attribute factors from multimodal inputs and aligns them with graph priors using graph-constrained optimal transport, encouraging attributes to follow graph-consistent trajectories while preserving reference identity. To evaluate this setting, we construct \textbf{AffectHuman-43K}, a leakage-aware multimodal benchmark for reference-grounded affective control, and introduce geometry-aware metrics for trajectory consistency and latent disentanglement. Experiments show consistent improvements in controllability, identity preservation, and cross-modal alignment, with additional analyses on graph sensitivity, extensibility, and robustness. \url{https://jamunasmurthy.github.io/controlla-affecthuman/}
\end{abstract}
\vspace{-6pt}
\section{Introduction}
\label{sec:intro}
\vspace{-6pt}
Recent advances in generative modeling, particularly diffusion models \cite{rombach2022high,saharia2022imagen,podell2023sdxlimprovinglatentdiffusion,monsefi2025taxadiffusion,khurana2026taxaadapter} and diffusion transformers \cite{drobyshev2024emoportraits,zhang2023dreamtalk,peebles2023scalablediffusionmodelstransformers,chen2023pixartalphafasttrainingdiffusion,monsefi2025fs,meyarian2026direct}, together with multimodal representation learning frameworks such as CLIP \cite{radford2021learning} and ImageBind \cite{girdhar2023imagebind}, have enabled high-fidelity and cross-modal image generation. 
Despite this progress, \emph{controllability} remains fundamentally under-constrained. Current models can generate realistic outputs, but they often lack an explicit mechanism for representing how semantic attributes should change while preserving identity and multimodal consistency~\cite{patashnik2021styleclip,kim2022diffusionclip}.

A central challenge is \emph{identity-preserving multimodal emotional editing}: how to manipulate attributes such as emotion using multimodal inputs (text and audio) while preserving identity and maintaining cross-modal consistency. For instance, given an image of a person with a neutral expression, a text prompt describing ``excitement,'' and an audio signal conveying laughter, existing methods often distort identity or produce inconsistent expressions across modalities. These failures arise because semantic changes are imposed externally, rather than emerging from structured representations.

Current approaches address controllability through heuristic conditioning—via prompts, guidance, or auxiliary modules 
\cite{mokady2023nulltext,brooks2023instructpix2pix,zhang2023adding,
li2024controlnet++,bhat2024loosecontrol,zhou2025bidedpo,navard2024knobgen}. 
Recent advances extend this paradigm through instruction-based editing and in-context generation 
\cite{sheynin2024emu,zhang2025enabling,mao2025ace++,zhao2024ultraedit,zhang2023magicbrush,labs2025flux}, 
as well as trajectory-level control in diffusion and rectified flow models 
\cite{patel2025flowchef,wang2024taming,esser2024scalingrectifiedflowtransformers}. 
Large multimodal foundation models such as GPT-4o \cite{openai2024gpt4ocard}, Gemini \cite{comanici2025gemini25pushingfrontier}, and Kosmos-2 \cite{peng2023kosmos2groundingmultimodallarge} learn strong joint representations but do not explicitly enforce structured geometry during traversal, often leading to unstable intermediate control.
Similarly, practical pipelines combining personalization (e.g., DreamBooth \cite{ruiz2023dreambooth}) with control modules (e.g., ControlNet \cite{zhang2023adding}) remain loosely coupled, resulting in identity drift and inconsistent attribute transitions under continuous edits.
While effective, these methods treat controllability as an inference-time problem over implicit latent spaces, without enforcing structure on the representations, leading to semantic drift, weak identity preservation, and inconsistent multimodal alignment.

We instead view controllability as latent geometry: the structure of directions in a model's latent space where semantic attributes can change while identity remains stable. Semantic edits are modeled as structured latent transformations $z \rightarrow z'$, where attributes evolve predictably and identity is locally preserved. These edits correspond to graph-consistent traversal over a structured semantic manifold \cite{yu2025probability}, enabling interpretable multimodal transformations.

\begin{figure}[t]
  \includegraphics[width=\textwidth]{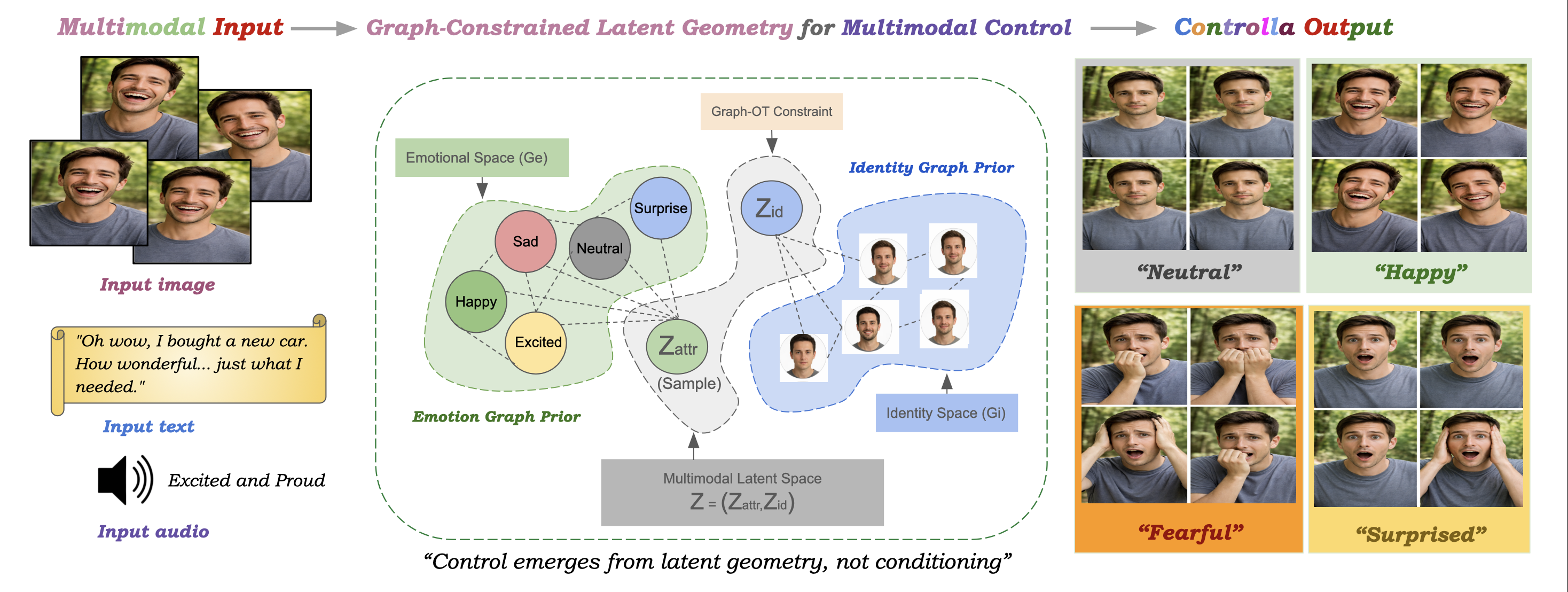}
 \caption{\textbf{Controllability as structured latent geometry.} Controlla maps image, reference image, text, and audio into factorized identity and attribute spaces. Attribute factors follow graph-consistent semantic traversal, while the reference-grounded identity factor is held stable.}
  \label{fig:teaser}
\end{figure}

To realize this idea, we introduce \textbf{Controlla}, a framework that learns a graph-structured latent space for multimodal control, as illustrated in Figure~\ref{fig:teaser}.
Multimodal inputs are mapped to shared representations and factorized into identity and attribute components. Graph-constrained optimal transport \cite{xu2019gromov,maretic2019gw,peyre2019computational} aligns these factors with identity and affective graph priors, encouraging reference-identity stability while attributes follow graph-consistent semantic trajectories.
The affective graph captures both the normalized 8-class emotion taxonomy and fine-grained valence--arousal variation, so traversal is framed as structured semantic control rather than unconstrained interpolation.

We further introduce \textbf{AffectHuman-43K}, a leakage-aware identity-disjoint benchmark with no cross-split reference-identity overlap, together with geometry-aware metrics for trajectory smoothness, graph consistency, and identity--attribute disentanglement; Table~\ref{tab:dataset_comparison} compares it with existing datasets. Unlike prior optimal-transport and graph-based methods that mainly align representations or distributions, Controlla uses graph-constrained optimal transport to organize the latent directions along which generation is controlled.

\begin{table*}[t]
\centering
\scriptsize
\renewcommand{\arraystretch}{0.95}

\begin{minipage}[t]{0.64\textwidth}
\centering
\caption{\textbf{AffectHuman-43K benchmark overview.} Comparison with existing datasets highlighting multimodal coverage, identity-aware splits, and variation.}
\label{tab:dataset_comparison}
\resizebox{\linewidth}{!}{%
\begin{tabular}{lccccccc}
\toprule
\textbf{Dataset} & \textbf{Img} & \textbf{Aud} & \textbf{Txt} & \textbf{ID} & \textbf{Cross} & \textbf{Scale} & \textbf{Var.} \\
\midrule
\rowcolor{gray!10}
AffectNet \cite{Mollahosseini_2019} & \cmark & \xmark & \xmark & \xmark & \xmark & 450K & Low \\
RAF-DB \cite{li2017reliable} & \cmark & \xmark & \xmark & \xmark & \xmark & 30K & Low \\
\rowcolor{gray!10}
CelebA-HQ \cite{karras2018progressivegrowinggansimproved,liu2015deep} & \cmark & \xmark & \xmark & \xmark & \xmark & 30K & Low \\
IEMOCAP \cite{busso2008iemocap} & \xmark & \cmark & \cmark & \xmark & \pmark & 12K & Med \\
\rowcolor{gray!10}
EmoBank \cite{buechel2017emobank} & \xmark & \xmark & \cmark & \xmark & \xmark & 10K & Med \\
RAVDESS \cite{livingstone2018ryerson} & \xmark & \cmark & \xmark & \xmark & \xmark & 7K & Low \\
\midrule
\rowcolor{gray!5}
MagicBrush \cite{li2023magicbrush} & \cmark & \xmark & \cmark & \xmark & \xmark & 10K & Med \\
UltraEdit \cite{zhao2024ultraedit} & \cmark & \xmark & \cmark & \xmark & \xmark & 4M & High \\
\rowcolor{gray!5}
Emu Edit \cite{sheynin2024emu} & \cmark & \xmark & \cmark & \xmark & \xmark & 10M & High \\
KontextBench \cite{labs2025flux} & \cmark & \xmark & \cmark & \xmark & \xmark & 1K & Med \\
\midrule
\rowcolor{blue!15}
\textbf{AffectHuman-43K} & \textbf{\cmark} & \textbf{\cmark} & \textbf{\cmark} & \textbf{\cmark} & \textbf{\cmark} & \textbf{43K} & \textbf{High} \\
\bottomrule
\end{tabular}%
}
\end{minipage}
\hfill
\begin{minipage}[t]{0.30\textwidth}
\centering
\caption{\textbf{Leakage-safe split statistics} of the AffectHuman-43K benchmark.}
\label{tab:split_stats}
\resizebox{\linewidth}{!}{%
\begin{tabular}{lc}
\toprule
\rowcolor{gray!20}
\textbf{Statistic} & \textbf{Value} \\
\midrule
\rowcolor{gray!10}
Aligned candidate samples & 43,514 \\
Final usable samples & 42,469 \\
\rowcolor{gray!10}
Train / Val / Test & 29,728 / 6,371 / 6,370 \\
Reference-identity groups & 26,241 \\
\rowcolor{gray!10}
Repeated ref.-identity groups & 10,204 \\
Singleton ref.-identity groups & 16,037 \\
\rowcolor{gray!10}
Multi-emotion repeated groups & 7,290 \\
Cross-split ref.-ID leakage & \textcolor{green!60!black}{\textbf{0}} \\
\rowcolor{gray!10}
Modality coverage & \textcolor{green!60!black}{\textbf{100\%}} \\
\bottomrule
\end{tabular}%
}
\end{minipage}
\end{table*}

\noindent\textbf{Contributions.}
\noindent\textbf{(1)} We propose a formulation of controllable multimodal generation in which \contrib{controllability is encouraged by graph-consistent latent transformations}, rather than inference-time conditioning alone.

\noindent\textbf{(2)} We introduce \contrib{Controlla, a modular factorized-control architecture} with a multimodal factorization encoder, graph-OT alignment module, and factorized generator-conditioning adapters for identity-preserving semantic traversal.

\noindent\textbf{(3)} We construct \contrib{AffectHuman-43K}, a leakage-aware multimodal benchmark, introduce geometry-aware metrics for structured controllability, and demonstrate consistent gains in controllability, identity preservation, and cross-modal alignment.

\section{Related Works}
\label{sec:background}

\paragraph{Controllable Generation and Editing.}
Diffusion models have become the dominant paradigm for controllable image generation, enabling high-fidelity synthesis under flexible conditioning, including large-scale latent diffusion systems such as SDXL \cite{podell2023sdxlimprovinglatentdiffusion} and diffusion transformers \cite{peebles2023scalablediffusionmodelstransformers,chen2023pixartalphafasttrainingdiffusion}. Early approaches rely on external guidance such as CLIP-based optimization and inversion techniques \cite{patashnik2021styleclip,mokady2023nulltext,kim2022diffusionclip}, while later methods manipulate cross-attention or instructions, including Prompt-to-Prompt \cite{hertz2022prompt}, InstructPix2Pix \cite{brooks2023instructpix2pix}, and plug-and-play editing \cite{tumanyan2023plug}. Structural conditioning methods such as ControlNet \cite{zhang2023adding} and its extensions \cite{li2024controlnet++,bhat2024loosecontrol} introduce explicit spatial controls, while recent work addresses conflicts between textual and conditional inputs via preference optimization \cite{zhou2025bidedpo}. Personalization approaches such as DreamBooth \cite{ruiz2023dreambooth}, DiffusionRig \cite{ding2023diffusionrig}, and related methods improve identity preservation, and emotion-aware models \cite{yang2024emogen,yang2025emoedit,drobyshev2024emoportraits,dang2025emoticraftertexttoemotionalimagegenerationbased} extend controllability to affective domains. However, these methods treat control as externally imposed over implicit latent spaces, often leading to semantic drift and inconsistent identity preservation.

Instruction-based editing further expands controllability through natural language interfaces. Models such as Emu Edit \cite{sheynin2024emu}, ICEdit \cite{zhang2025enabling}, and ACE++ \cite{mao2025ace++} enable multi-task and instruction-following editing, supported by large-scale datasets such as UltraEdit \cite{zhao2024ultraedit} and MagicBrush \cite{zhang2023magicbrush}. Unified generative systems such as FLUX.1 Kontext \cite{greenberg2025demystifying,labs2025flux} combine generation and editing within a single architecture, achieving strong multi-turn consistency. Complementary work explores control at the level of sampling and trajectory dynamics, including rectified flow and transformer-based generative models \cite{patel2025flowchef,wang2024taming,esser2024scalingrectifiedflowtransformers}. While RF-Solver and RF-Edit \cite{wang2024taming} improve inversion and editing in flow-based models, these approaches operate through conditioning, instruction following, or trajectory-level manipulation, without explicitly structuring the latent representation space or enforcing geometry-aware transformations between identity and semantic attributes.

\paragraph{Multimodal Learning, Structure, and Optimal Transport.}
Learning shared representations across modalities has been widely studied in vision-language models such as CLIP \cite{radford2021learning} and extended to unified multimodal embeddings in ImageBind \cite{girdhar2023imagebind}, as well as large-scale multimodal systems such as Flamingo \cite{alayrac2022flamingo}, PaLI \cite{chen2022pali}, and Kosmos-2 \cite{peng2023kosmos2groundingmultimodallarge}. Recent multimodal generative frameworks \cite{rojas2025diffuse} enable joint modeling across modalities, but do not explicitly constrain controllable transformations. Existing datasets for affective and multimodal understanding \cite{Mollahosseini_2019,li2017reliable,karras2017progressive,liu2015deep,busso2008iemocap,buechel2017emobank,livingstone2018ryerson} are typically unimodal or lack identity-disjoint evaluation, limiting their suitability for controllable generation.

Graph-based inductive biases and optimal transport provide principled tools for modeling structured relationships. Prior work explores relational reasoning \cite{kipf2018neural}, scene graph generation \cite{johnson2018image}, and graph-based multimodal learning \cite{xie2024graphbasedunsuperviseddisentangledrepresentation}, while optimal transport methods \cite{villani2008optimal,peyre2019computational,cuturi2013sinkhorn,xu2019gromov,maretic2019gw,vayer2020fused} enable alignment of distributions and relational structures. However, these approaches are typically used for local alignment or representation matching, rather than explicitly shaping the global geometry of latent spaces for controllable generation. 

\noindent \textit{\textbf{
While prior approaches primarily treat controllability as an inference-time conditioning or alignment problem, we instead model controllability as a property of global latent geometry, where semantic transformations follow graph-consistent trajectories while preserving identity.
}}

\section{Method}
\label{sec:method}

We present \textbf{Controlla}, a framework that models controllability as structured latent traversal under identity constraints.
Given multimodal inputs, Controlla learns a shared representation and factorizes it into attribute and reference-grounded identity components. Control is performed by moving the attribute factor along graph-consistent semantic trajectories while keeping the identity factor stable.
To induce this behavior, Controlla aligns the attribute and identity spaces with affective and reference-identity graph priors using graph-constrained optimal transport. This imposes relational structure on the latent space, reducing semantic drift and unintended identity changes without relying only on inference-time conditioning. As illustrated in Figure~\ref{fig:method}, Controlla combines a multimodal factorization encoder, graph-OT latent alignment, and factorized conditioning adapters around a frozen pretrained diffusion backbone.
\noindent\textbf{Multimodal Latent Representation.}
Given multimodal inputs $(x^{img}, x^{ref}, x^{txt}, x^{aud})$, where $x^{ref}$ specifies the identity to preserve, we learn a shared representation:
\begin{equation}
z = f_\theta(x^{img}, x^{ref}, x^{txt}, x^{aud}) \in \mathbb{R}^{d},
\end{equation}
where $f_\theta$ consists of modality-specific encoders, lightweight projection adapters, and a fusion module. We learn attribute and identity factors:
\begin{equation}
z_{\text{attr}} = h_{\text{attr}}(z), 
\qquad
z_{\text{id}} = h_{\text{id}}(z).
\end{equation}
Here, $z_{\text{attr}}$ captures semantic variations such as emotion, while $z_{\text{id}}$ encodes the reference-grounded identity from $x^{ref}$. This factorization enables attribute control while preserving identity. The factors are learned through trainable projection heads, with disentanglement encouraged by orthogonality and structural constraints.

\begin{figure*}[t]
\centering
\includegraphics[width=\textwidth]{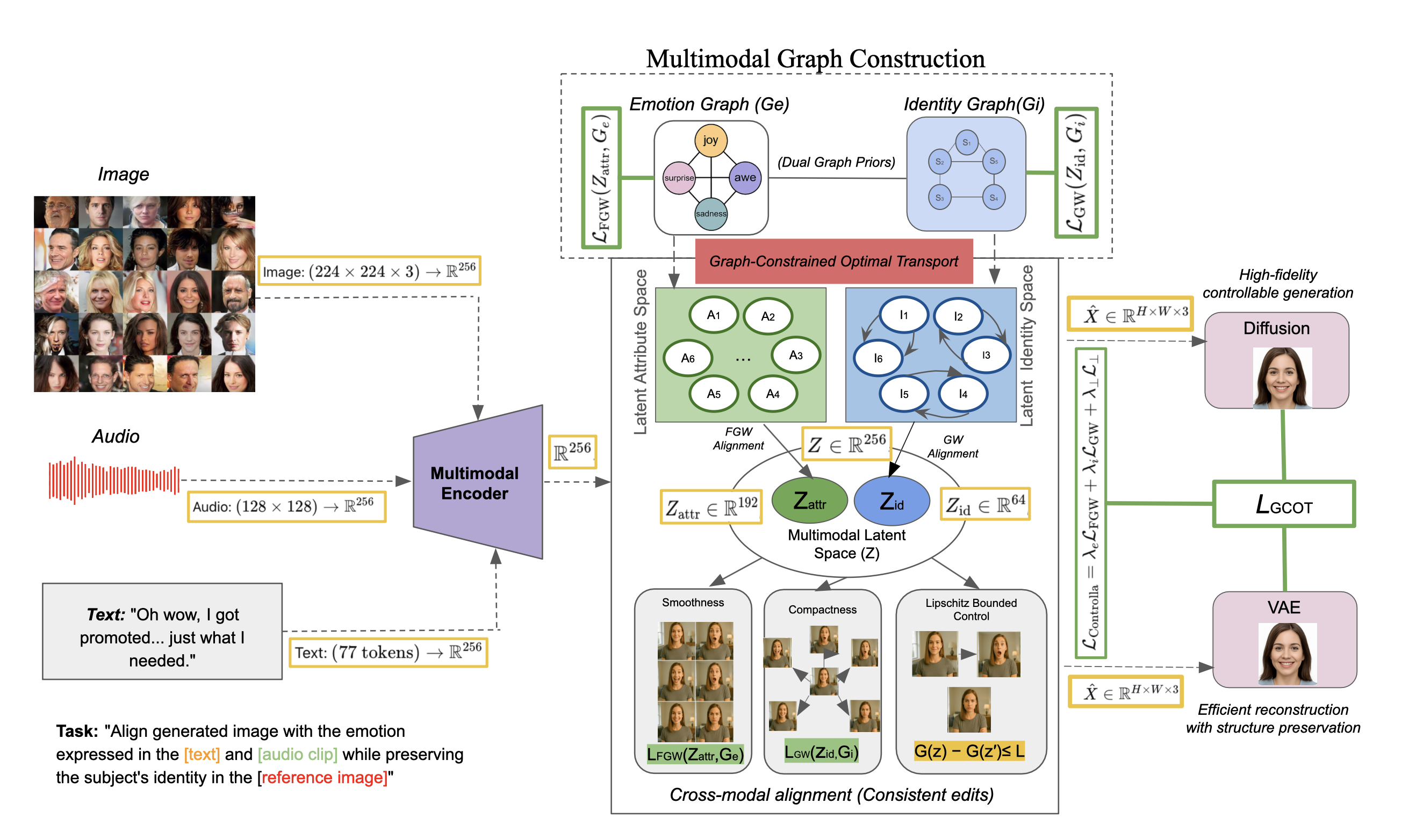}
\caption{
\textbf{Controlla framework.}
Multimodal inputs (image, reference image, text, audio) are encoded into a shared representation and factorized into attribute and identity components. Emotion and identity graph priors are aligned with these factors using graph-constrained optimal transport, enabling graph-consistent attribute traversal while preserving reference identity.
}
\label{fig:method}
\end{figure*}

\noindent\textbf{Graph Priors as Latent Geometry.}
We impose global structure using two graphs constructed offline: an emotion graph $\mathcal{G}_e$ capturing semantic relationships between attributes, and an identity graph $\mathcal{G}_i$ encoding relationships between reference identities. 

The emotion graph is constructed over a unified valence--arousal space, where edges reflect semantic proximity between affective states. The identity graph is derived from reference-identity embeddings, encouraging compact intra-identity structure while preserving inter-identity separation. These graphs define relational distance matrices $D_e$ and $D_i$ that serve as geometric priors. Edges are constructed using a k-nearest neighbor graph ($k=10$) with cosine similarity in the embedding space, with edge weights proportional to similarity.

These priors define a latent geometry in which semantic transformations correspond to graph-consistent traversal, enabling structured and predictable control.

\noindent\textbf{Graph-Constrained Optimal Transport.}
We align latent representations with graph structure using optimal transport. Intuitively, distances in latent space should reflect distances on the corresponding graph priors.

For attribute alignment, we use \emph{Fused Gromov--Wasserstein (FGW)}, which enforces that \emph{relative distances} between latent attribute embeddings match those in the emotion graph:
\begin{equation}
\mathcal{L}_{\text{FGW}} = \min_{\pi} 
\sum_{i,j,k,l} \left| D_e(i,j) - D_z(k,l) \right|^2 \pi_{ik}\pi_{jl} 
+ \alpha \sum_{i,k} C_{ik}\pi_{ik}.
\end{equation}
Here, $\pi$ is a transport plan, and $D_z$ denotes pairwise distances in latent attribute space. The first term aligns relational structure, while the second preserves feature similarity. As a result, semantic variations in $z_{\text{attr}}$ follow the geometry of $\mathcal{G}_e$, making attribute changes structured and consistent.

For identity preservation, we use \emph{Gromov--Wasserstein (GW)}, which aligns relational structure without feature matching:
\begin{equation}
\mathcal{L}_{\text{GW}}(z_{\text{id}}, \mathcal{G}_i).
\end{equation}
This ensures that reference-identity representations maintain consistent relational structure, preventing identity drift under transformations.

Together, FGW and GW induce a latent space where semantic changes correspond to structured movement along graph-consistent directions rather than arbitrary interpolation.

\noindent\textbf{Disentanglement and Geometry Regularization.}
To separate identity and attributes, we impose:
\begin{equation}
\mathcal{L}_{\perp} = \| z_{\text{attr}}^\top z_{\text{id}} \|_2,
\end{equation}
which encourages orthogonality between learned factors.

We further enforce smooth transformations using a Lipschitz regularization:
\begin{equation}
\mathcal{L}_{L} = \mathbb{E}_{z,z'} \left[ \max\left(0, \| g_\phi(z) - g_\phi(z') \| - L \| z - z' \| \right) \right],
\end{equation}
ensuring that small latent changes lead to controlled output changes. Intuitively, FGW aligns the attribute space with the semantic emotion graph, while GW stabilizes the reference-identity space, thereby preserving identity under attribute traversal.

\noindent\textbf{Learning and Controllable Generation.}
The overall objective combines structural alignment, identity preservation, disentanglement, and smoothness:
\begin{equation}
\mathcal{L}_{\text{Controlla}} =
\lambda_e \mathcal{L}_{\text{FGW}} +
\lambda_i \mathcal{L}_{\text{GW}} +
\lambda_{\perp} \mathcal{L}_{\perp} +
\lambda_L \mathcal{L}_{L}.
\end{equation}

A generative model $g_\phi$, implemented as a latent diffusion model \cite{rombach2022high}, produces outputs:
\begin{equation}
\hat{x} = g_\phi(z_{\text{attr}}, z_{\text{id}}).
\end{equation}

Controllability is achieved by traversing $z_{\text{attr}}$ along graph-consistent paths on $\mathcal{G}_e$ while keeping the reference-grounded identity factor $z_{\text{id}}$ fixed. This yields identity-preserving and semantically coherent transformations, where control is induced by structured latent geometry rather than external conditioning alone.
\textit{\textbf{See Appendix Sec. \hyperref[sec1]{~A} for theoretical analysis and proofs.}}
\vspace{-6pt}
\section{Experiments}
\label{sec:experiments}
\vspace{-6pt}
We evaluate whether \textbf{structured latent geometry} supports controllability as a learned property of representation space. 
In contrast to prior works that treat controllability primarily as inference-time conditioning, our experiments test three hypotheses: 
(i) controllability appears as smooth and predictable latent traversal, 
(ii) identity is preserved under large semantic transformations, and 
(iii) multimodal signals remain consistently aligned in the learned latent space.
\subsection{AffectHuman-43K: Benchmark for Multimodal Controllability}

To evaluate whether structured latent geometry enables identity-preserving multimodal control, we use \textbf{AffectHuman-43K}, a leakage-aware benchmark designed as an \textbf{emotion-aligned controlled evaluation protocol}, not a naturally co-recorded multimodal identity dataset. AffectHuman-43K is therefore not intended to model person-level correspondence across image, text, and audio; instead, identity is visually specified by $x^{ref}$, while text and audio provide affective control signals. This design deliberately decouples identity preservation from affective multimodal control. The task requires generating outputs that (i) preserve identity, (ii) follow target emotional transformations, and (iii) remain consistent with multimodal affective inputs. AffectHuman-43K contains \textbf{42,469 samples} across images, reference images, text, and audio, organized into \textbf{26,241 reference-identity groups}. Each sample is represented as $(x^{img}, x^{ref}, x^{text}, x^{aud}, y^{id}, y^{emo})$, enabling evaluation under heterogeneous multimodal signals. Cross-modal consistency is enforced in a shared semantic emotion space rather than through identity correspondence. Table~\ref{tab:dataset_comparison} compares AffectHuman-43K with existing datasets; ID denotes identity-aware evaluation, Cross denotes cross-modal coverage, and Var. denotes affective/control variation. The dataset spans 8 emotion classes with fine-grained intra-class variation and repeated reference-identity groups for consistency evaluation. We construct \emph{reference-identity-disjoint splits} with \textbf{0\% cross-split reference-identity leakage}, no duplicate overlap, and no cross-dataset near-duplicate overlap (Table~\ref{tab:split_stats}), ensuring unseen-reference generalization. Unlike existing datasets, AffectHuman-43K enables controlled evaluation of identity preservation and multimodal controllability without conflating identity alignment across modalities. \textit{\textbf{See Appendix Sec. \hyperref[sec2]{~B} and Sec. \hyperref[sec3]{~C} for detailed construction, alignment pipeline, and statistical analysis.}}
\vspace{-6pt}
\subsection{Evaluation Metrics}
\vspace{-6pt}
A central premise of Controlla is that controllability should be evaluated not only by output quality, but also by whether the learned latent space supports structured attribute change, identity invariance, and multimodal coherence. We therefore evaluate three complementary aspects: endpoint controllability, identity preservation, and latent-geometry structure.
\vspace{-6pt}
\paragraph{Primary Metrics.}
Emotion Accuracy (Acc) measures whether generated outputs match the target affective class, while Trajectory Smoothness (TS) measures the consistency of latent transitions and is normalized to $[0,1]$, where higher is smoother. Identity preservation is measured using face-region identity similarity (ID-Sim) and verification AUC under large affective transformations. Human preference (H) provides a perceptual measure of whether outputs better preserve identity, follow the target emotion, and remain visually plausible.
\vspace{-6pt}
\paragraph{Geometry-aware and Auxiliary Metrics.}
Latent Disentanglement Score (LDS) measures separation between learned identity and attribute factors, while Geodesic Consistency (GC) measures deviation from graph-consistent latent trajectories. These are primary diagnostic metrics for the latent-geometry claim. CLIP similarity measures image--text alignment, and ImageBind similarity (IB) measures image--audio alignment; both are treated only as auxiliary diagnostics because related encoders are used during benchmark filtering. Thus, our main controllability claims rely on Acc, ID-Sim/AUC, TS, LDS, GC, and H, which are independent of CLIP/ImageBind filtering.
\textit{\textbf{See Appendix Secs.~\hyperref[sec4]{D}, \hyperref[sec5]{E}, and \hyperref[sec6]{F} for detailed definitions, human evaluation, and significance analyses.}}
\vspace{-6pt}
\subsection{Implementation Details.}
We use a shared latent representation of dimension $d=256$ and learn attribute and identity factors through trainable projection heads. Image features are extracted using CLIP ViT-B/32 \cite{radford2021learning}; text and audio embeddings use pretrained Transformer and wav2vec2 encoders, fused by learned adapters. Controlla fine-tunes the projection adapters and factorization heads on top of a pretrained Stable Diffusion v1.5 latent-diffusion backbone initialized from the public checkpoint, while keeping the base generator frozen unless otherwise specified. All methods use standardized face crops and the same reference image and target emotion; for baselines without native audio input, audio is mapped to the normalized emotion label and inserted as text/instruction control. Baselines are evaluated under their strongest supported adaptation protocol, with prompts standardized across FLUX~\cite{labs2025flux}, ICEdit~\cite{zhang2025enabling}, and ControlNet variants~\cite{zhang2023adding,li2024controlnet++}. Identity preservation is measured on detected generated face regions. Training uses AdamW for 100k steps on 8$\times$L40S GPUs; graph-constrained OT adds $\sim$8\% cached training overhead with negligible inference cost. \textit{\textbf{See Appendix Sec.\hyperref[sec7]{~G} for additional implementation details, architecture configurations, and training hyperparameters.}}
\vspace{-6pt}
\subsection{Baselines}

We compare Controlla with the baselines used in our evaluation tables: StyleCLIP~\cite{patashnik2021styleclip}, ControlNet~\cite{zhang2023adding}, ICEdit~\cite{zhang2025enabling}, FLUX.1 Kontext~\cite{labs2025flux}, DreamBooth~\cite{ruiz2023dreambooth}, and the hybrid DBooth + ControlNet++ pipeline~\cite{ruiz2023dreambooth,li2024controlnet++}. 
For architecture-level comparisons, we additionally report SDXL \cite{podell2023sdxlimprovinglatentdiffusion} and SDXL + ControlNet variants. 
For cross-modal retrieval evaluation, we compare against CLIP~\cite{radford2021learning} and ImageBind~\cite{girdhar2023imagebind}.
These baselines cover prompt-based editing, spatial conditioning, instruction-guided editing, personalization, hybrid identity-preserving control, and multimodal alignment.


\begin{table*}[t]
\centering
\scriptsize
\setlength{\tabcolsep}{1.8pt}
\renewcommand{\arraystretch}{0.88}

\begin{minipage}[t]{0.48\linewidth}
\centering
\caption{
\textbf{Main results on AffectHuman-43K.}}
\begin{tabular}{@{}l@{\hspace{3pt}}cccc@{\hspace{6pt}}cccc@{}}
\toprule
& \multicolumn{4}{c}{\cellcolor{gray!20}\textbf{AffectHuman-43K(Val)}}
& \multicolumn{4}{c}{\cellcolor{gray!20}\textbf{AffectHuman-43K(Test)}} \\
\textbf{Method}
& Acc$\uparrow$ & TS$\uparrow$ & IB$\uparrow$ & H$\uparrow$
& Acc$\uparrow$ & TS$\uparrow$ & IB$\uparrow$ & H$\uparrow$ \\
\midrule

StyleCLIP~\cite{patashnik2021styleclip}
& 61.8 & 0.51 & 0.214 & 2.71
& 60.3 & 0.49 & 0.208 & 2.64 \\

ControlNet~\cite{zhang2023adding}
& 67.4 & 0.58 & 0.229 & 2.91
& 65.8 & 0.56 & 0.222 & 2.82 \\

\midrule
ICEdit~\cite{zhang2025enabling}
& 73.1 & 0.66 & 0.262 & 3.40
& 71.8 & 0.64 & 0.256 & 3.32 \\

FLUX.1 Kontext~\cite{labs2025flux}
& 74.3 & 0.67 & 0.266 & 3.52
& 73.0 & 0.65 & 0.259 & 3.43 \\

\midrule
DBooth + CNet++~\cite{ruiz2023dreambooth,li2024controlnet++}
& 75.1 & 0.67 & 0.271 & 3.58
& 73.6 & 0.65 & 0.264 & 3.47 \\

\midrule
\rowcolor{bestrow}
\textbf{Controlla}
& \textbf{77.6} & \textbf{0.73} & \textbf{0.286} & \textbf{3.91}
& \textbf{76.4} & \textbf{0.71} & \textbf{0.279} & \textbf{3.82} \\
\bottomrule
\end{tabular}
\label{tab:main_results}
\end{minipage}
\hfill
\begin{minipage}[t]{0.40\linewidth}
\centering
\caption{
\textbf{Cross-dataset generalization.}}
\begin{tabular}{@{}lcc@{}}
\toprule
\textbf{Method} & \textbf{AffectNet}$\uparrow$ & \textbf{CelebA-HQ}$\uparrow$ \\
\midrule
ControlNet++~\cite{li2024controlnet++} & 60.3 & 0.857 \\
ICEdit~\cite{zhang2025enabling} & 67.5 & 0.858 \\
FLUX~\cite{labs2025flux} & 69.1 & 0.862 \\
DreamBooth+CNet++~\cite{ruiz2023dreambooth,li2024controlnet++} & 69.2 & 0.865 \\
\midrule
\rowcolor{bestrow}
\textbf{Controlla} & \textbf{72.8} & \textbf{0.868} \\
\bottomrule
\end{tabular}
\label{tab:generalization}
\end{minipage}

\end{table*}

\begin{table*}[h]
\centering
\scriptsize
\setlength{\tabcolsep}{1.5pt}
\renewcommand{\arraystretch}{0.82}

\begin{minipage}[t]{0.30\linewidth}
\centering
\caption{
\textbf{Architecture-level comparison.}}
\begin{tabular}{@{}lcccc@{}}
\toprule
\textbf{Method} & Acc $\uparrow$ & TS $\uparrow$ & LDS $\uparrow$ & ID $\uparrow$\\
\midrule
SDXL \cite{podell2023sdxlimprovinglatentdiffusion} & 75.2 & 0.66 & 0.64 & 0.861 \\
+ControlNet++~\cite{li2024controlnet++} & 76.1 & 0.68 & 0.66 & 0.864 \\
\midrule
\rowcolor{bestrow}
Controlla & \textbf{77.6} & \textbf{0.73} & \textbf{0.71} & \textbf{0.868} \\
\bottomrule
\end{tabular}
\label{tab:archi}
\end{minipage}
\hfill
\begin{minipage}[t]{0.35\linewidth}
\centering
\caption{
\textbf{Geometry-aware evaluation.}}
\begin{tabular}{@{}lcccccc@{}}
\toprule
& \multicolumn{3}{c}{Val} & \multicolumn{3}{c}{Test} \\
\textbf{Method}
& ID$\uparrow$ & LDS$\uparrow$ & GC$\downarrow$
& ID$\uparrow$ & LDS$\uparrow$ & GC$\downarrow$ \\
\midrule
DreamBooth~\cite{ruiz2023dreambooth}
& .882 & .510 & .214 & .871 & .490 & .227 \\
ICEdit~\cite{zhang2025enabling}
& .842 & .640 & .166 & .829 & .620 & .178 \\
FLUX~\cite{labs2025flux}
& .848 & .660 & .159 & .836 & .640 & .171 \\
\midrule
\rowcolor{bestrow}
Controlla
& \textbf{.874} & \textbf{.710} & \textbf{.118}
& \textbf{.862} & \textbf{.690} & \textbf{.126} \\
\bottomrule
\end{tabular}
\label{tab:geometry}
\end{minipage}
\hfill
\begin{minipage}[t]{0.25\linewidth}
\centering
\caption{
\textbf{Cross-modal retrieval evaluation.}}
\begin{tabular}{@{}lcccc@{}}
\toprule
\textbf{Method} & \multicolumn{2}{c}{\textbf{I2T}} & \multicolumn{2}{c}{\textbf{I2A}} \\
\cmidrule(lr){2-3}\cmidrule(l){4-5}
& \textbf{R@1} & \textbf{R@5} & \textbf{R@1} & \textbf{R@5} \\
\midrule
CLIP~\cite{radford2021learning} & 71.4 & 86.2 & 69.1 & 84.0 \\
ImageBind~\cite{girdhar2023imagebind} & 73.0 & 87.1 & 70.5 & 85.2 \\
\midrule
\rowcolor{bestrow}
Controlla & \textbf{76.2} & \textbf{89.8} & \textbf{74.1} & \textbf{87.5} \\
\bottomrule
\end{tabular}
\label{tab:cross}
\end{minipage}
\vspace{-6pt}
\end{table*}
\subsection{Main Results}
\vspace{-6pt}
\noindent\textbf{Controllability and generalization.}
Table~\ref{tab:main_results} reports the main AffectHuman-43K results. Controlla improves emotion accuracy, trajectory smoothness, and human preference. We treat ImageBind alignment as an auxiliary coherence diagnostic, while the main claims are supported by CLIP/ImageBind-independent metrics such as Acc, TS, ID, LDS, GC, and human evaluation. The gains remain consistent under cross-dataset evaluation on AffectNet and CelebA-HQ, suggesting improved affective and identity-related generalization compared with ControlNet++, ICEdit, FLUX, and DreamBooth+CNet++ (Table~\ref{tab:generalization}).

\noindent\textbf{Backbone, geometry, and retrieval analysis.}
Controlla also improves over stronger backbone and conditioning variants, including SDXL and SDXL+ControlNet++, indicating that the proposed control mechanism is complementary to model scaling and external conditioning (Table~\ref{tab:archi}). 
Geometry-aware evaluation further shows stronger latent disentanglement and graph consistency while maintaining competitive identity preservation (Table~\ref{tab:geometry}). 
Finally, cross-modal retrieval results show improved alignment over CLIP and ImageBind baselines, suggesting better consistency across visual, textual, and audio-conditioned representations (Table~\ref{tab:cross}).
\textit{\textbf{All reported improvements are statistically significant ($p<0.01$); see Appendix Sec.~\hyperref[sec6]{F}. Additional baseline analyses are provided in Appendix Sec.~\hyperref[sec8]{H}.}}

\subsection{Ablation Studies}

We perform targeted ablations to identify which components drive Controlla's gains in structured controllability. The studies isolate loss terms, graph structure, modality grounding, and traversal strategy, testing whether the improvements arise from meaningful affective geometry rather than arbitrary smoothing or additional conditioning alone. \textit{\textbf{Additional ablations and sensitivity analyses are provided in Appendix Sec.~\hyperref[sec9]{I}.}}

\begin{figure}[h]
\centering
\includegraphics[width=\linewidth]{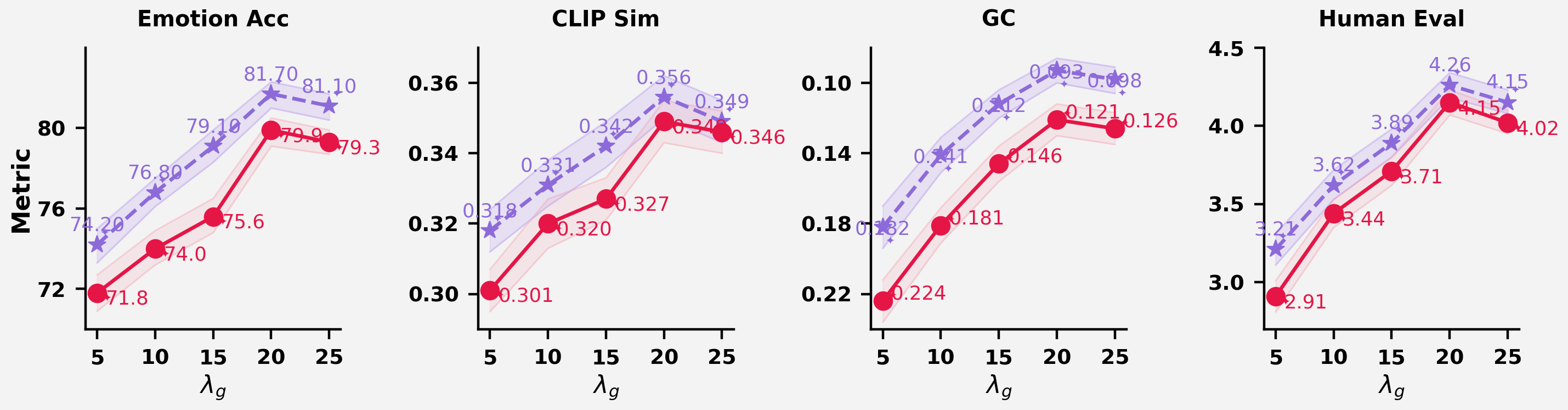}
\caption{\textbf{Effect of graph strength.} Increasing $\lambda_g$ improves controllability and human preference while reducing GC, with CLIP reported only as an auxiliary alignment diagnostic.}
\label{fig:ablation_graphs}
\end{figure}

\paragraph{Effect of Components and Graph Structure.}
Table~\ref{tab:ablation_components_graph} shows that removing FGW/GW weakens attribute controllability and identity consistency, while removing $\mathcal{L}_{\perp}$ or $\mathcal{L}_{\mathrm{Lip}}$ reduces factor separation and transition smoothness. Random and fully connected graphs degrade TS, LDS, and GC, whereas the fine-grained graph improves over the compact 8-node graph, indicating that gains come from meaningful affective geometry rather than arbitrary connectivity.

\vspace{-6pt}
\paragraph{Effect of Multimodal Grounding.}
Table~\ref{tab:ablation_modalities} shows that image, text, and audio provide complementary signals.
Single-modality and partial-modality variants generally reduce endpoint accuracy, trajectory smoothness, disentanglement, and human preference, while the full image+text+audio setting achieves the strongest overall performance.

\vspace{-6pt}
\paragraph{Effect of Latent Traversal.}
Figure~\ref{fig:ablation_graphs} shows that increasing graph regularization improves controllability and human preference while reducing GC.
Table~\ref{tab:geodesic_vs_euclidean} shows that Controlla's graph-consistent path achieves stronger smoothness, identity preservation, LDS, and GC than linear, spline, random, or fully connected paths.
Figure~\ref{fig:geodesic_vs_linear} qualitatively illustrates that graph-consistent traversal produces smoother and more identity-preserving transitions than linear interpolation. Together, these results confirm that structured latent traversal is central to Controlla's controllability mechanism.
\begin{table*}[t]
\centering
\fontsize{5.7}{6.3}\selectfont
\setlength{\tabcolsep}{0.9pt}
\renewcommand{\arraystretch}{0.70}

\begin{minipage}[t]{0.55\textwidth}
\centering
\caption{\textbf{Component and graph ablation.}}
\label{tab:ablation_components_graph}
\vspace{-5pt}
\resizebox{\linewidth}{!}{%
\begin{tabular}{@{}lcccccc@{}}
\toprule
\rowcolor{gray!12}
\textbf{Method} & \textbf{Acc$\uparrow$} & \textbf{TS$\uparrow$} & \textbf{CLIP$\uparrow$} & \textbf{LDS$\uparrow$} & \textbf{GC$\downarrow$} & \textbf{ID$\uparrow$} \\
\midrule
\textit{Simpler Alternatives} \\
Contrastive-only & 74.6 & 0.63 & 0.310 & 0.60 & 0.152 & 0.858 \\
Metric (no OT) & 74.4 & 0.62 & 0.308 & 0.58 & 0.158 & 0.857 \\
\midrule
\textit{Loss / Component Ablations} \\
w/o FGW & 74.8 & 0.65 & 0.312 & 0.63 & 0.141 & 0.861 \\
w/o GW & 75.1 & 0.66 & 0.314 & 0.66 & 0.134 & 0.842 \\
w/o $\mathcal{L}_{\perp}$ & 73.9 & 0.62 & 0.309 & 0.59 & 0.152 & 0.853 \\
w/o $\mathcal{L}_{L}$ & 75.2 & 0.64 & 0.313 & 0.66 & 0.143 & 0.859 \\
w/o Graph & 74.3 & 0.61 & 0.307 & 0.57 & 0.161 & 0.856 \\
\midrule
\textit{Graph Structure / Sensitivity} \\
Random Graph & 73.7 & 0.59 & 0.304 & 0.56 & 0.178 & 0.854 \\
Fully Connected & 74.1 & 0.60 & 0.306 & 0.58 & 0.166 & 0.855 \\
8-node Graph & 75.3 & 0.67 & 0.313 & 0.66 & 0.138 & 0.860 \\
Fine-grained Graph & 76.1 & 0.70 & 0.315 & 0.68 & 0.128 & 0.862 \\
$k=5$ & 76.0 & 0.69 & 0.315 & 0.67 & 0.131 & 0.861 \\
$k=20$ & 76.2 & 0.70 & 0.315 & 0.68 & 0.129 & 0.862 \\
\midrule
\rowcolor{bestrow}
\textbf{Full Model ($k=10$)} & \textbf{76.4} & \textbf{0.71} & \textbf{0.316} & \textbf{0.69} & \textbf{0.126} & \textbf{0.862} \\
\bottomrule
\end{tabular}
}
\end{minipage}
\hspace{0.012\textwidth}
\begin{minipage}[t]{0.43\textwidth}

\centering
\caption{\textbf{Modality ablation.}}
\label{tab:ablation_modalities}
\vspace{-5pt}
\resizebox{\linewidth}{!}{%
\begin{tabular}{@{}lcccccc@{}}
\toprule
\rowcolor{gray!12}
\textbf{Setting} & \textbf{Acc$\uparrow$} & \textbf{CLIP$\uparrow$} & \textbf{IB$\uparrow$} & \textbf{TS$\uparrow$} & \textbf{LDS$\uparrow$} & \textbf{H$\uparrow$} \\
\midrule
Image only & 69.1 & 0.292 & 0.251 & 0.54 & 0.58 & 3.36 \\
Audio only & 68.4 & 0.286 & 0.263 & 0.53 & 0.56 & 3.29 \\
Text only & 72.8 & 0.306 & 0.268 & 0.62 & 0.63 & 3.58 \\
Image + Audio & 71.7 & 0.298 & 0.272 & 0.59 & 0.61 & 3.54 \\
Image + Text & 74.2 & 0.311 & 0.281 & 0.66 & 0.66 & 3.71 \\
Text + Audio & 73.4 & 0.309 & 0.276 & 0.64 & 0.65 & 3.66 \\
\midrule
\rowcolor{bestrow}
\textbf{Img+Txt+Aud} & \textbf{76.4} & \textbf{0.316} & \textbf{0.279} & \textbf{0.71} & \textbf{0.69} & \textbf{3.82} \\
\bottomrule
\end{tabular}
}

\vspace{4pt}

\caption{\textbf{Traversal comparison.}}
\label{tab:geodesic_vs_euclidean}
\vspace{-5pt}
\resizebox{\linewidth}{!}{%
\begin{tabular}{@{}lccccc@{}}
\toprule
\rowcolor{gray!12}
\textbf{Traversal} & \textbf{TS$\uparrow$} & \textbf{LDS$\uparrow$} & \textbf{GC$\downarrow$} & \textbf{ID$\uparrow$} & \textbf{CLIP$\uparrow$} \\
\midrule
Linear & 0.64 & 0.61 & 0.148 & 0.812 & 0.302 \\
Spline & 0.69 & 0.63 & 0.165 & 0.870 & 0.318 \\
Random Path & 0.60 & 0.58 & 0.179 & 0.851 & 0.303 \\
Full Conn. Path & 0.62 & 0.60 & 0.163 & 0.856 & 0.309 \\
\rowcolor{bestrow}
\textbf{Controlla Path} & \textbf{0.71} & \textbf{0.69} & \textbf{0.126} & \textbf{0.918} & \textbf{0.325} \\
\bottomrule
\end{tabular}
}

\end{minipage}
\end{table*}

\begin{figure}[h]
\centering
\includegraphics[width=0.8\textwidth]{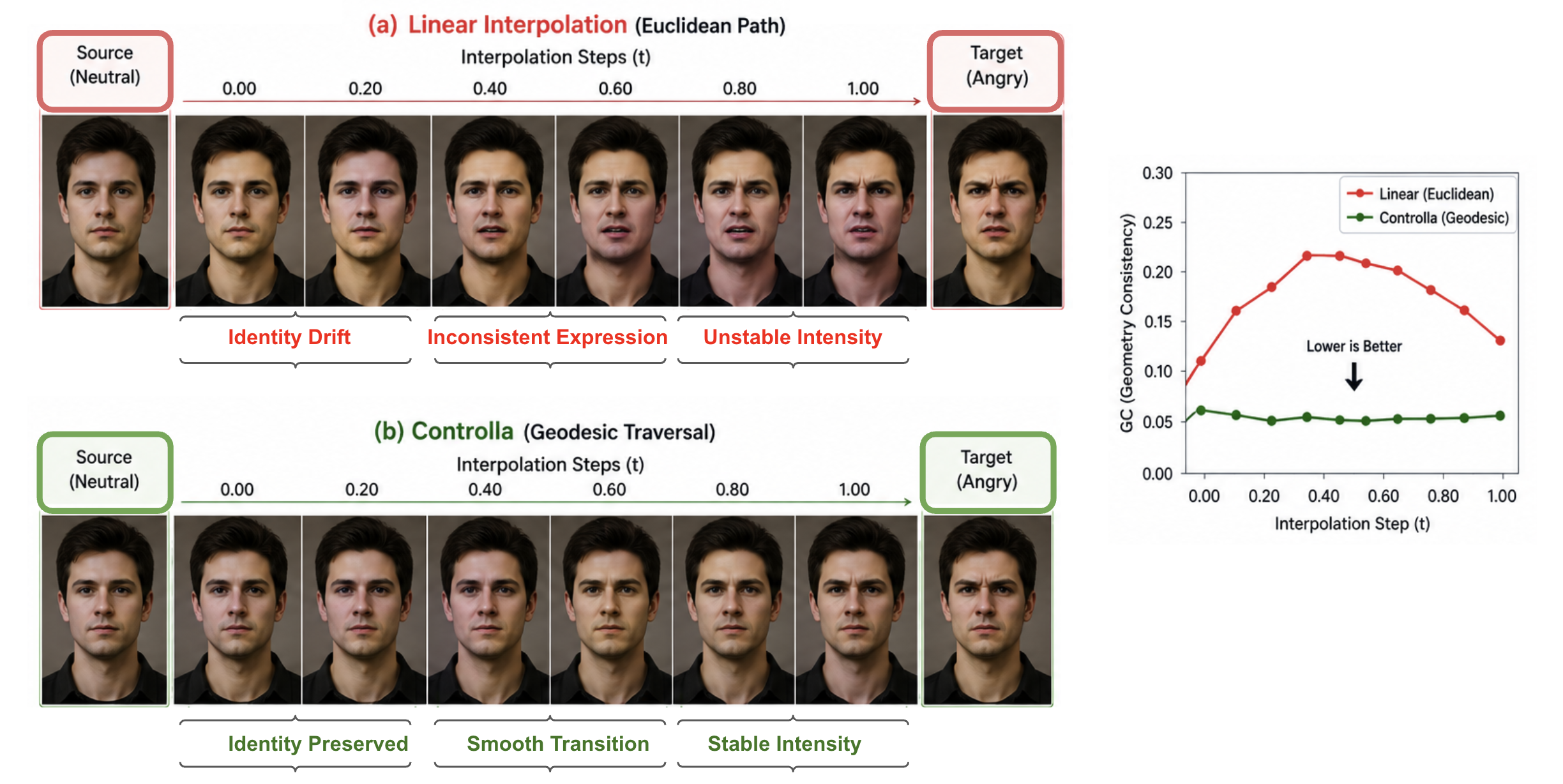}
\caption{
\textbf{Graph-consistent vs. linear traversal.}
Graph-consistent traversal yields smoother, identity-preserving transitions and lower GC than linear interpolation.
}
\label{fig:geodesic_vs_linear}
\end{figure}
\vspace{-6pt}
\subsection{Qualitative Results}
Figure~\ref{fig:qual_compare} compares emotional edits across CelebA-HQ, AffectNet, and AffectHuman-43K. 
Controlla produces semantically accurate and identity-consistent edits, while prior methods show identity drift, weaker attribute control, text-condition mismatch, or non-smooth transitions.  These trends are consistent across identities and datasets, indicating robust transfer beyond the training benchmark.  Figure~\ref{fig:qual_graph} further shows that graph-guided traversal produces smoother and more interpretable transitions than linear interpolation. \textit{\textbf{See Appendix  Sec. \hyperref[sec10]{~J} and Sec. \hyperref[sec11]{~K} for extended qualitative analysis, including cross-dataset transfer, fine-grained emotion variation, and challenging failure cases.}} 
The formulation can support alternative graph-defined control factors beyond affective control, such as pose, expression intensity, or style; \textit{\textbf{see Appendix Sec.~\hyperref[sec12]{L} and Sec.~\hyperref[sec13]{M}.}}
\vspace{-8pt}
\begin{figure}[!t]
\centering
\includegraphics[width=1.0\linewidth]{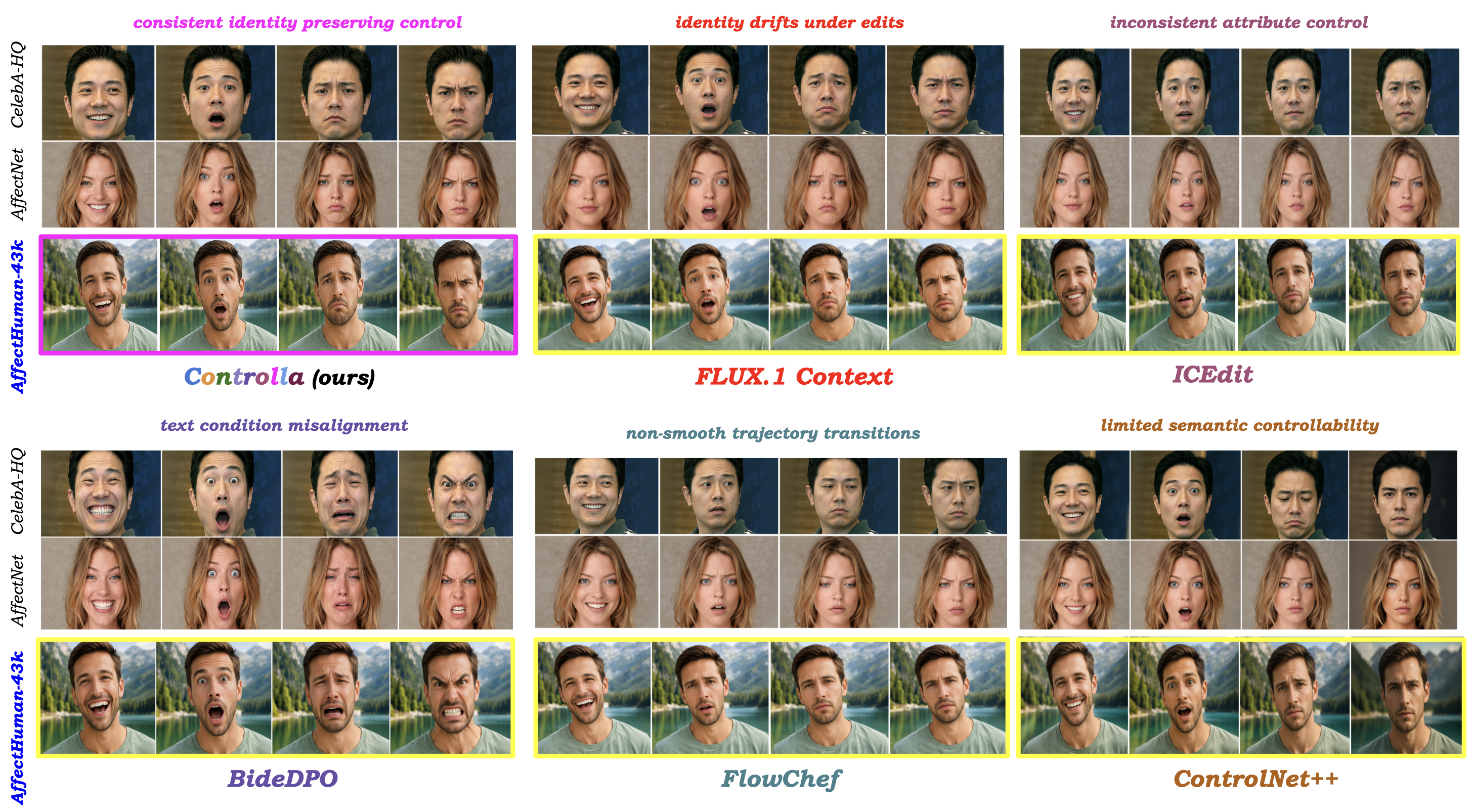}
\vspace{-6pt}
\caption{
\textbf{Cross-dataset qualitative comparison.}
Generated examples show identity preservation, expression consistency, and semantic alignment across datasets and methods under matched inputs.
}
\label{fig:qual_compare}
\end{figure}
\vspace{-6pt}
\begin{figure}[!t]
\centering
\includegraphics[width=0.95\linewidth]{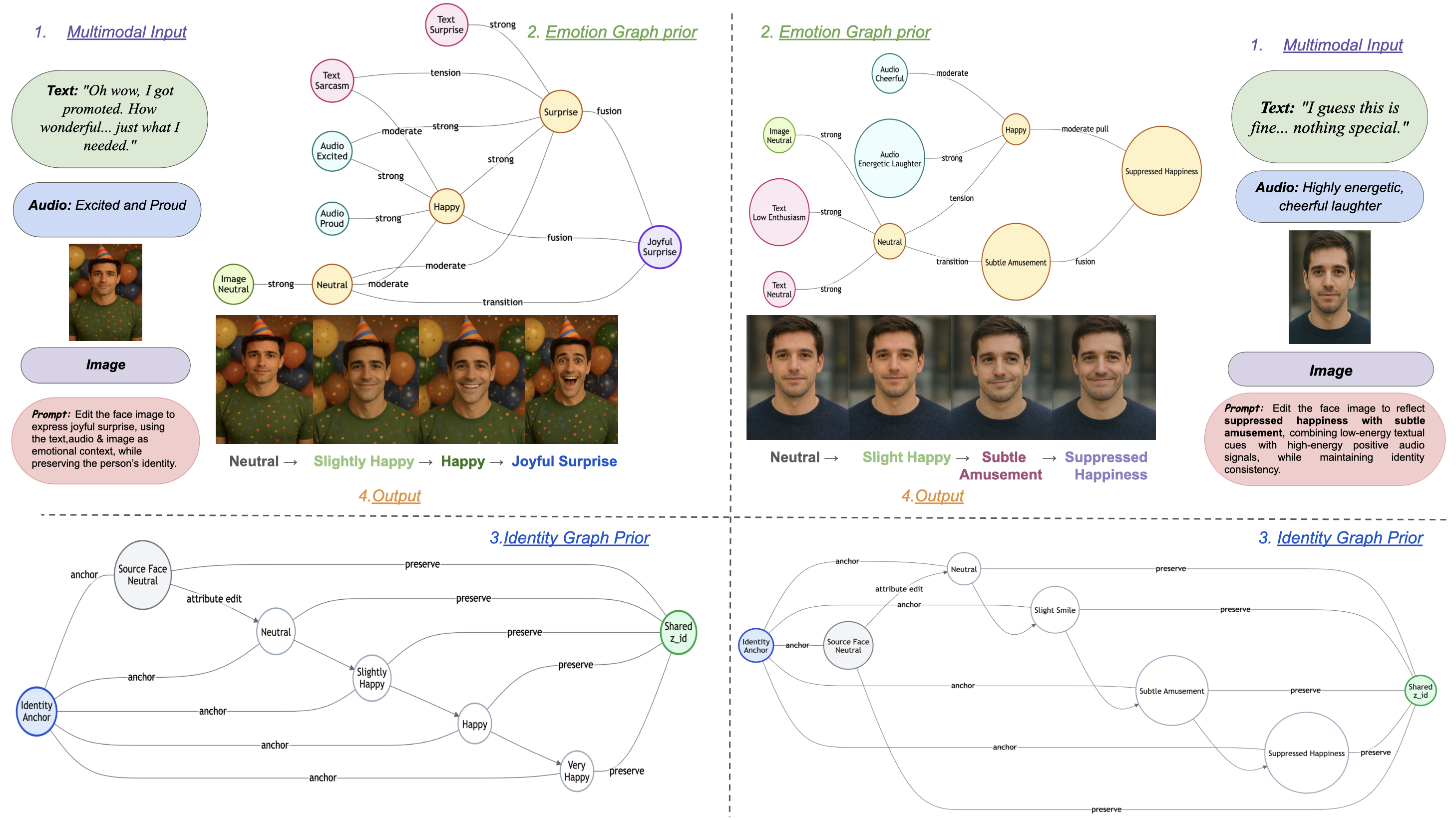}
\vspace{-8pt}
\caption{
\textbf{Graph-guided latent control.}
Structured latent geometry yields smooth emotion transitions with consistent identity, unlike linear interpolation.
}
\vspace{-8pt}
\label{fig:qual_graph}
\end{figure}
\vspace{-8pt}
\section{Conclusion and Future Work}
\label{sec:conclusion}
\vspace{-8pt}
We presented Controlla, which models controllability as a property of structured latent geometry rather than inference-time conditioning. By combining graph-constrained optimal transport with learned identity and attribute factors, Controlla enables smooth, identity-preserving transformations along graph-consistent trajectories. Together with AffectHuman-43K and geometry-aware metrics, our experiments evaluate controllability through endpoint accuracy, identity preservation, latent traversal, cross-modal alignment, and human preference. Results, ablations, graph-sensitivity analyses, and robustness studies show consistent gains over strong conditioning, editing, and personalization baselines. These findings support a shift from heuristic conditioning toward geometry-driven representation learning for controllable multimodal generation. Future work includes extending Controlla to video, broader graph-defined control factors, and stronger geometry-aware evaluation protocols.
\vspace{-8pt}
\bibliography{software}


\clearpage
\appendix




















\part*{Appendix}
\addcontentsline{toc}{part}{Appendix}

\etocsettocstyle{\section*{Appendix Contents}}{}
\etocsetnexttocdepth{subsection}
\localtableofcontents

\newpage

\section{Theoretical Analysis of Controlla}
\label{sec1}

We provide a formal analysis of the representation-level claim underlying Controlla: controllability can be encouraged by learning latent factors that approximately preserve a semantic graph metric while maintaining identity stability. The purpose of this analysis is not to prove global optimality of a neural generator, nor to guarantee perfect controllability for all inputs. Instead, we show that, under explicit approximation and regularity assumptions, the Controlla objective induces a bound in which output variation is controlled by semantic graph distance and by measurable representation errors.

The analysis connects four quantities that correspond to the main components of the proposed objective:
\begin{itemize}
    \item \textbf{attribute-graph distortion}, measuring how well the learned attribute factor preserves the emotion-graph metric;
    \item \textbf{identity instability}, measuring variation in the learned identity factor for the same reference identity;
    \item \textbf{attribute--identity coupling}, measuring leakage between controllable and identity-related factors;
    \item \textbf{local generator sensitivity}, measuring how strongly output changes respond to latent perturbations.
\end{itemize}

The result should be read as a conditional representation-level justification: if graph alignment, identity stability, factor separation, and local smoothness errors are small, then semantic traversal in the learned attribute factor produces bounded output variation under a fixed or approximately fixed identity factor.

\subsection{Structured Controllability as Metric Preservation}

Let $G_e=(V_e,E_e,D_e)$ denote an emotion graph, where $V_e$ is the set of semantic emotion states, $E_e$ is the edge set, and $D_e(u,v)$ is the shortest-path distance between two states $u,v\in V_e$. We interpret $D_e(u,v)$ as the semantic cost of moving from affective state $u$ to affective state $v$.

Given a multimodal input
\[
x=(x^{\mathrm{img}},x^{\mathrm{ref}},x^{\mathrm{text}},x^{\mathrm{aud}}),
\]
the encoder first maps the input to a shared representation
\begin{equation}
z=f_\theta(x).
\label{eq:supp_shared_encoder}
\end{equation}
Rather than imposing a fixed dimensional split, Controlla learns attribute and identity factors through trainable factorization heads:
\begin{equation}
z_{\mathrm{attr}}=h_{\mathrm{attr}}(z),
\qquad
z_{\mathrm{id}}=h_{\mathrm{id}}(z).
\label{eq:supp_factor_heads}
\end{equation}
Here, $z_{\mathrm{attr}}$ represents controllable semantic attributes and $z_{\mathrm{id}}$ represents reference-grounded identity information. The generator produces an output from the learned factors:
\begin{equation}
\hat{x}
=
g_\phi(z_{\mathrm{attr}},z_{\mathrm{id}}).
\label{eq:supp_generator}
\end{equation}

For each emotion node $u\in V_e$, let $z_u$ denote the corresponding learned attribute representation. In practice, $z_u$ may be an emotion prototype, a cluster centroid, or an averaged attribute embedding associated with state $u$. For graph nodes $u,v\in V_e$, we write
\[
D_Z^{\mathrm{attr}}(u,v)=\|z_u-z_v\|
\]
for the distance between their corresponding learned attribute representations.

\begin{definition}[Graph-Metric Controllability]
A model is $(L,\delta)$ graph-metric controllable on $G_e$ if, for a fixed identity factor $z_{\mathrm{id}}$ and semantic states $u,v\in V_e$, the generated outputs satisfy
\begin{equation}
\left\|
g_\phi(z_u,z_{\mathrm{id}})
-
g_\phi(z_v,z_{\mathrm{id}})
\right\|
\le
L D_e(u,v)+\delta .
\label{eq:supp_graph_metric_control}
\end{equation}
\end{definition}

This definition treats controllability as a geometric property of latent traversal. Nearby states on the semantic graph should induce bounded changes in the output, while more distant semantic states may induce larger changes. This differs from inference-time conditioning, where a control signal is imposed externally without necessarily constraining the geometry of intermediate latent states.

\subsection{Controlla Objective}

For a batch of $N$ samples, let
\[
Z_{\mathrm{attr}}=\{z_{\mathrm{attr}}^n\}_{n=1}^N,
\qquad
Z_{\mathrm{id}}=\{z_{\mathrm{id}}^n\}_{n=1}^N
\]
denote the learned attribute and identity factors. Similarly, let $G_i=(V_i,E_i,D_i)$ denote the reference-identity graph, where $D_i$ encodes relational distances among reference-identity groups.

Controlla optimizes a graph-constrained latent objective of the form
\begin{equation}
\mathcal{L}_{\mathrm{Controlla}}
=
\lambda_e \mathcal{L}_{\mathrm{FGW}}(Z_{\mathrm{attr}},G_e)
+
\lambda_i \mathcal{L}_{\mathrm{GW}}(Z_{\mathrm{id}},G_i)
+
\lambda_\perp \mathcal{L}_{\perp}
+
\lambda_L \mathcal{L}_{\mathrm{Lip}} .
\label{eq:supp_controlla_objective}
\end{equation}

The attribute alignment term uses Fused Gromov--Wasserstein alignment:
\begin{equation}
\mathcal{L}_{\mathrm{FGW}}(Z_{\mathrm{attr}},G_e)
=
\min_{\pi}
\sum_{i,j,k,l}
\left|
D_e(i,j)
-
D_Z^{\mathrm{attr}}(k,l)
\right|^2
\pi_{ik}\pi_{jl}
+
\alpha
\sum_{i,k}
C_{ik}\pi_{ik},
\label{eq:supp_fgw}
\end{equation}
where $D_e$ is the emotion-graph distance matrix, $D_Z^{\mathrm{attr}}$ is the pairwise distance matrix in the learned attribute factor, $C$ is a feature-level matching cost, and $\pi$ is a transport plan.

The identity alignment term uses Gromov--Wasserstein alignment:
\begin{equation}
\mathcal{L}_{\mathrm{GW}}(Z_{\mathrm{id}},G_i)
=
\min_{\pi}
\sum_{i,j,k,l}
\left|
D_i(i,j)
-
D_Z^{\mathrm{id}}(k,l)
\right|^2
\pi_{ik}\pi_{jl},
\label{eq:supp_gw}
\end{equation}
where $D_i$ is the identity-graph distance matrix and $D_Z^{\mathrm{id}}$ is the pairwise distance matrix in the learned identity factor.

To discourage linear dependence between the learned attribute and identity factors, we use
\begin{equation}
\mathcal{L}_{\perp}
=
\left\|
Z_{\mathrm{attr}}^\top Z_{\mathrm{id}}
\right\|_F^2.
\label{eq:supp_orthogonality}
\end{equation}

Finally, to reduce sensitivity to small latent perturbations, we use a local Lipschitz regularizer:
\begin{equation}
\mathcal{L}_{\mathrm{Lip}}
=
\mathbb{E}_{z,z'}
\left[
\max
\left(
0,
\left\|
g_\phi(z)-g_\phi(z')
\right\|
-
L
\left\|
z-z'
\right\|
\right)
\right].
\label{eq:supp_lip_loss}
\end{equation}

Each term controls a distinct failure mode. The FGW term reduces attribute-graph distortion, the GW term encourages stable identity structure, the orthogonality term reduces attribute--identity coupling, and the Lipschitz term limits unstable generator responses.

\subsection{Assumptions}

We state the result under explicit approximation assumptions. These assumptions are not claims that neural optimization always achieves exact structure. Rather, they formalize the representation properties encouraged by the training objective.

\begin{assumption}[Approximate Attribute Metric Preservation]
After optimization, the learned attribute distances approximately preserve the emotion-graph metric in expectation:
\begin{equation}
\mathbb{E}_{u,v\sim V_e}
\left[
\left|
D_Z^{\mathrm{attr}}(u,v)
-
D_e(u,v)
\right|^2
\right]
\le
\varepsilon_e .
\label{eq:supp_attr_metric_assumption}
\end{equation}
\end{assumption}

This assumption is encouraged by the FGW term, which penalizes disagreement between pairwise distances in the emotion graph and pairwise distances in the learned attribute factor.

\begin{assumption}[Approximate Identity Stability]
For samples $p,q$ belonging to the same reference-identity group $C_k$, the identity factors satisfy
\begin{equation}
\mathbb{E}_{p,q\in C_k}
\left[
\left\|
z_p^{\mathrm{id}}
-
z_q^{\mathrm{id}}
\right\|^2
\right]
\le
\varepsilon_i .
\label{eq:supp_identity_stability}
\end{equation}
\end{assumption}

This assumption formalizes identity stability: samples sharing a reference identity should have nearby learned identity factors. We denote the resulting identity-factor variation during traversal by $\Delta_{\mathrm{id}}$.

\begin{assumption}[Bounded Attribute--Identity Coupling]
The learned representation has bounded attribute--identity coupling. We model the identity-relevant perturbation induced by imperfect factor separation as $\Delta_\perp$ and assume
\begin{equation}
\mathbb{E}
\left[
\Delta_\perp
\right]
\le
\sqrt{\varepsilon_\perp}.
\label{eq:supp_coupling_assumption}
\end{equation}
In practice, this perturbation is controlled by the factor-separation penalty
\begin{equation}
\left\|
Z_{\mathrm{attr}}^\top Z_{\mathrm{id}}
\right\|_F^2
\le
\varepsilon_\perp .
\label{eq:supp_subspace_decoupling}
\end{equation}
\end{assumption}

This assumption does not imply full statistical independence between $z_{\mathrm{attr}}$ and $z_{\mathrm{id}}$. It only states that the identity-relevant variation induced by imperfect factor separation is bounded.

\begin{assumption}[Local Lipschitz Generator]
The generator is locally Lipschitz on a neighborhood of the learned latent manifold. For latent codes $(z_a,z_i)$ and $(z_a',z_i')$ in this neighborhood,
\begin{equation}
\left\|
g_\phi(z_a,z_i)
-
g_\phi(z_a',z_i')
\right\|
\le
L_g
\left(
\left\|
z_a-z_a'
\right\|
+
\left\|
z_i-z_i'
\right\|
\right).
\label{eq:supp_lipschitz_assumption}
\end{equation}
\end{assumption}

This is a local regularity assumption, not a global guarantee over all latent codes. It matches the empirical setting, where the generator is regularized near the training manifold.

\subsection{Main Result}

\begin{theorem}[Expected Graph-Metric Controllability]
Under Assumptions~1--4, for semantic states $u,v\sim V_e$ and identity variation satisfying Assumption~2, the expected output variation under attribute traversal satisfies
\begin{equation}
\mathbb{E}_{u,v\sim V_e}
\left[
\left\|
g_\phi(z_u,z_{\mathrm{id}})
-
g_\phi(z_v,z_{\mathrm{id}})
\right\|
\right]
\le
L_g
\mathbb{E}_{u,v\sim V_e}
\left[
D_e(u,v)
\right]
+
L_g
\left(
\sqrt{\varepsilon_e}
+
\sqrt{\varepsilon_i}
+
\sqrt{\varepsilon_\perp}
\right).
\label{eq:supp_expected_control_bound}
\end{equation}
\end{theorem}

\begin{proof}
By the local Lipschitz property in Eq.~\eqref{eq:supp_lipschitz_assumption}, for two attribute states $z_u$ and $z_v$ and an approximately fixed identity factor, we have
\begin{equation}
\left\|
g_\phi(z_u,z_{\mathrm{id}})
-
g_\phi(z_v,z_{\mathrm{id}})
\right\|
\le
L_g
\left\|
z_u-z_v
\right\|
+
L_g\Delta_{\mathrm{id}}
+
L_g\Delta_{\perp},
\label{eq:supp_lip_step}
\end{equation}
where $\Delta_{\mathrm{id}}$ denotes identity-factor variation and $\Delta_{\perp}$ denotes additional identity-relevant perturbation caused by imperfect attribute--identity separation.

From Assumption~1 and Jensen's inequality,
\begin{equation}
\mathbb{E}_{u,v\sim V_e}
\left[
\left|
D_Z^{\mathrm{attr}}(u,v)-D_e(u,v)
\right|
\right]
\le
\sqrt{\varepsilon_e}.
\label{eq:supp_attr_jensen}
\end{equation}
Since $D_Z^{\mathrm{attr}}(u,v)=\|z_u-z_v\|$, this gives
\begin{equation}
\mathbb{E}_{u,v\sim V_e}
\left[
\left\|
z_u-z_v
\right\|
\right]
\le
\mathbb{E}_{u,v\sim V_e}
\left[
D_e(u,v)
\right]
+
\sqrt{\varepsilon_e}.
\label{eq:supp_attr_distance_bound}
\end{equation}

From Assumption~2,
\begin{equation}
\mathbb{E}
\left[
\Delta_{\mathrm{id}}
\right]
\le
\sqrt{\varepsilon_i}.
\label{eq:supp_identity_error_bound}
\end{equation}
From Assumption~3,
\begin{equation}
\mathbb{E}
\left[
\Delta_{\perp}
\right]
\le
\sqrt{\varepsilon_\perp}.
\label{eq:supp_leakage_error_bound}
\end{equation}

Taking expectations in Eq.~\eqref{eq:supp_lip_step} and substituting Eqs.~\eqref{eq:supp_attr_distance_bound}, \eqref{eq:supp_identity_error_bound}, and \eqref{eq:supp_leakage_error_bound} yields
\begin{equation}
\mathbb{E}_{u,v\sim V_e}
\left[
\left\|
g_\phi(z_u,z_{\mathrm{id}})
-
g_\phi(z_v,z_{\mathrm{id}})
\right\|
\right]
\le
L_g
\mathbb{E}_{u,v\sim V_e}
\left[
D_e(u,v)
\right]
+
L_g
\left(
\sqrt{\varepsilon_e}
+
\sqrt{\varepsilon_i}
+
\sqrt{\varepsilon_\perp}
\right).
\end{equation}
This proves the theorem.
\end{proof}

\paragraph{Pointwise version.}
A pointwise version follows only under stronger bounded-distortion assumptions. Specifically, if for a given pair $(u,v)$,
\[
\left|
D_Z^{\mathrm{attr}}(u,v)-D_e(u,v)
\right|
\le
\eta_e,
\qquad
\Delta_{\mathrm{id}}\le \eta_i,
\qquad
\Delta_{\perp}\le \eta_\perp,
\]
then
\begin{equation}
\left\|
g_\phi(z_u,z_{\mathrm{id}})
-
g_\phi(z_v,z_{\mathrm{id}})
\right\|
\le
L_gD_e(u,v)
+
L_g(\eta_e+\eta_i+\eta_\perp).
\label{eq:supp_pointwise_control_bound}
\end{equation}
We use the expected form as the main statement because it better matches empirical neural-network optimization and dataset-level evaluation.

\subsection{Interpretation and Scope}

Theorem~1 shows that Controlla converts controllability into a metric-regularized representation learning problem. The bound decomposes output variation into semantic graph distance and three error terms: attribute-graph distortion, identity instability, and attribute--identity coupling. Reducing attribute-graph distortion makes learned attribute distances more consistent with semantic distances on the emotion graph. Reducing identity instability makes the identity factor more stable for samples sharing the same reference identity. Reducing coupling error limits unintended interaction between the learned attribute and identity factors. Reducing local generator sensitivity limits how strongly outputs change in response to small latent perturbations.

This result also clarifies the difference between Controlla and conditioning-based control. Conditioning-based methods may impose prompts, instructions, spatial maps, or auxiliary signals at inference time, but they do not necessarily constrain the latent geometry through which semantic changes occur. Controlla instead encourages the learned representation itself to carry a graph-structured notion of control.

The analysis is conditional and representation-level. It does not prove that neural training reaches a global optimum, that the generator is globally Lipschitz, or that graph-metric controllability holds for arbitrary out-of-distribution inputs. It also does not guarantee perceptual realism or perfect semantic correctness of generated images. Those properties must be evaluated empirically. The role of the analysis is to explain why lower graph distortion, stronger identity stability, weaker factor coupling, and smoother generator response should improve controllability under fixed-identity traversal.


\section{AffectHuman-43K Benchmark Construction}
\label{sec2}

We introduce \textbf{AffectHuman-43K}, a controlled multimodal benchmark for evaluating \textit{reference-grounded, identity-preserving affective control}. The benchmark tests whether a generative model can modify affective attributes specified through heterogeneous multimodal signals while preserving the identity specified by a visual reference image. Unlike naturally co-recorded multimodal identity datasets, AffectHuman-43K is constructed as a \textbf{controlled evaluation protocol}: identity supervision is grounded only through a reference-image anchor, whereas text and audio provide affective control signals. This separation is intentional, allowing us to evaluate whether a model can disentangle identity from controllable semantic attributes without relying on cross-modal identity correspondence.

Naturally paired multimodal emotion datasets such as IEMOCAP, CMU-MOSEI, MELD, and RAVDESS provide valuable aligned audio--visual--text or audio--visual emotion data, but they are primarily designed for recognition or sentiment analysis and do not provide reference-image-grounded, identity-disjoint evaluation for controllable generation. AffectHuman-43K addresses this complementary setting by decoupling visual identity from affective control signals and evaluating whether models preserve a reference identity while following heterogeneous affective cues.

AffectHuman-43K is constructed from publicly available datasets used under their respective research licenses. Image samples are drawn from affective and face-centric datasets such as AffectNet, RAF-DB, and CelebA-HQ. Audio samples are drawn from real human speech datasets, primarily IEMOCAP. Textual affective descriptions and utterances are derived from IEMOCAP transcripts and EmoBank annotations. No new data is scraped, no private data is collected, and no synthetic human identities are generated during benchmark construction.

Given heterogeneous source datasets
$\mathcal{D}_{\mathrm{img}}$, $\mathcal{D}_{\mathrm{text}}$, and $\mathcal{D}_{\mathrm{aud}}$,
our goal is to construct a leakage-safe benchmark
\[
\mathcal{S}^{*}
=
\left\{
\left(
x^{\mathrm{img}},
x^{\mathrm{ref}},
x^{\mathrm{text}},
x^{\mathrm{aud}},
y^{\mathrm{id}},
y^{\mathrm{emo}},
g^{\mathrm{id}}
\right)
\right\},
\]
where $x^{\mathrm{img}}$ is the target image, $x^{\mathrm{ref}}$ is the reference image specifying identity,
$x^{\mathrm{text}}$ and $x^{\mathrm{aud}}$ are affective control signals,
$y^{\mathrm{id}}$ denotes the reference identity label or group identifier,
$y^{\mathrm{emo}}$ denotes the normalized emotion label, and
$g^{\mathrm{id}}$ denotes the reference-identity group used for leakage-safe splitting.

Each tuple is constructed to satisfy four requirements:
\begin{itemize}
    \item \textbf{Reference-grounded identity:} identity is specified through a visual reference image rather than inferred from text or audio;
    \item \textbf{Affective multimodal alignment:} image, text, and audio are aligned through shared emotion semantics;
    \item \textbf{Identity-control decoupling:} identity supervision and affective control signals are deliberately separated;
    \item \textbf{Leakage-safe evaluation:} all samples sharing the same reference-image identity anchor are assigned wholly to a single split.
\end{itemize}

This formulation makes AffectHuman-43K a benchmark for \textit{identity-preserving multimodal affective editing}, rather than a naturally paired multimodal identity dataset. This distinction is important: the benchmark is designed to isolate whether a model preserves visual identity while following affective controls from text and audio.

\begin{algorithm}[h]
\small
\caption{AffectHuman-43K Construction Pipeline}
\label{alg:affecthuman_construction}

\textbf{Input:} Image dataset $\mathcal{D}_{\mathrm{img}}$, text dataset $\mathcal{D}_{\mathrm{text}}$, audio dataset $\mathcal{D}_{\mathrm{aud}}$, identity threshold $\tau_{\mathrm{id}}$, consistency threshold $\tau_{\mathrm{cons}}$, duplicate threshold $\tau_{\mathrm{dup}}$ \\
\textbf{Output:} Leakage-safe aligned multimodal benchmark $\mathcal{S}^{*}$

\hrule
\vspace{2pt}

\begin{algorithmic}[1]
\State Extract face and identity embeddings from candidate reference images using ArcFace.
\State Define reference-identity anchors from visual reference images.
\State Group samples sharing the same reference-image anchor into reference-identity groups $g^{\mathrm{id}}$.
\State Verify intra-group identity consistency using threshold $\tau_{\mathrm{id}}$.
\State Map emotion annotations from all source datasets into a shared valence--arousal space.
\State Discretize valence--arousal coordinates into normalized emotion labels $y^{\mathrm{emo}}$.
\State Construct candidate tuples $(x^{\mathrm{img}}, x^{\mathrm{ref}}, x^{\mathrm{text}}, x^{\mathrm{aud}}, y^{\mathrm{id}}, y^{\mathrm{emo}})$.
\State Compute multimodal affective consistency:
\[
S_{\mathrm{cons}}
=
\alpha s(e_{\mathrm{img}}, e_{\mathrm{text}})
+
\beta s(e_{\mathrm{img}}, e_{\mathrm{aud}})
+
\gamma s(e_{\mathrm{text}}, e_{\mathrm{aud}}),
\]
where $e_{\mathrm{img}}$, $e_{\mathrm{text}}$, and $e_{\mathrm{aud}}$ are modality embeddings and $s(\cdot,\cdot)$ denotes similarity.
\State Retain candidate tuples satisfying $S_{\mathrm{cons}} > \tau_{\mathrm{cons}}$.
\State Remove near-duplicates and redundant samples using threshold $\tau_{\mathrm{dup}}$.
\State Assign each reference-identity group wholly to train, validation, or test.
\State Detect cross-split reference-identity conflicts and repair them by reassignment or removal.
\State Remove cross-split near-duplicates and leakage-risk samples.
\State Validate final splits for reference-identity leakage, near-duplicate overlap, and modality coverage.
\State \Return $\mathcal{S}^{*}$.
\end{algorithmic}
\end{algorithm}

\subsection{Stage-wise Construction}

The construction pipeline consists of five stages. Each stage enforces a structural property required for reliable evaluation of identity-preserving controllable generation. The initial multimodal alignment stage produces \textbf{43,514} candidate tuples. After leakage repair, duplicate removal, and split validation, the final benchmark contains \textbf{42,469} usable samples organized into \textbf{26,241 reference-identity groups}. The name AffectHuman-43K reflects the rounded scale of the final benchmark.

\paragraph{(1) Reference-Identity Grounding.}
Identity supervision is grounded through the visual reference image. For each candidate reference image, we extract identity embeddings using ArcFace and assign samples sharing the same reference-image anchor to a common reference-identity group. This creates a visual identity anchor that is independent of the text and audio modalities. The purpose of this step is not to infer a shared identity across image, text, and audio sources; rather, it is to define a stable visual identity reference against which generated outputs can be evaluated.

This design is central to the benchmark. By grounding identity in $x^{\mathrm{ref}}$, AffectHuman-43K evaluates whether a model can preserve the reference identity while modifying affective attributes. It also avoids an unrealistic assumption that independently sourced text and audio samples correspond to the same person as the image.

\paragraph{(2) Emotion Normalization.}
Source datasets use different emotion taxonomies, annotation granularities, and modality-specific labels. To make them comparable, we map all emotion annotations into a shared valence--arousal representation. Continuous or categorical annotations are then normalized into a common affective taxonomy. In the final benchmark, each sample is assigned a normalized emotion label $y^{\mathrm{emo}}$.

This normalization serves two purposes. First, it enables semantic alignment across image, text, and audio sources without requiring identity correspondence. Second, it supports evaluation of affective controllability under both coarse emotion categories and fine-grained intra-class variation.

\paragraph{(3) Multimodal Affective Alignment.}
Because AffectHuman-43K is not a naturally co-recorded multimodal identity dataset, cross-modal alignment is defined through \textbf{affective consistency}, not person-level correspondence. For each candidate tuple, we compute a multimodal consistency score
\[
S_{\mathrm{cons}}
=
\alpha s(e_{\mathrm{img}}, e_{\mathrm{text}})
+
\beta s(e_{\mathrm{img}}, e_{\mathrm{aud}})
+
\gamma s(e_{\mathrm{text}}, e_{\mathrm{aud}}),
\]
where $e_{\mathrm{img}}$, $e_{\mathrm{text}}$, and $e_{\mathrm{aud}}$ denote modality embeddings, and $\alpha$, $\beta$, and $\gamma$ control the relative contribution of image--text, image--audio, and text--audio consistency. Candidate tuples with low affective agreement are removed.

Pretrained models such as CLIP and ImageBind are used only for filtering low-agreement candidate tuples during dataset construction. They are not used to define ground-truth emotion labels, to supervise Controlla, or to determine the final evaluation scores. This separation reduces the risk that benchmark construction and evaluation collapse into the same pretrained-model signal.

\paragraph{(4) Filtering and De-duplication.}
After affective alignment, we remove redundant, low-quality, and leakage-risk samples. Near-duplicate detection is performed using a similarity threshold $\tau_{\mathrm{dup}}$. Samples with conflicting or weak multimodal affective signals are discarded. This stage improves benchmark consistency and reduces the chance that models can exploit repeated samples rather than learning controllable transformations.

Before leakage repair, the aligned candidate pool contains \textbf{43,514} tuples. After filtering, duplicate removal, and split validation, \textbf{42,469} samples remain in the final benchmark.

\paragraph{(5) Leakage-Free Grouped Splitting.}
To evaluate identity preservation under true generalization, splits are created at the level of reference-identity groups rather than individual samples. Let $\mathcal{G}_{\mathrm{train}}$, $\mathcal{G}_{\mathrm{val}}$, and $\mathcal{G}_{\mathrm{test}}$ denote the sets of reference-identity groups assigned to each split. We enforce
\[
\mathcal{G}_{\mathrm{train}} \cap \mathcal{G}_{\mathrm{val}} = \emptyset,
\qquad
\mathcal{G}_{\mathrm{train}} \cap \mathcal{G}_{\mathrm{test}} = \emptyset,
\qquad
\mathcal{G}_{\mathrm{val}} \cap \mathcal{G}_{\mathrm{test}} = \emptyset.
\]
Thus, all samples sharing the same reference-image identity anchor appear in exactly one split. Cross-split conflicts are repaired through reassignment or removal, and cross-split near-duplicates are removed.

The final split contains \textbf{29,728} training samples, \textbf{6,371} validation samples, and \textbf{6,370} test samples. The final benchmark achieves \textbf{0.0\% cross-split reference-identity leakage}, zero near-duplicate overlap, and complete modality coverage across image, reference image, text, and audio fields.

\subsection{Benchmark Positioning and Scope}

\paragraph{A controlled benchmark, not a naturally co-recorded identity dataset.}
AffectHuman-43K should be interpreted as a controlled benchmark for reference-grounded multimodal affective control. It is not intended to claim naturally paired identity correspondence across image, text, and audio. Instead, its construction deliberately separates two factors:
\begin{itemize}
    \item \textbf{Identity preservation}, evaluated through a visual reference-image anchor;
    \item \textbf{Affective controllability}, specified through emotion-aligned text and audio signals.
\end{itemize}

This design allows the benchmark to test a precise question: given a reference identity and heterogeneous affective control signals, can a model generate outputs that follow the intended emotional transformation while preserving identity?

\paragraph{Relationship to existing datasets.}
Existing affective datasets typically provide image-only, text-only, or audio-only supervision. Some multimodal datasets contain paired signals, but they do not provide reference-identity-disjoint evaluation for controllable generation. AffectHuman-43K complements these resources by introducing:
\begin{itemize}
    \item reference-image-grounded identity supervision;
    \item identity-disjoint train, validation, and test splits;
    \item repeated-reference groups for evaluating identity consistency under multiple affective controls;
    \item text and audio affective control signals;
    \item leakage validation through grouped splitting and duplicate removal;
    \item benchmark metadata for identity group, source dataset, emotion label, and modality coverage.
\end{itemize}

The resulting benchmark is therefore suited to evaluating identity-preserving affective editing, controllable multimodal generation, and disentangled representation learning.

\paragraph{Why this design supports Controlla.}
The Controlla formulation decomposes latent representations into an identity component $z_{\mathrm{id}}$ and an attribute component $z_{\mathrm{attr}}$. AffectHuman-43K is designed to mirror this factorization. The reference image specifies identity and provides supervision for $z_{\mathrm{id}}$, while the text and audio signals provide affective control information that should be captured in $z_{\mathrm{attr}}$. By enforcing reference-identity-disjoint splits, the benchmark tests whether the learned factorization generalizes to unseen reference identities rather than memorizing identities observed during training.

\paragraph{Scope and limitations.}
AffectHuman-43K is designed for scenarios in which identity is visually specified and auxiliary modalities provide semantic control. It does not assume that the image, text, and audio originate from the same person. Cross-modal alignment is defined at the level of shared affective semantics rather than person-level identity. This limitation is an explicit design choice: it intentionally decouples visual identity from affective control signals, allowing us to evaluate whether a model preserves a reference identity while following heterogeneous affective controls. We therefore use AffectHuman-43K to evaluate reference-grounded controllability, not natural multimodal person recognition.

\subsection{Use of Pretrained Models and Evaluation Circularity}

AffectHuman-43K uses pretrained models during construction only for candidate filtering and affective consistency estimation. Specifically, CLIP and ImageBind are used to remove candidate tuples whose image, text, and audio signals show weak affective agreement. They are not used to generate emotion labels, define supervision targets, assign reference-identity groups, train Controlla, or determine the final benchmark splits.

Because related pretrained encoders are also commonly used as alignment metrics, we treat CLIP/ImageBind scores as \textit{auxiliary diagnostics} rather than primary evidence for the main claims. To mitigate evaluation circularity, we separate dataset-construction signals from the metrics used to support controllability, identity preservation, and latent-geometry claims.

\paragraph{Task labels independent of filtering models.}
Emotion accuracy is computed using normalized labels inherited from the source datasets and mapped into the common affective taxonomy. These labels are not generated by CLIP or ImageBind.

\paragraph{Identity evaluation independent of affective alignment.}
Identity preservation is evaluated using ArcFace-based identity similarity and verification metrics. ArcFace is used for identity evaluation and reference-identity validation, but it is not used for CLIP/ImageBind-based multimodal affective filtering.

\paragraph{Geometry-aware latent metrics.}
Latent Disentanglement Score (LDS) and Geodesic Consistency (GC) are computed from learned latent factors and graph-induced distances. They do not rely on CLIP or ImageBind embeddings.

\paragraph{Human evaluation.}
Human evaluation is conducted using blinded pairwise comparisons. Method names are hidden, output order is randomized, and results are reported only in aggregate, providing a perceptual signal independent of automatic embedding-based alignment scores.

\paragraph{Held-out and source-specific evaluation.}
Evaluation is performed on held-out reference-identity groups, with additional source-specific or transfer analyses. This tests whether performance generalizes beyond dataset-specific filtering artifacts.

\subsection{Graph Construction Details}
\label{sec:graph_construction}

Controlla uses two graph priors: an emotion graph $G_e$ and an identity graph $G_i$. These graphs define the relational structures used by graph-constrained optimal transport. The emotion graph constrains the attribute subspace $z_{\mathrm{attr}}$, while the identity graph constrains the identity subspace $z_{\mathrm{id}}$.

\paragraph{Emotion graph $G_e$.}
The emotion graph models semantic relationships among affective states. We construct
\[
G_e = (V_e, E_e, W_e),
\]
where $V_e$ denotes affective prototype nodes, $E_e$ denotes edges between affective states, and $W_e$ denotes edge weights.

\textbf{Node definition.}
Emotion labels from all source datasets are mapped into a unified valence--arousal space. Although endpoint evaluation reports accuracy over the normalized 8-class emotion taxonomy, the graph prior is constructed over fine-grained valence--arousal prototypes obtained by clustering the continuous affective coordinates. Specifically, we use $N_e=32$ fine-grained affective prototypes for graph construction, while the 8-class taxonomy is used only for endpoint emotion accuracy. Each node $v_i \in V_e$ is represented by a centroid embedding $e_i$ in the valence--arousal space, enabling trajectories both within and across coarse emotion classes.

\textbf{Edge construction.}
We connect each affective prototype node to its nearest neighbors in valence--arousal space using a $k$-nearest neighbor graph:
\[
k_e = \min(k, N_e - 1).
\]
This produces a semantic graph over fine-grained affective states rather than only a dense graph over the 8 endpoint labels. In sensitivity experiments, we vary the number of prototypes, neighborhood size, and graph sparsity to verify that performance is not dependent on a single graph construction choice.

\textbf{Edge weights.}
For connected nodes $v_i$ and $v_j$, edge weights are computed using cosine similarity:
\[
w^{e}_{ij}
=
\frac{\langle e_i, e_j \rangle}
{\|e_i\|_2 \|e_j\|_2}.
\]
The graph is symmetrized by setting
\[
w^{e}_{ij} = w^{e}_{ji},
\]
and weights are normalized to $[0,1]$.

\textbf{Distance matrix.}
To convert similarities into distances, we define edge lengths as
\[
\ell^{e}_{ij}
=
\frac{1}{w^{e}_{ij}+\epsilon},
\]
where $\epsilon$ is a small constant for numerical stability. The emotion distance matrix $D_e$ is computed using shortest-path distances over the weighted graph. This matrix defines the semantic geometry that Controlla aligns with the attribute latent space.

\paragraph{Identity graph $G_i$.}
The identity graph models relationships among reference identities and encourages stable identity representations. We construct
\[
G_i = (V_i, E_i, W_i),
\]
where each node corresponds to a reference-identity group.

\textbf{Node definition.}
For each reference-identity group, we extract ArcFace embeddings from the corresponding reference images and compute a group centroid. During training, the identity graph is constructed only from training reference-identity groups; validation and test identity embeddings are used only for evaluation-time metric computation. The final benchmark contains \textbf{26,241 reference-identity groups}; therefore, the identity graph contains one node per group. Repeated groups provide multiple affective controls for the same reference anchor, while singleton groups support unseen-reference evaluation.

\textbf{Edge construction.}
We construct a $k$-nearest neighbor graph over identity centroids using cosine similarity, with $k_i=10$ unless otherwise specified. This connects visually similar reference identities while preserving separation between dissimilar identity groups.

\textbf{Edge weights.}
For identity nodes $u_i$ and $u_j$, edge weights are computed as
\[
w^{i}_{ij}
=
\frac{\langle h_i, h_j \rangle}
{\|h_i\|_2 \|h_j\|_2},
\]
where $h_i$ and $h_j$ are ArcFace centroid embeddings. The graph is symmetrized and weights are normalized to $[0,1]$.

\textbf{Distance matrix.}
Identity edge lengths are defined as
\[
\ell^{i}_{ij}
=
\frac{1}{w^{i}_{ij}+\epsilon}.
\]
The identity distance matrix $D_i$ is then obtained by computing shortest-path distances over the weighted identity graph.

\paragraph{Role in Controlla.}
The emotion graph $G_e$ defines the desired geometry of semantic traversal in $z_{\mathrm{attr}}$, while the identity graph $G_i$ defines relational structure in $z_{\mathrm{id}}$. During training, FGW alignment encourages attribute latent distances to reflect the fine-grained affective distance matrix $D_e$, while GW alignment encourages identity latent distances to reflect $D_i$. Thus, endpoint metrics use the 8-class taxonomy, whereas trajectory-level metrics such as GC evaluate movement over the fine-grained affective graph. Graph structure is therefore not used as a post-hoc regularizer; it defines the metric geometry under which controllability is learned.

\paragraph{Robustness to graph construction.}
We evaluate sensitivity to graph sparsity, neighborhood size, number of affective prototypes, and edge weighting. Empirically, performance remains stable across reasonable graph construction choices, indicating that Controlla benefits from relational structure rather than overfitting to a specific value of $k$, prototype count, or graph topology.


\section{Dataset Statistics and Benchmark Profile (AffectHuman-43K)}
\label{sec3}

This section reports the final benchmark statistics for \textbf{AffectHuman-43K}. While Sec.~B describes the construction and alignment protocol, this section focuses on the resulting dataset profile: sample scale, split composition, reference-identity grouping, emotion coverage, modality completeness, and leakage validation. These statistics support the use of AffectHuman-43K as a controlled benchmark for reference-grounded, identity-preserving multimodal affective control.

\subsection{Final Benchmark Summary}

As shown in Table~\ref{tab:affecthuman_summary}, AffectHuman-43K contains \textbf{42,469} final usable multimodal samples after filtering, leakage repair, and split validation. The benchmark is constructed from \textbf{43,514} initially aligned candidate tuples, of which \textbf{1,045} are removed during repair and validation. The name AffectHuman-43K reflects the rounded scale of the final benchmark. Each sample contains a target image, a visual reference image, a text affective signal, an audio affective signal, a normalized emotion label, and a reference-identity group identifier.

The benchmark is organized around \textbf{26,241 reference-identity groups}, including \textbf{10,204} repeated groups and \textbf{16,037} singleton groups, as summarized in Table~\ref{tab:affecthuman_summary}. A reference-identity group is defined by a shared visual reference-image anchor. This grouping is used only for identity-preserving evaluation and split construction; it does not assume that image, text, and audio originate from the same person. Instead, identity is specified visually through the reference image, while text and audio provide affective control signals.

\begin{table}[t]
\centering
\caption{\textbf{AffectHuman-43K benchmark summary.} Final statistics after filtering, duplicate removal, and leakage-safe split validation.}
\label{tab:affecthuman_summary}
\small
\setlength{\tabcolsep}{6pt}
\renewcommand{\arraystretch}{1.15}
\begin{tabular}{l|c}
\rowcolor{black!12}
\textbf{Statistic} & \textbf{Value} \\
\midrule
Initial aligned candidate tuples & 43,514 \\
Final usable samples & 42,469 \\
Removed during repair / validation & 1,045 \\
Reference-identity groups & 26,241 \\
Repeated reference-identity groups & 10,204 \\
Singleton reference-identity groups & 16,037 \\
Multi-emotion repeated groups & 7,290 \\
Emotion classes & 8 \\
Modality coverage & 100\% \\
Cross-split reference-identity leakage & 0.0\% \\
Cross-split near-duplicate overlap & 0 \\
\bottomrule
\end{tabular}
\end{table}

\subsection{Split Composition}

AffectHuman-43K uses reference-identity-disjoint train, validation, and test splits. All samples sharing the same reference-image anchor are assigned wholly to a single split. This prevents the same visual reference identity from appearing in both training and evaluation splits.

As shown in Table~\ref{tab:affecthuman_split_stats}, the final split contains \textbf{29,728} training samples, \textbf{6,371} validation samples, and \textbf{6,370} test samples. The split ratio is approximately 70/15/15 at the sample level, while preserving group-level exclusivity.

\begin{table}[t]
\centering
\caption{\textbf{Leakage-safe split statistics.} Splits are created at the reference-identity group level, not at the individual-sample level.}
\label{tab:affecthuman_split_stats}
\small
\setlength{\tabcolsep}{6pt}
\renewcommand{\arraystretch}{1.15}
\begin{tabular}{l|c|c}
\rowcolor{black!12}
\textbf{Split} & \textbf{\# Samples} & \textbf{Approx. Share} \\
\midrule
Train & 29,728 & 70.0\% \\
Validation & 6,371 & 15.0\% \\
Test & 6,370 & 15.0\% \\
\midrule
Total & 42,469 & 100.0\% \\
\bottomrule
\end{tabular}
\end{table}

\paragraph{Evaluation implication.}
Because split assignment is performed at the reference-identity group level, the split statistics in Table~\ref{tab:affecthuman_split_stats} indicate that validation and test performance measure generalization to unseen reference identities rather than memorization of identities seen during training. This is important for identity-preserving generation, where leakage can otherwise inflate identity similarity and controllability scores.

\subsection{Reference-Identity Group Profile}

As shown in Table~\ref{tab:identity_group_stats}, the benchmark contains \textbf{26,241} reference-identity groups. Of these, \textbf{16,037} are singleton groups and \textbf{10,204} are repeated groups. Singleton groups contain one usable sample for a reference anchor and support unseen-reference evaluation. Repeated groups contain multiple samples sharing the same reference anchor and support evaluation of identity consistency across different affective controls.

Among the repeated groups, Table~\ref{tab:identity_group_stats} shows that \textbf{7,290} contain multiple emotion labels. These groups are especially important for evaluating whether a model can preserve the same reference identity while changing emotional attributes.

\begin{table}[t]
\centering
\caption{\textbf{Reference-identity group statistics.} Repeated groups enable identity-consistency evaluation under multiple affective controls, while singleton groups support unseen-reference generalization.}
\label{tab:identity_group_stats}
\small
\setlength{\tabcolsep}{6pt}
\renewcommand{\arraystretch}{1.15}
\begin{tabular}{l|c}
\rowcolor{black!12}
\textbf{Group Type} & \textbf{Count} \\
\midrule
Total reference-identity groups & 26,241 \\
Singleton groups & 16,037 \\
Repeated groups & 10,204 \\
Repeated groups with multiple emotion labels & 7,290 \\
\bottomrule
\end{tabular}
\end{table}

\paragraph{Why repeated groups matter.}
Repeated reference-identity groups in Table~\ref{tab:identity_group_stats} provide a direct test of identity-preserving controllability: the reference identity should remain stable while the requested affective state changes. This supports evaluation beyond single-output identity similarity by allowing consistency analysis across multiple generated outputs for the same reference anchor.

\paragraph{Why singleton groups matter.}
Singleton groups prevent the benchmark from being dominated by repeated identities. They also test whether a model can preserve identity from a single reference image, which is a common real-world use case for reference-guided editing.

\subsection{Emotion-Class Distribution}

AffectHuman-43K uses an 8-class normalized emotion taxonomy:
\[
\{\text{Happiness, Sadness, Anger, Fear, Surprise, Disgust, Contempt, Neutral}\}.
\]
The distribution is moderately balanced, with natural variation across classes. The largest classes are Surprise, Neutral, and Happiness, while Contempt and Fear are smaller but retained to preserve affective diversity, as illustrated in Fig.~\ref{fig:emotion_distribution}.

\begin{figure}[t]
\centering
\includegraphics[width=0.92\linewidth]{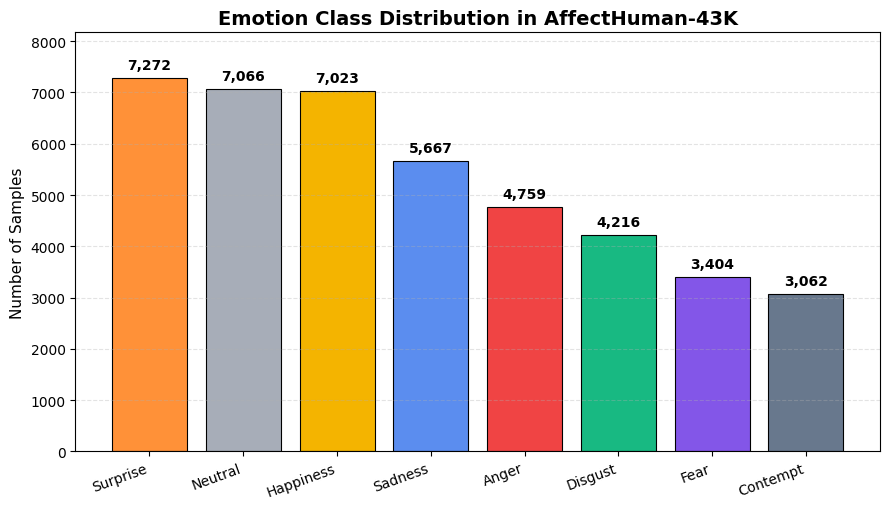}
\caption{\textbf{Emotion distribution across 8 classes.} AffectHuman-43K provides broad affective coverage with mild natural skew across emotion categories.}
\label{fig:emotion_distribution}
\end{figure}

Table~\ref{tab:dataset_emotion_stats} reports the class-level counts and representative fine-grained variations associated with each primary emotion category. The table shows that the benchmark supports both coarse endpoint evaluation over the normalized 8-class taxonomy and fine-grained affective variation within each class.

\begin{table}[t]
\centering
\caption{\textbf{Emotion-class coverage.} Distribution across the final leakage-safe AffectHuman-43K benchmark, with representative fine-grained variations.}
\small
\setlength{\tabcolsep}{5pt}
\renewcommand{\arraystretch}{1.2}
\begin{tabular}{l|c|c}
\rowcolor{black!12}
\textbf{Primary Class} & \textbf{\# Samples} & \textbf{Representative Fine-grained Variations} \\
\midrule
\rowcolor[rgb]{0.98,0.86,0.35}
Happiness & 7,023 & joy, excitement, amusement, contentment, hope \\
\rowcolor[rgb]{0.63,0.75,0.95}
Sadness & 5,667 & sadness, gloom, sorrow, despair, melancholy \\
\rowcolor[rgb]{0.98,0.60,0.60}
Anger & 4,759 & anger, rage, irritation, frustration, annoyance \\
\rowcolor[rgb]{0.73,0.62,0.95}
Fear & 3,404 & fear, anxiety, nervousness, panic, dread \\
\rowcolor[rgb]{0.98,0.73,0.50}
Surprise & 7,272 & surprise, astonishment, curiosity, anticipation, suspense \\
\rowcolor[rgb]{0.55,0.84,0.71}
Disgust & 4,216 & disgust, aversion, revulsion, distaste \\
\rowcolor[rgb]{0.67,0.72,0.80}
Contempt & 3,062 & contempt, disdain, resentment, cynicism \\
\rowcolor[rgb]{0.80,0.83,0.88}
Neutral & 7,066 & neutral, calm, relaxed, composed, stillness \\
\bottomrule
\end{tabular}
\label{tab:dataset_emotion_stats}
\end{table}

\paragraph{Implication for controllability evaluation.}
The class distribution in Fig.~\ref{fig:emotion_distribution} and Table~\ref{tab:dataset_emotion_stats} supports both endpoint and trajectory-level evaluation. Endpoint metrics measure whether a generated output reaches the target affective class, while geometry-aware metrics evaluate whether transitions between affective states follow structured latent paths. Fine-grained variations within each class make the task more challenging than coarse emotion classification alone.

\subsection{Modality Coverage and Sample Fields}

Every final sample in AffectHuman-43K contains complete modality fields required for reference-grounded multimodal affective control. The benchmark provides visual identity through a reference image, visual content through the target image, and affective control through text and audio.

Each released sample contains the following fields:
\[
\left(
x^{\mathrm{img}},
x^{\mathrm{ref}},
x^{\mathrm{text}},
x^{\mathrm{aud}},
y^{\mathrm{emo}},
g^{\mathrm{id}},
s^{\mathrm{src}},
r^{\mathrm{split}}
\right),
\]
where $s^{\mathrm{src}}$ denotes source metadata and $r^{\mathrm{split}}$ denotes the assigned split.

\begin{table}[t]
\centering
\caption{\textbf{Released sample fields.} Each sample contains complete modality and metadata fields required for controlled evaluation.}
\label{tab:sample_fields}
\small
\setlength{\tabcolsep}{5pt}
\renewcommand{\arraystretch}{1.15}
\begin{tabular}{l|p{0.63\linewidth}}
\rowcolor{black!12}
\textbf{Field} & \textbf{Description} \\
\midrule
$x^{\mathrm{img}}$ & Target image used for affective generation or editing. \\
$x^{\mathrm{ref}}$ & Visual reference image specifying identity. \\
$x^{\mathrm{text}}$ & Textual affective control signal. \\
$x^{\mathrm{aud}}$ & Audio affective control signal. \\
$y^{\mathrm{emo}}$ & Normalized emotion label in the 8-class taxonomy. \\
$g^{\mathrm{id}}$ & Reference-identity group identifier used for grouped splitting. \\
$s^{\mathrm{src}}$ & Source-dataset metadata for auditing and source-specific evaluation. \\
$r^{\mathrm{split}}$ & Train, validation, or test split assignment. \\
\bottomrule
\end{tabular}
\end{table}

\paragraph{Modality completeness.}
As summarized in Table~\ref{tab:sample_fields}, all \textbf{42,469} final samples contain image, reference image, text, audio, normalized emotion label, and reference-identity group metadata. This ensures that every sample can be used for multimodal affective control evaluation without missing-modality handling.

\subsection{Leakage and Duplicate Validation}

AffectHuman-43K is validated for three major leakage risks: reference-identity leakage, cross-split near-duplicate overlap, and cross-dataset near-duplicate overlap across source datasets. Reference-identity leakage is prevented by assigning each reference-identity group wholly to one split. Near-duplicate overlap is removed by similarity-based duplicate filtering. Source metadata is retained to support source-specific evaluation and auditing.

Let $\mathcal{G}_{\mathrm{train}}$, $\mathcal{G}_{\mathrm{val}}$, and $\mathcal{G}_{\mathrm{test}}$ denote the sets of reference-identity groups in each split. The final benchmark satisfies:
\[
\mathcal{G}_{\mathrm{train}} \cap \mathcal{G}_{\mathrm{val}} = \varnothing,
\qquad
\mathcal{G}_{\mathrm{train}} \cap \mathcal{G}_{\mathrm{test}} = \varnothing,
\qquad
\mathcal{G}_{\mathrm{val}} \cap \mathcal{G}_{\mathrm{test}} = \varnothing.
\]
Table~\ref{tab:split_validation} reports the resulting validation checks, confirming zero reference-identity overlap, zero cross-split near-duplicate overlap, and no missing modality fields.

\begin{table}[t]
\centering
\caption{\textbf{Split validation results.} The final benchmark is validated for group-level identity leakage and near-duplicate overlap.}
\label{tab:split_validation}
\small
\setlength{\tabcolsep}{6pt}
\renewcommand{\arraystretch}{1.15}
\begin{tabular}{l|c}
\rowcolor{black!12}
\textbf{Validation Check} & \textbf{Final Result} \\
\midrule
Train--validation reference-identity overlap & 0 \\
Train--test reference-identity overlap & 0 \\
Validation--test reference-identity overlap & 0 \\
Cross-split reference-identity leakage rate & 0.0\% \\
Cross-split near-duplicate overlap & 0 \\
Missing modality fields & 0 \\
\bottomrule
\end{tabular}
\end{table}

\paragraph{Why this validation matters.}
The zero-overlap checks in Table~\ref{tab:split_validation} ensure that the same reference identity cannot appear in both training and test sets. Without grouped splitting, identity-preservation metrics could overestimate generalization. The validation protocol therefore ensures that evaluation reflects performance on unseen reference identities.

\subsection{Source and Modality Notes}

AffectHuman-43K combines public image, text, and audio sources into a controlled evaluation benchmark. The source datasets differ in annotation format and modality type; therefore, source metadata is preserved for auditing, source-specific analysis, and future benchmark extensions.

\paragraph{Image modality.}
The image modality provides both target visual content and reference identity anchors. Face-centric and affective image datasets contribute visual samples with emotion annotations or identity-relevant reference images.

\paragraph{Text modality.}
Text samples provide affective semantic controls. These may correspond to utterances, affective descriptions, or emotion-bearing annotations normalized into the shared taxonomy.

\paragraph{Audio modality.}
Audio samples are drawn from real human speech data and provide affective prosodic cues. Each final sample includes an audio waveform or audio reference together with a corresponding precomputed feature representation for efficient evaluation.

\paragraph{Source-specific evaluation.}
Because source metadata is retained, models can be evaluated not only on the full benchmark but also on source-specific subsets. This helps test whether performance is robust across different dataset origins rather than dependent on a single source distribution.

\subsection{Ethics and Data Usage}

All samples are derived from publicly available datasets under their respective research licenses. AffectHuman-43K does not introduce newly scraped private data and does not collect new personally identifiable information. The benchmark is intended only for research on controllable multimodal generation, identity-preserving editing, and affective representation learning.

Because the benchmark contains human faces and speech, release and use should follow the licenses and ethical constraints of the underlying datasets. We recommend that downstream users avoid applications involving impersonation, identity misuse, surveillance, or deceptive generation. Evaluation results should be reported in aggregate, and generated examples should be used only in research contexts consistent with the source-data licenses.

\section{Evaluation Metrics Details and Bias Separation}
\label{sec4}

Standard metrics such as emotion accuracy, identity similarity, CLIP similarity, and ImageBind similarity evaluate endpoint behavior: whether a generated output reaches the target emotion, preserves identity, or aligns with text/audio inputs. However, Controlla also makes a representation-level claim: controllable generation should follow a disentangled and graph-consistent latent path. We therefore evaluate both output-level behavior and latent-geometry structure.

Because CLIP/ImageBind-related encoders are used during candidate filtering in benchmark construction, we separate primary evidence from auxiliary alignment diagnostics. As summarized in Table~\ref{tab:metric_bias_separation}, the main claims about emotion control, identity preservation, latent disentanglement, and graph-consistent traversal are supported by metrics that do not use CLIP/ImageBind. CLIP and ImageBind are reported only for cross-modal alignment diagnostics, not as primary evidence for controllability.

\begin{table}[h]
\centering
\caption{
\textbf{Metric separation for avoiding evaluation circularity.}
Primary claims are supported by metrics that do not use CLIP/ImageBind. CLIP/ImageBind are reported only as auxiliary alignment diagnostics.
}
\label{tab:metric_bias_separation}
\small
\rowcolors{2}{blue!5}{white}
\begin{tabular}{l|l|c}
\toprule
\rowcolor{blue!15}
\textbf{Claim} & \textbf{Primary Evidence} & \textbf{Uses CLIP/ImageBind?} \\
\midrule
Emotion control & Acc, human evaluation & No \\
Identity preservation & ArcFace ID-Sim / AUC & No \\
Latent disentanglement & LDS & No \\
Graph-consistent traversal & GC, TS & No \\
Cross-modal alignment & CLIP / ImageBind & Yes, auxiliary only \\
\bottomrule
\end{tabular}
\end{table}

\subsection{Primary Output-Level Metrics}

\paragraph{Emotion Accuracy.}
Emotion Accuracy (Acc) measures whether the generated output $\hat{x}_i$ matches the target emotion label $y^{\mathrm{emo}}_i$. Let $c_{\psi}(\hat{x}_i)$ denote an emotion classifier or evaluator independent of CLIP/ImageBind. We compute
\begin{equation}
\mathrm{Acc}
=
\frac{1}{N}
\sum_{i=1}^{N}
\mathbb{I}
\left[
c_{\psi}(\hat{x}_i)=y^{\mathrm{emo}}_i
\right].
\label{eq:acc}
\end{equation}
Acc is a primary metric for Controlla because it evaluates whether attribute traversal reaches the intended affective endpoint.

\paragraph{Identity Similarity and Verification AUC.}
Identity preservation is measured using a pretrained face identity encoder $\phi_{\mathrm{id}}$, implemented with ArcFace. Given the generated output $\hat{x}_i$ and reference image $x^{\mathrm{ref}}_i$, identity similarity is
\begin{equation}
\mathrm{ID\text{-}Sim}
=
\frac{1}{N}
\sum_{i=1}^{N}
\frac{
\left\langle
\phi_{\mathrm{id}}(\hat{x}_i),
\phi_{\mathrm{id}}(x^{\mathrm{ref}}_i)
\right\rangle
}{
\left\|\phi_{\mathrm{id}}(\hat{x}_i)\right\|_2
\left\|\phi_{\mathrm{id}}(x^{\mathrm{ref}}_i)\right\|_2
}.
\label{eq:idsim}
\end{equation}
Verification AUC is computed by thresholding these identity similarities over matched and non-matched reference pairs:
\begin{equation}
\mathrm{AUC}
=
\int_{0}^{1}
\mathrm{TPR}(\mathrm{FPR})\,d\mathrm{FPR}.
\label{eq:auc}
\end{equation}
ID-Sim and AUC are primary metrics because Controlla explicitly holds $z_{\mathrm{id}}$ fixed during attribute traversal.

\paragraph{Trajectory Smoothness.}
Trajectory Smoothness (TS) measures whether intermediate latent states change smoothly during attribute traversal. For a trajectory $\mathcal{T}_{Z}=\{z_0,z_1,\ldots,z_T\}$, we define
\begin{equation}
\mathrm{TS}
=
1-
\frac{1}{T-1}
\sum_{t=1}^{T-1}
\frac{
\left\|
(z_{t+1}-z_t)-(z_t-z_{t-1})
\right\|_2
}{
\sum_{\tau=1}^{T}
\left\|z_{\tau}-z_{\tau-1}\right\|_2+\epsilon
}.
\label{eq:ts}
\end{equation}
TS is normalized so that higher values indicate smoother transitions. It is relevant to Controlla because graph-constrained traversal should avoid abrupt or unstable latent jumps.

\paragraph{Human Preference.}
Human preference (H) measures perceptual preference over identity preservation, target emotion correctness, and visual plausibility. Given $M$ blinded pairwise or rating-based judgments with score $s_m$, we report
\begin{equation}
H
=
\frac{1}{M}
\sum_{m=1}^{M}
s_m.
\label{eq:human}
\end{equation}
Human evaluation is a primary perceptual signal independent of CLIP/ImageBind-based filtering.

\subsection{Latent Disentanglement Score (LDS)}

Let $z=f_\theta(x)$ denote the shared multimodal representation. Controlla learns two factorization heads,
\[
z_{\mathrm{attr}} = h_{\mathrm{attr}}(z),
\qquad
z_{\mathrm{id}} = h_{\mathrm{id}}(z),
\]
where $z_{\mathrm{attr}}$ captures controllable affective attributes and $z_{\mathrm{id}}$ captures reference-grounded identity information. A desirable factorization should allow affective changes in $z_{\mathrm{attr}}$ without inducing unintended variation in $z_{\mathrm{id}}$.

Given a batch of $B$ samples, let
\[
Z_{\mathrm{attr}}\in\mathbb{R}^{B\times d_{\mathrm{attr}}},
\qquad
Z_{\mathrm{id}}\in\mathbb{R}^{B\times d_{\mathrm{id}}}
\]
denote centered attribute and identity factors. We compute their empirical cross-covariance:
\begin{equation}
\mathrm{Cov}(Z_{\mathrm{attr}},Z_{\mathrm{id}})
=
\frac{1}{B-1}Z_{\mathrm{attr}}^\top Z_{\mathrm{id}}.
\label{eq:cross_cov_matrix}
\end{equation}

The cross-factor dependence penalty is
\begin{equation}
C_{\mathrm{cross}}
=
\left\|
\mathrm{Cov}(Z_{\mathrm{attr}},Z_{\mathrm{id}})
\right\|_F .
\label{eq:lds_cross_cov}
\end{equation}

We define LDS as
\begin{equation}
\mathrm{LDS}
=
1-
\mathrm{clip}
\left(
\frac{
C_{\mathrm{cross}}
}{
C_{\mathrm{cross}}
+
\left\|\mathrm{Cov}(Z_{\mathrm{attr}})\right\|_F
+
\left\|\mathrm{Cov}(Z_{\mathrm{id}})\right\|_F
+
\epsilon
},
0,1
\right),
\label{eq:lds}
\end{equation}
where $\epsilon$ is a small constant for numerical stability.

LDS lies in $[0,1]$, with higher values indicating weaker linear dependence between learned attribute and identity factors. LDS is a primary diagnostic metric for Controlla because the method explicitly relies on separating $z_{\mathrm{attr}}$ from $z_{\mathrm{id}}$. It does not claim full statistical independence; instead, it evaluates whether the learned representation reduces identity--attribute leakage.

\subsection{Geodesic Consistency (GC)}

GC measures whether movement in the learned attribute factor follows the semantic geometry induced by the emotion graph. Let
\[
\mathcal{T}_{Z}=\{z_0,z_1,\ldots,z_T\}
\]
denote a trajectory in the learned attribute factor space, and let
\[
\mathcal{T}_{G}=\{u_0,u_1,\ldots,u_T\}
\]
denote the corresponding sequence of emotion-graph nodes. Each $u_t$ is either specified by the target graph path or assigned by nearest-neighbor matching from $z_t$ to the closest emotion node in the attribute-factor embedding space.

Let $d_{\mathcal{G}}(u_{t-1},u_t)$ denote the shortest-path distance between consecutive graph nodes. We normalize graph progress as
\begin{equation}
\tilde{d}_{\mathcal{G}}(u_{t-1},u_t)
=
\frac{
d_{\mathcal{G}}(u_{t-1},u_t)
}{
\sum_{\tau=1}^{T}d_{\mathcal{G}}(u_{\tau-1},u_{\tau})+\epsilon
}.
\label{eq:normalized_graph_distance}
\end{equation}

We normalize latent progress as
\begin{equation}
\tilde{d}_{\mathcal{Z}}(z_{t-1},z_t)
=
\frac{
\|z_t-z_{t-1}\|_2
}{
\sum_{\tau=1}^{T}\|z_{\tau}-z_{\tau-1}\|_2+\epsilon
}.
\label{eq:normalized_latent_distance}
\end{equation}

Geodesic Consistency is defined as
\begin{equation}
\mathrm{GC}
=
\frac{1}{T}
\sum_{t=1}^{T}
\left|
\tilde{d}_{\mathcal{G}}(u_{t-1},u_t)
-
\tilde{d}_{\mathcal{Z}}(z_{t-1},z_t)
\right|.
\label{eq:gc}
\end{equation}

Lower GC indicates that latent movement is more consistent with the graph-induced semantic path. GC is central to Controlla because the method claims that controllability is induced through graph-consistent latent traversal, not only through endpoint conditioning.

\subsection{Auxiliary Cross-modal Alignment Metrics}

We report CLIP and ImageBind similarities only as auxiliary diagnostics for image--text and image--audio alignment. For image--text alignment, we compute
\begin{equation}
\mathrm{CLIP}
=
\frac{1}{N}
\sum_{i=1}^{N}
\frac{
\left\langle
\phi_{\mathrm{img}}^{\mathrm{CLIP}}(\hat{x}_i),
\phi_{\mathrm{text}}^{\mathrm{CLIP}}(x^{\mathrm{text}}_i)
\right\rangle
}{
\left\|\phi_{\mathrm{img}}^{\mathrm{CLIP}}(\hat{x}_i)\right\|_2
\left\|\phi_{\mathrm{text}}^{\mathrm{CLIP}}(x^{\mathrm{text}}_i)\right\|_2
}.
\label{eq:clip}
\end{equation}

For image--audio alignment, we compute ImageBind similarity:
\begin{equation}
\mathrm{IB}
=
\frac{1}{N}
\sum_{i=1}^{N}
\frac{
\left\langle
\phi_{\mathrm{img}}^{\mathrm{IB}}(\hat{x}_i),
\phi_{\mathrm{aud}}^{\mathrm{IB}}(x^{\mathrm{aud}}_i)
\right\rangle
}{
\left\|\phi_{\mathrm{img}}^{\mathrm{IB}}(\hat{x}_i)\right\|_2
\left\|\phi_{\mathrm{aud}}^{\mathrm{IB}}(x^{\mathrm{aud}}_i)\right\|_2
}.
\label{eq:imagebind}
\end{equation}

Because CLIP/ImageBind-related encoders are used during benchmark filtering, these scores are not used as primary evidence for controllability, identity preservation, or latent-geometry claims. They are reported only to indicate whether outputs remain compatible with text and audio controls.

\subsection{Computation Protocol}

LDS is computed over validation and test batches containing multiple identities and emotion labels. For GC and TS, we construct trajectories by varying $z_{\mathrm{attr}}$ while holding $z_{\mathrm{id}}$ fixed, isolating affective control from identity variation. Given a source emotion $u_s$ and target emotion $u_t$, we compute a shortest path on the emotion graph,
\[
u_s \rightarrow u_1 \rightarrow \cdots \rightarrow u_t,
\]
and compare normalized graph progress with normalized latent progress along the corresponding attribute-factor trajectory.

All metrics are averaged over validation and test splits. We report
\[
\text{higher Acc, ID-Sim, AUC, TS, LDS, H} \Rightarrow \text{better},
\qquad
\text{lower GC} \Rightarrow \text{more graph-consistent traversal}.
\]
Degenerate trajectories with no graph or latent movement are excluded from GC aggregation or assigned a neutral diagnostic value according to the evaluation script.

\subsection{Relationship to Controlla's Claims}

The metrics are aligned with the core claims of Controlla. Acc and human evaluation test whether the generated output follows the target affect. ID-Sim and AUC test whether the reference-grounded identity factor remains stable. LDS evaluates whether attribute and identity factors are separated. TS and GC evaluate whether traversal in $z_{\mathrm{attr}}$ is smooth and graph-consistent. CLIP and ImageBind evaluate cross-modal alignment only as auxiliary diagnostics. Thus, the main claims of controllability, identity preservation, disentanglement, and graph-consistent traversal are supported by metrics that are independent of CLIP/ImageBind filtering.

\section{Human Evaluation Protocol}
\label{sec5}

We conduct a blinded human evaluation to assess emotion fidelity, identity preservation, multimodal consistency, and controllability of generated outputs. Evaluation is performed using a pairwise comparison setup between Controlla and strong baselines, including FLUX.1 Kontext, ICEdit, and ControlNet++. Each trial presents two outputs generated from the same input condition, with method names hidden and output order randomized. Participants select the preferred result under predefined criteria. The protocol includes both single-image comparisons and sequence-based evaluation to capture static output quality and transformation behavior.

Participants were recruited through academic and research networks for a low-risk perceptual evaluation. Evaluated samples were anonymized before presentation, and results are reported only in aggregate. The study did not collect personally identifying information.

Participants were recruited through university academic and research networks for a low-risk perceptual evaluation. Method names were hidden, image pairs were shown in randomized order, and evaluated samples were anonymized before presentation. The study did not collect sensitive personal information, and results are reported only in aggregate.

\subsection{Questionnaire Design}

The questionnaire is constructed to reflect the emotion taxonomy of AffectHuman-43K, which includes primary emotion classes and fine-grained variations within each class. Evaluation focuses on correctness of emotion, identity consistency, intra-class variation, multimodal alignment, and transition behavior. Question order and sample order were randomized to reduce ordering effects.

\paragraph{Emotion Fidelity.}
Participants evaluate whether generated expressions correctly reflect both primary and fine-grained emotions:
\begin{itemize}
\item Q1: Which image better expresses the target emotion?
\item Q2: Which image better matches the specified emotion label?
\item Q3: Which image better captures fine-grained variations within the same class?
\end{itemize}

\paragraph{Identity Preservation and Visual Quality.}
Participants assess whether identity remains consistent under transformation:
\begin{itemize}
\item Q4: Which image better preserves facial identity?
\item Q5: Which image better maintains identity under strong emotion changes?
\item Q6: Which image appears more visually coherent and realistic?
\end{itemize}

\paragraph{Fine-Grained Controllability.}
To evaluate intra-class structure, participants compare:
\begin{itemize}
\item Q7: Which image better distinguishes subtle emotional differences?
\item Q8: Which image better represents low-intensity states?
\item Q9: Which image better captures socially nuanced expressions?
\end{itemize}

\paragraph{Failure-Mode Robustness.}
We evaluate known failure cases in prior methods:
\begin{itemize}
\item Q10: Which image better avoids collapsing multiple emotions into a single expression?
\item Q11: Which image better differentiates overlapping high-arousal emotions?
\end{itemize}

\paragraph{Multimodal Consistency.}
Participants assess cross-modal alignment:
\begin{itemize}
\item Q12: Which image better reflects combined multimodal inputs?
\item Q13: Which image better handles ambiguous or partially inconsistent multimodal signals?
\item Q14: Which image better maintains consistent emotion across modalities?
\end{itemize}

\paragraph{Trajectory Consistency.}
For sequence-based evaluation, participants assess transformation behavior:
\begin{itemize}
\item Q15: Which sequence shows smoother transitions?
\item Q16: Which sequence follows a gradual progression?
\item Q17: Which sequence better preserves identity across all steps?
\end{itemize}

\noindent
The final questionnaire consists of 17 criteria.

\paragraph{Evaluation Size.}
We evaluate 120 single-image comparison pairs and 36 sequence-level comparison pairs, yielding 156 pairwise comparison trials in total. The single-image comparisons are balanced across the three baseline methods, with 40 comparison pairs per baseline. The sequence-level evaluation includes 12 trajectory comparisons per baseline, covering gradual emotional transitions such as neutral $\rightarrow$ happy $\rightarrow$ excited and neutral $\rightarrow$ sad $\rightarrow$ distressed. Each comparison pair is rated by at least seven independent participants, yielding a total of 1,092 pairwise judgments. Samples are balanced across emotion classes, identity groups, and compared baselines to avoid over-representing any single emotion, identity cluster, or method.

\subsection{Raters and Evaluation Setup}

A total of 96 participants completed the study. Participants were recruited from academic and research networks and included individuals with backgrounds in computer vision, HCI, and cognitive science. We collected only coarse demographic information for aggregate reporting and did not collect personally identifying information. Each image pair was evaluated by at least seven independent raters.

Participants used a 5-point preference scale:
\[
+2 \ (\text{strongly prefer Controlla}), \quad 0 \ (\text{no preference}), \quad -2 \ (\text{prefer baseline}).
\]

\subsection{Preference Scale and Aggregation}

Participants rated preferences using a 5-point scale summarized in Table~\ref{tab:pref_scale}. For each criterion, scores are averaged across raters to obtain a mean preference value. Statistical significance is evaluated using the Wilcoxon signed-rank test with Bonferroni correction across all criteria, with $p < 0.01$ considered significant.

\begin{table}[h]
\centering
\caption{\textbf{Preference scale used in human evaluation.} Positive scores indicate preference for Controlla, while negative scores indicate preference for baseline methods.}
\small
\setlength{\tabcolsep}{6pt}
\renewcommand{\arraystretch}{1.15}

\rowcolors{2}{blue!4}{white}
\begin{tabular}{c|l}
\rowcolor{blue!15}
\textbf{Score} & \textbf{Interpretation} \\
\toprule
+2 & Strongly prefer Controlla \\
+1 & Prefer Controlla \\
0  & No preference \\
-1 & Prefer baseline \\
-2 & Strongly prefer baseline \\
\bottomrule
\end{tabular}
\label{tab:pref_scale}
\end{table}

\subsection{Trajectory-Based Evaluation}

In addition to single-image comparisons, we include sequence-based evaluation where participants observe ordered samples corresponding to gradual emotional transitions (e.g., neutral $\rightarrow$ happy $\rightarrow$ excited). 

Participants assess:
(i) smoothness of transitions,  
(ii) consistency of progression, and  
(iii) identity preservation across steps.

\begin{figure}[h]
    \centering
    \includegraphics[width=\linewidth]{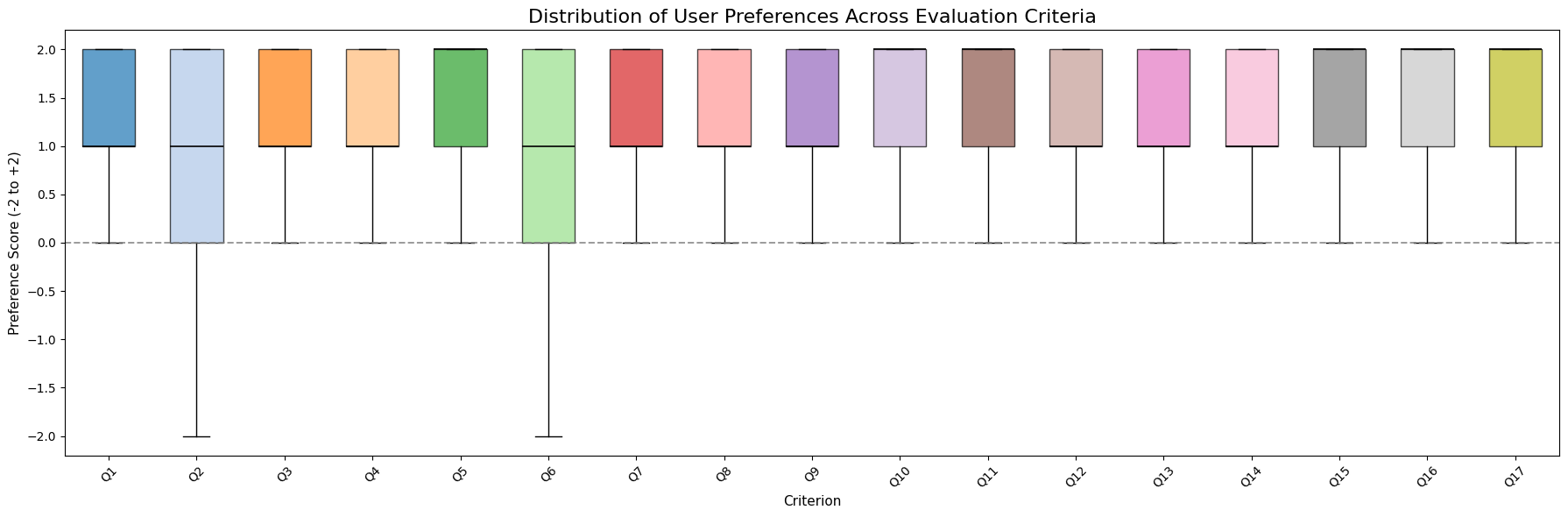}
    \caption{\textbf{Preference boxplots.} Scores compare Controlla with baselines across evaluation criteria.}
    \label{fig:supp4_boxplot}
\end{figure}

\begin{figure}[h]
    \centering
    \includegraphics[width=\linewidth]{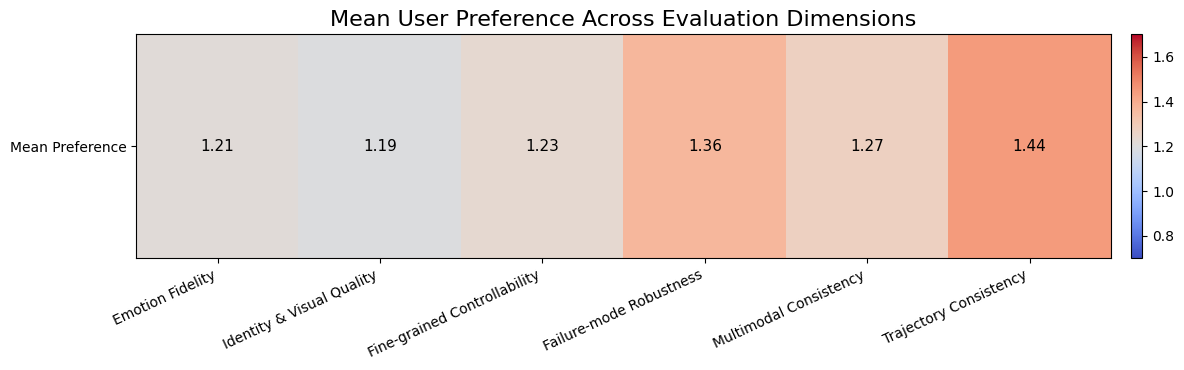}
    \caption{\textbf{Mean preference heatmap.} Average scores summarize Controlla's gains across evaluation dimensions.}
    \label{fig:supp5_heatmap}
\end{figure}

\subsection{Results and Analysis}

Across evaluation dimensions, Controlla receives consistently positive preference scores.
The boxplots in Fig.~\ref{fig:supp4_boxplot} show that most criteria remain above the neutral line, indicating stable preference for Controlla over baselines.
The heatmap in Fig.~\ref{fig:supp5_heatmap} further shows stronger gains for emotion fidelity, identity preservation, multimodal consistency, and trajectory smoothness.
The violin plots in Fig.~\ref{fig:supp6_violin} illustrate response variability: emotion and identity-related criteria are more concentrated, while fine-grained and multimodal criteria show broader but still positive distributions.
The ranked-bar plot in Fig.~\ref{fig:supp7_rankedbar} shows that the largest gains occur in trajectory consistency, identity preservation, and fine-grained emotion differentiation, consistent with the geometry-aware metrics in the main paper.
Finally, Fig.~\ref{fig:supp8_genderbox} shows that preference trends remain positive across demographic groups, suggesting that the observed improvements are not driven by a single participant subgroup.

\begin{figure}[h]
    \centering
    \includegraphics[width=\linewidth]{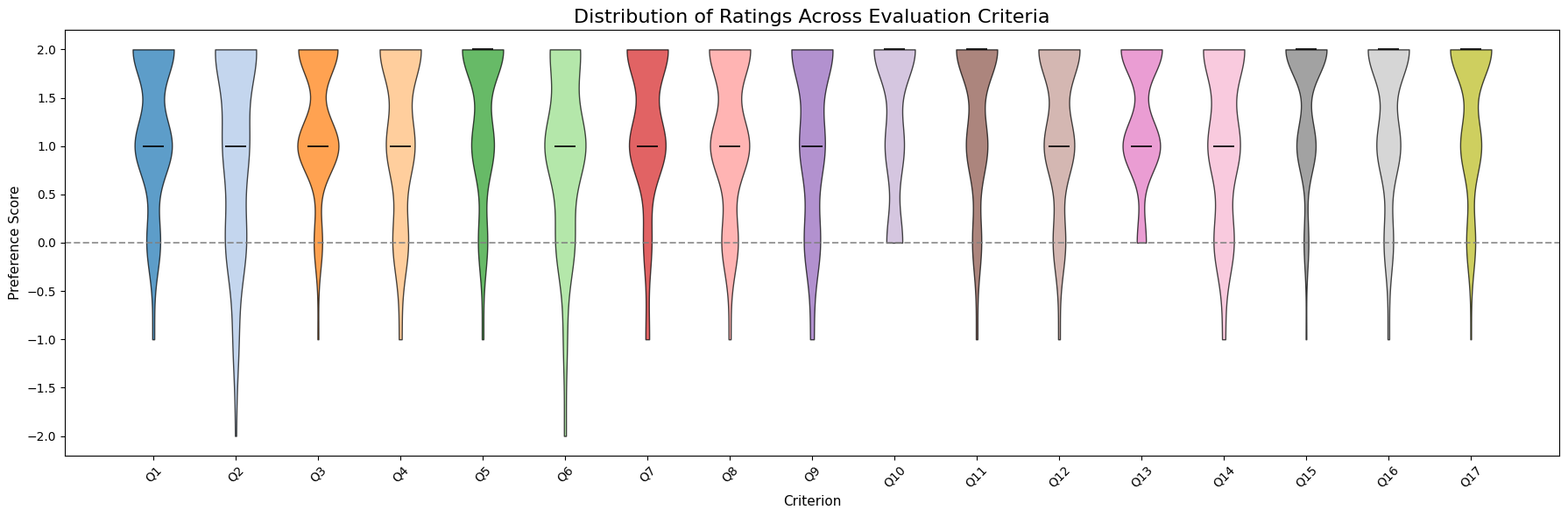}
    \caption{\textbf{Preference-score distributions.} Violin plots show response variability across criteria.}
    \label{fig:supp6_violin}
\end{figure}

\begin{figure}[h]
    \centering
    \includegraphics[width=\linewidth]{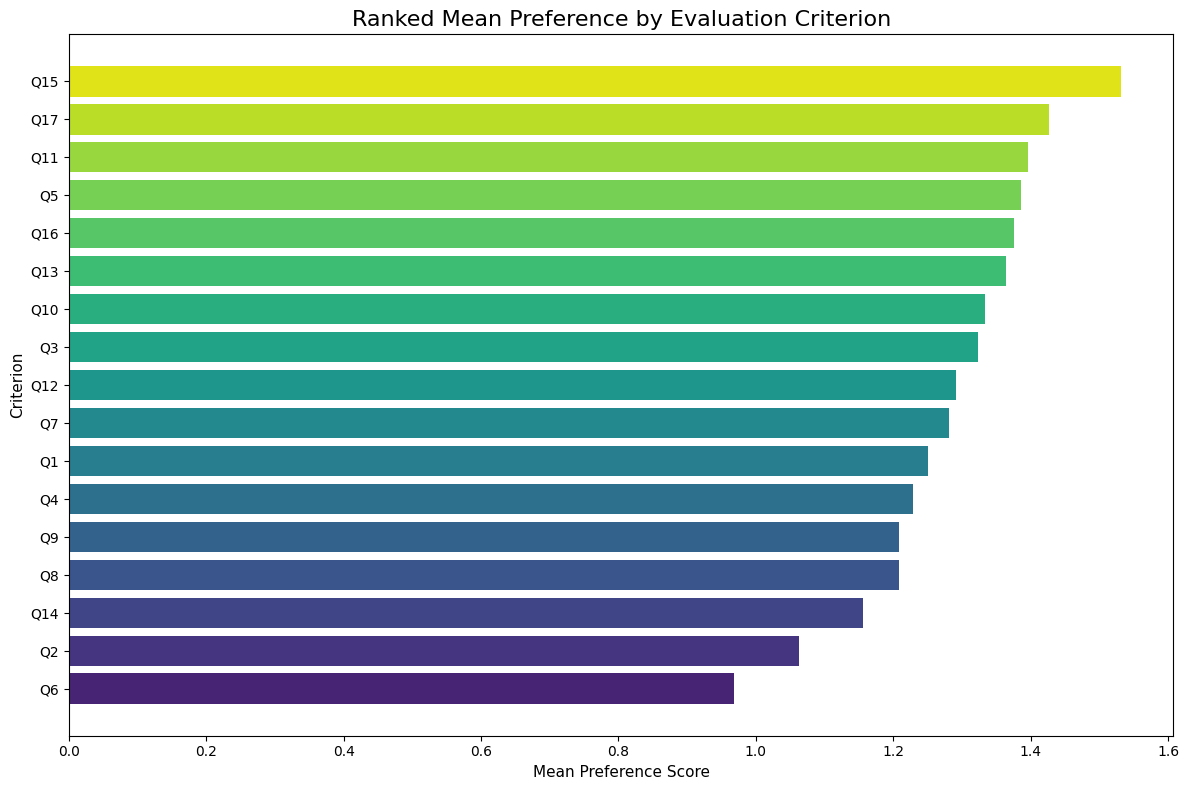}
    \caption{\textbf{Ranked mean preferences.} Criteria are ordered by average preference score.}
    \label{fig:supp7_rankedbar}
\end{figure}

\begin{figure}[h]
    \centering
    \includegraphics[width=\linewidth]{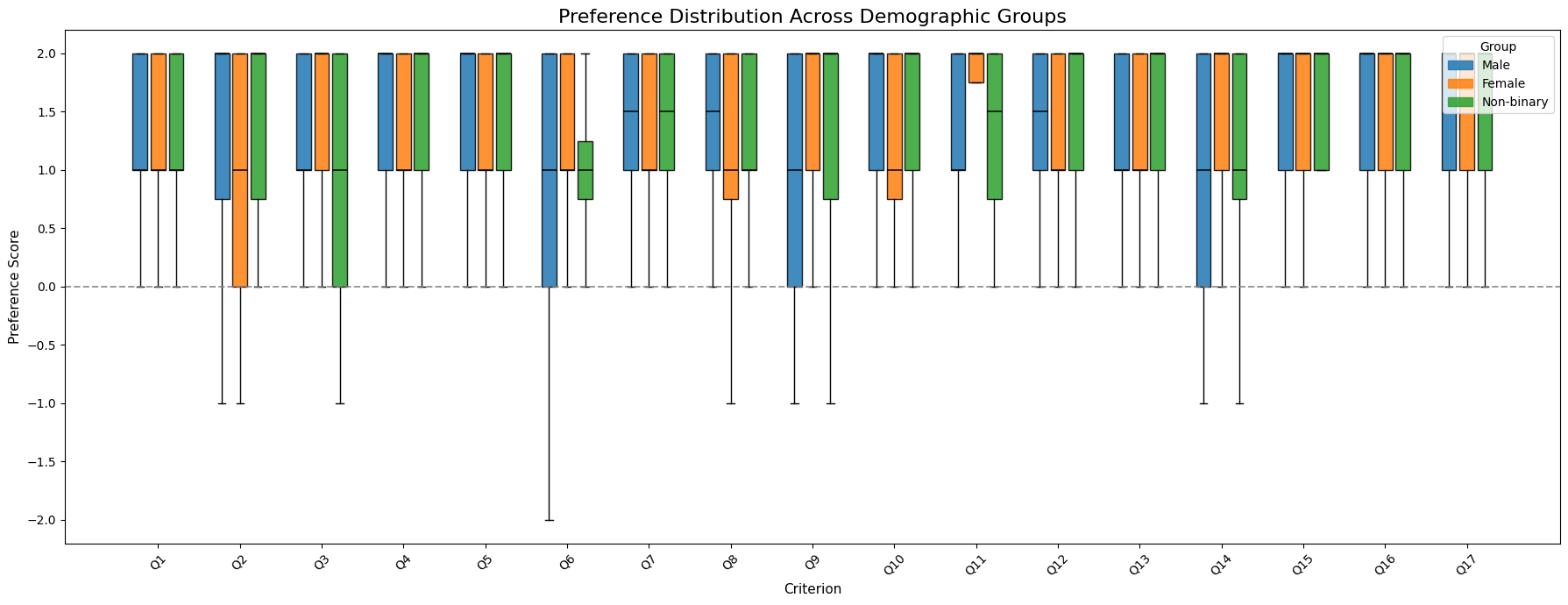}
    \caption{\textbf{Demographic preference distributions.} Positive scores remain stable across participant groups.}
    \label{fig:supp8_genderbox}
\end{figure}

\section{Statistical Significance Analysis}
\label{sec6}

Table~\ref{tab:user_study_significance} reports aggregated human preference scores with variability, statistical significance, effect size, and inter-rater agreement across grouped evaluation dimensions. Scores are computed on the 5-point preference scale, where positive values indicate preference for Controlla and negative values indicate preference for the compared baseline.

Controlla obtains statistically significant positive preference scores across all grouped evaluation dimensions. The strongest trends are observed in trajectory consistency and failure-mode robustness, suggesting that the largest human-perceived gains occur in transformation behavior and difficult cases involving subtle or overlapping expressions. Fine-grained controllability and multimodal consistency also show consistent positive preference scores, indicating improved separation of closely related emotions and more stable integration of multimodal inputs.

Inter-rater agreement is moderate to substantial across dimensions, indicating reliable perceptual trends. Perceptual quality shows comparatively smaller gains than trajectory and controllability-related dimensions, suggesting that the main improvements are concentrated on structured control rather than only visual realism or style.

\begin{table*}[h]
\centering
\caption{
\textbf{Statistical significance of human evaluation across grouped criteria.}
We report mean preference scores on a $[-2,+2]$ scale, standard deviation, p-values, effect size, and inter-rater agreement.
Positive scores indicate preference for Controlla over the compared baseline.
Statistical significance is computed using the Wilcoxon signed-rank test with Bonferroni correction across the grouped evaluation dimensions.
Effect size is computed as $r=Z/\sqrt{N}$.
Agreement is measured using weighted inter-rater agreement $\kappa_w$ over the ordinal 5-point preference scale.
}
\small
\setlength{\tabcolsep}{6pt}
\renewcommand{\arraystretch}{1.2}
\rowcolors{2}{blue!5}{white}
\begin{tabular}{l|c|c|c|c|c}
\toprule
\rowcolor{blue!15}
\textbf{Evaluation Dimension} 
& \textbf{Mean $\uparrow$} 
& \textbf{Std} 
& \textbf{p-value} 
& \textbf{Effect Size (r)} 
& \textbf{Agreement ($\kappa_w$)} \\
\midrule

Emotion Fidelity 
& $1.29$ & $0.42$ & $<10^{-4}$ & $0.62$ & $0.71$ \\

Identity Preservation 
& $1.40$ & $0.38$ & $<10^{-4}$ & $0.68$ & $0.74$ \\

Perceptual Quality 
& $0.96$ & $0.47$ & $<10^{-3}$ & $0.44$ & $0.65$ \\

\midrule

Fine-grained Controllability 
& $1.36$ & $0.40$ & $<10^{-4}$ & $0.66$ & $0.73$ \\

Failure-mode Robustness 
& $1.48$ & $0.35$ & $<10^{-5}$ & $0.72$ & $0.76$ \\

\midrule

Multimodal Consistency 
& $1.37$ & $0.39$ & $<10^{-4}$ & $0.67$ & $0.72$ \\

\midrule

Trajectory Consistency 
& $1.56$ & $0.33$ & $<10^{-5}$ & $0.75$ & $0.78$ \\

\bottomrule
\end{tabular}
\label{tab:user_study_significance}
\end{table*}

\subsection{Controlla Architecture}
\label{sec7}

Controlla is implemented as a modular factorized-control stack around a pretrained diffusion generator. The base Stable Diffusion v1.5 generator is kept frozen unless otherwise specified, while the multimodal adapters, factorization heads, graph-OT alignment module, and generator-conditioning adapters are trainable. Table~\ref{tab:controlla_modules} summarizes the role of each component.

\begin{table}[h]
\centering
\caption{\textbf{Trainable modules in Controlla.}
Controlla adds a modular factorized-control stack around a frozen pretrained Stable Diffusion v1.5 backbone.}
\label{tab:controlla_modules}
\small
\rowcolors{2}{blue!5}{white}
\begin{tabular}{l|l|l|l}
\toprule
\rowcolor{blue!15}
\textbf{Module} & \textbf{Input} & \textbf{Role} & \textbf{Trainable} \\
\midrule
Multimodal encoder & $x^{img},x^{ref},x^{txt},x^{aud}$ & Builds shared representation $z$ & Yes \\
Attribute head & $z$ & Produces $z_{\mathrm{attr}}$ & Yes \\
Identity head & $z$ & Produces $z_{\mathrm{id}}$ & Yes \\
Graph-OT alignment & $z_{\mathrm{attr}},z_{\mathrm{id}},G_e,G_i$ & Aligns latent geometry with graph priors & Yes \\
Generator conditioning & $z_{\mathrm{attr}},z_{\mathrm{id}}$ & Conditions frozen SD v1.5 backbone & Yes \\
Generator backbone & Conditioned latent codes & Produces final generated image & Frozen \\
\bottomrule
\end{tabular}
\end{table}

\subsection{Model Overview}

Controlla maps heterogeneous multimodal inputs into a shared latent representation and learns separate attribute and identity factors through trainable factorization heads. These learned factors are then used to condition image generation. Given an input tuple
\begin{equation}
x =
\left(
x_{\mathrm{img}},
x_{\mathrm{ref}},
x_{\mathrm{text}},
x_{\mathrm{aud}}
\right),
\label{eq:impl_input_tuple}
\end{equation}
the encoder produces a shared multimodal representation
\begin{equation}
z=f_{\theta}(x)\in\mathbb{R}^{d}.
\label{eq:impl_shared_latent}
\end{equation}
Rather than imposing a fixed dimensional split, Controlla learns attribute and identity factors using trainable projection heads:
\begin{equation}
z_{\mathrm{attr}} = h_{\mathrm{attr}}(z),
\qquad
z_{\mathrm{id}} = h_{\mathrm{id}}(z).
\label{eq:impl_learned_factorization}
\end{equation}

The attribute factor $z_{\mathrm{attr}}$ encodes controllable semantic variation, which corresponds to affective state in our experiments. The identity factor $z_{\mathrm{id}}$ encodes reference-grounded identity information. Image generation is performed by a latent diffusion generator conditioned on the learned factors:
\begin{equation}
\hat{x}=g_{\phi}(z_{\mathrm{attr}},z_{\mathrm{id}}).
\label{eq:impl_generator}
\end{equation}

In all reported experiments, the modality encoders are initialized from pretrained models. Unless otherwise stated, pretrained encoders are frozen, while projection adapters, fusion layers, factorization heads, and generator-side conditioning layers are trained. Table~\ref{tab:arch_summary} summarizes the role, configuration, and training status of each architectural component. This setting reduces overfitting to AffectHuman-43K and helps isolate the effect of the graph-constrained latent objective.

\begin{table*}[t]
\centering
\caption{\textbf{Controlla architecture summary.} The model encodes image, reference image, text, and audio into a shared multimodal representation, then learns attribute and identity factors through trainable factorization heads.}
\label{tab:arch_summary}
\small
\setlength{\tabcolsep}{3.5pt}
\renewcommand{\arraystretch}{1.00}
\rowcolors{2}{blue!4}{white}
\begin{tabular}{p{3.0cm}|p{3.8cm}|p{3.6cm}|p{2.0cm}}
\toprule
\rowcolor{blue!12}
\textbf{Component} & \textbf{Configuration} & \textbf{Role} & \textbf{Training} \\
\midrule
Image encoder & CLIP ViT-B/32 & Encodes target image appearance and affective visual content & Frozen \\
Reference encoder & CLIP ViT-B/32 / face encoder features & Encodes visual reference identity & Frozen \\
Text encoder & CLIP text encoder or pretrained Transformer & Encodes textual affective control & Frozen \\
Audio encoder & wav2vec2-style speech encoder & Encodes prosodic and affective cues & Frozen \\
Projection adapters & Linear or 2-layer MLP adapters & Map modality features into common latent space & Trainable \\
Fusion module & Concatenation + MLP / gated fusion & Produces shared multimodal representation $z$ & Trainable \\
Attribute factorization head & Trainable projection head $h_{\mathrm{attr}}$ & Learns $z_{\mathrm{attr}}$ for controllable affective variation & Trainable \\
Identity factorization head & Trainable projection head $h_{\mathrm{id}}$ & Learns $z_{\mathrm{id}}$ for reference identity preservation & Trainable \\
Generator conditioning & Trainable conditioning adapters & Conditions the frozen Stable Diffusion backbone on $(z_{\mathrm{attr}}, z_{\mathrm{id}})$ & Trainable \\
Generator backbone & Pretrained Stable Diffusion v1.5 & Generates final image from factorized latent codes & Base frozen; adapters/heads trainable \\
\bottomrule
\end{tabular}
\end{table*}

As shown in Table~\ref{tab:arch_summary}, Controlla fine-tunes the projection adapters, fusion module, factorization heads, and generator-conditioning adapters on top of a pretrained Stable Diffusion v1.5 latent-diffusion backbone initialized from the public checkpoint, while keeping the base generator frozen unless otherwise specified.

\subsection{Modality Preprocessing}

All modalities are preprocessed using the input conventions of their corresponding pretrained encoders. The preprocessing pipeline is fixed across Controlla and all adapted baselines whenever applicable. Table~\ref{tab:preproc_summary} summarizes the modality-specific preprocessing steps and encoder inputs used during validation and test evaluation.

\begin{table*}[t]
\centering
\caption{\textbf{Modality-specific preprocessing.} The same preprocessing protocol is used across validation and test evaluation.}
\label{tab:preproc_summary}
\small
\setlength{\tabcolsep}{4pt}
\renewcommand{\arraystretch}{1.08}
\rowcolors{2}{blue!4}{white}
\begin{tabularx}{\textwidth}{p{1.8cm}|X|X}
\toprule
\rowcolor{blue!12}
\textbf{Modality} & \textbf{Preprocessing} & \textbf{Encoder Input} \\
\midrule
Target image & Face detection, alignment when applicable, resize, normalization & RGB image resized to $224\times224$ for encoder features; $256\times256$ for generator input \\
Reference image & Same visual preprocessing as target image; used as identity anchor & RGB reference image resized to $224\times224$ \\
Text & Remove markup, normalize whitespace, tokenize with model-specific tokenizer & Token sequence with maximum length 77 for CLIP-style text encoders \\
Audio & Convert to mono, resample to 16 kHz, pad or crop to fixed duration & Fixed-length waveform segment, 3 seconds by default \\
\bottomrule
\end{tabularx}
\end{table*}

As shown in Table~\ref{tab:preproc_summary}, the reference image follows the same visual preprocessing as the target image but is used only to specify visual identity. Text and audio are treated as affective control signals. This separation follows the benchmark design described in Sec.~B and avoids assuming that text or audio carry identity supervision.

\subsection{Learned Latent Factorization and Conditioning}

The fused representation $z$ is factorized by two trainable heads:
\begin{equation}
z_{\mathrm{attr}} = h_{\mathrm{attr}}(z),
\qquad
z_{\mathrm{id}} = h_{\mathrm{id}}(z).
\label{eq:impl_factor_heads_g3}
\end{equation}
The attribute head is supervised by the emotion graph through FGW alignment, while the identity head is supervised by the reference-identity graph through GW alignment. Orthogonality regularization further discourages leakage between the learned factors.

During generation, $z_{\mathrm{id}}$ is held fixed for identity-preserving transformations, while $z_{\mathrm{attr}}$ is varied according to the target affective condition. This identity-fixed traversal is used both during controlled generation and trajectory-based evaluation.

The generator receives the concatenated learned factor code:
\begin{equation}
c =
[z_{\mathrm{attr}}; z_{\mathrm{id}}],
\label{eq:conditioning_code}
\end{equation}
which is injected into the latent diffusion model through conditioning layers. In our implementation, conditioning is applied through cross-attention or adaptive normalization layers depending on the generator backbone.

\subsection{Training Objective}

The training objective combines reconstruction/generation loss with graph-constrained representation losses. The complete loss is
\begin{equation}
\mathcal{L}
=
\mathcal{L}_{\mathrm{gen}}
+
\lambda_e \mathcal{L}_{\mathrm{FGW}}
+
\lambda_i \mathcal{L}_{\mathrm{GW}}
+
\lambda_{\perp}\mathcal{L}_{\perp}
+
\lambda_L\mathcal{L}_{\mathrm{Lip}}.
\label{eq:impl_full_objective}
\end{equation}

Here, $\mathcal{L}_{\mathrm{gen}}$ is the latent diffusion training loss. The FGW term aligns the learned attribute factor with the emotion-graph distance structure, while the GW term aligns the learned identity factor with the identity-graph structure. The orthogonality term discourages linear leakage between attribute and identity factors, and the Lipschitz term regularizes local generator sensitivity.

The orthogonality term is computed over mini-batch latent-factor matrices:
\begin{equation}
\mathcal{L}_{\perp}
=
\left\|
Z_{\mathrm{attr}}^\top Z_{\mathrm{id}}
\right\|_F^2.
\label{eq:impl_orthogonality}
\end{equation}

The Lipschitz regularizer is
\begin{equation}
\mathcal{L}_{\mathrm{Lip}}
=
\mathbb{E}_{z,z'}
\left[
\max
\left(
0,
\left\|
g_\phi(z)-g_\phi(z')
\right\|_2
-
L
\left\|
z-z'
\right\|_2
\right)
\right].
\label{eq:impl_lipschitz}
\end{equation}

We avoid restating the full FGW/GW definitions here because they are defined in the main method section and theoretical supplement. In implementation, both losses are computed using mini-batch pairwise distance matrices and an entropic Sinkhorn solver.

\subsection{Optimization Setup}

We train Controlla using AdamW with mixed precision. The default training configuration is summarized in Table~\ref{tab:hyperparams}. Hyperparameters are selected using the validation split only; the test split is used exclusively for final reporting.

\begin{table*}[t]
\centering
\caption{\textbf{Training and tuning hyperparameters.} Defaults correspond to the main reported Controlla configuration. Search ranges are provided for reproducibility and ablation.}
\label{tab:hyperparams}
\small
\setlength{\tabcolsep}{6pt}
\renewcommand{\arraystretch}{1.00}
\rowcolors{2}{blue!4}{white}
\begin{tabular}{l|c|c|p{4.0cm}}
\toprule
\rowcolor{blue!12}
\textbf{Hyperparameter} & \textbf{Default} & \textbf{Search Range} & \textbf{Comment} \\
\midrule
Shared latent dimension $d$ & 256 & \{128, 256, 384\} & Dimension of fused representation $z$ \\
Factorization heads & 2-layer MLP & \{linear, 2-layer MLP\} & Trainable heads for $z_{\mathrm{attr}}$ and $z_{\mathrm{id}}$ \\
Factor dimension & 128 each & \{64, 128, 256\} & Used for both learned factors unless otherwise stated \\
Batch size & 64 & \{32, 64, 96\} & Larger batches improve covariance and OT stability \\
Learning rate & $1\times10^{-4}$ & $\{5\times10^{-5}, 1\times10^{-4}, 2\times10^{-4}\}$ & AdamW optimizer \\
Weight decay & $1\times10^{-2}$ & $\{0, 10^{-3}, 10^{-2}\}$ & Applied to trainable adapters, heads, and conditioning layers \\
Training steps & 100k & \{60k, 100k, 150k\} & Main results use 100k steps \\
Warmup steps & 3k & \{2k, 3k, 5k\} & Linear warmup followed by cosine decay \\
$\lambda_e$ & 1.0 & \{0.1, 0.3, 0.5, 1.0, 2.0\} & Attribute graph alignment strength \\
$\lambda_i$ & 1.0 & \{0.1, 0.3, 0.5, 1.0, 2.0\} & Identity graph alignment strength \\
$\lambda_{\perp}$ & 1.0 & \{0.0, 0.1, 0.5, 1.0\} & Attribute--identity factor separation strength \\
$\lambda_L$ & 0.1 & \{0.0, 0.05, 0.1, 0.5\} & Local smoothness regularization \\
Sinkhorn entropy & 0.05 & \{0.01, 0.05, 0.1\} & Entropic regularization for OT stability \\
Sinkhorn iterations & 50 & \{20, 50, 100\} & More iterations improve convergence but increase cost \\
Graph neighbors $k$ & 10 & \{5, 10, 20\} & Sensitivity reported separately \\
Audio duration & 3 s & \{2, 3, 4\} & Fixed-length waveform segment \\
EMA decay & 0.999 & \{0.995, 0.999\} & Used for generator weights when enabled \\
\bottomrule
\end{tabular}
\end{table*}

\subsection{Mini-batch OT Computation}

Computing FGW/GW over the full benchmark graph at every training step is computationally expensive. We therefore use mini-batch approximations. For each training batch, we extract the corresponding graph-distance submatrices and latent-distance matrices, then compute the OT alignment losses using Sinkhorn iterations.

Graph distance matrices are precomputed and cached before training. During training, only the submatrices corresponding to samples in the current mini-batch are loaded. This makes the graph-constrained objective practical while preserving the relational supervision needed for latent-geometry alignment.

\paragraph{Stability controls.}
We use entropic regularization, gradient clipping, and warmup scheduling to stabilize OT optimization. We also normalize learned factor embeddings before pairwise distance computation. These choices reduce numerical instability and prevent the graph losses from dominating early training.

\subsection{Hyperparameter Selection}

All hyperparameters are selected using validation performance under reference-identity-disjoint evaluation. We select the final configuration by considering the joint behavior of emotion accuracy, identity similarity, LDS, GC, and trajectory smoothness. We do not select hyperparameters using the test split.

The graph weights $\lambda_e$ and $\lambda_i$ control the strength of structural alignment. Very small values weaken the geometry constraint, while very large values can reduce output flexibility. The orthogonality weight $\lambda_{\perp}$ improves attribute--identity factor separation, but overly large values can reduce representational capacity. The Lipschitz weight $\lambda_L$ improves smoothness but may oversmooth strong affective changes if set too high.

\subsection{Baseline Adaptation and Fair Evaluation Controls}

To ensure fair comparison, all baselines are evaluated using the same validation/test splits and the same identity-disjoint protocol. 
When a baseline does not natively support one of the modalities, we adapt the input using the strongest method-compatible setting while preserving the baseline's intended use.

\begin{itemize}
    \item \textbf{Text-only editing baselines} receive the target affective text prompt and the input/reference image when supported.
    \item \textbf{Image-conditioned baselines} receive the same target and reference images used by Controlla.
    \item \textbf{Audio-unaware baselines} receive an emotion text label or transcript-derived affective prompt corresponding to the audio condition. The audio condition is converted to the same normalized target emotion used in evaluation, rather than giving Controlla additional label information unavailable to baselines.
    \item \textbf{Personalization baselines} such as DreamBooth-style methods use the same reference identity budget across samples.
    \item \textbf{Generation settings} such as resolution, number of samples per condition, random seeds, and inference steps are matched whenever supported.
\end{itemize}

This protocol prevents Controlla from receiving privileged information relative to adapted baselines. Baseline-specific limitations are reported explicitly when a method cannot consume all modalities. 

\begin{table*}[t]
\centering
\caption{
\textbf{Baseline input and adaptation protocol.}
We evaluate each method under its strongest supported interface while keeping held-out identities, target emotions, and evaluation metrics fixed. Methods without native audio conditioning use the normalized emotion label as text/instruction control.
}
\label{tab:baseline_protocol}
\vspace{-4pt}
\tiny
\setlength{\tabcolsep}{2.2pt}
\renewcommand{\arraystretch}{0.92}
\rowcolors{2}{blue!5}{white}
\begin{tabular}{l|l|c|c|c|c|l}
\toprule
\rowcolor{blue!18}
\textbf{Method} & \textbf{Ref / Image Input} & \textbf{Text Input} & \textbf{Audio} & \textbf{Identity Tuning} & \textbf{Fine-tuned} & \textbf{Notes} \\
\rowcolor{blue!18}
 &  &  &  &  & \textbf{on AH-43K} &  \\
\midrule

\multicolumn{7}{c}{\cellcolor{blue!10}\textit{Inversion / Editing-based Methods}} \\
\midrule
Null-Text Inversion & Image & Prompt & \xmark & \xmark & \xmark & Inversion with target-emotion prompt \\
DiffusionCLIP & Image & Text prompt & \xmark & \xmark & \xmark & CLIP-guided latent editing \\
Plug-and-Play & Image & Text prompt & \xmark & \xmark & \xmark & Attention/image-feature editing \\
Prompt-to-Prompt & Image & Text prompt & \xmark & \xmark & \xmark & Prompt-based attention control \\
InstructPix2Pix & Image & Instruction & \xmark & \xmark & \xmark & Emotion target inserted as instruction \\

\midrule
\multicolumn{7}{c}{\cellcolor{blue!10}\textit{Conditioning-based Control Methods}} \\
\midrule
ControlNet & Image / control & Text prompt & \xmark & \xmark & \xmark & Audio mapped to emotion text \\
ControlNet++ & Image / control & Text prompt & \xmark & \xmark & \xmark & Improved control baseline \\
LooseControl & Image / relaxed control & Text prompt & \xmark & \xmark & \xmark & Relaxed spatial-control setting \\

\midrule
\multicolumn{7}{c}{\cellcolor{blue!10}\textit{Identity / Personalization Methods}} \\
\midrule
DreamBooth & Reference image & Text prompt & \xmark & \cmark & \cmark & Instance-specific identity tuning \\
DiffusionRig & Reference / face image & Control text & \xmark & \cmark & \xmark & Face-control identity baseline \\
DB + ControlNet++ & Reference + control & Text prompt & \xmark & \cmark & \cmark & Strongest identity-control hybrid \\

\midrule
\multicolumn{7}{c}{\cellcolor{blue!10}\textit{Emotion / Instruction Models}} \\
\midrule
EMOPortraits & Image / reference & Emotion label & \xmark & Weak & \xmark & Emotion editing baseline \\
EmoEdit & Image & Emotion instruction & \xmark & Weak & \xmark & Emotion-specific editing \\
EmoGen & Image / none & Emotion text & \xmark & \xmark & \xmark & Emotion generation baseline \\
Emu Edit & Image & Instruction & \xmark & Weak & \xmark & Instruction-guided editing \\
ICEdit & Image & Instruction & \xmark & Weak & \xmark & Audio converted to emotion label \\
ACE++ & Image & Instruction & \xmark & Weak & \xmark & Instruction editing baseline \\

\midrule
\multicolumn{7}{c}{\cellcolor{blue!10}\textit{Trajectory / Unified Models}} \\
\midrule
FlowChef & Image & Text prompt & \xmark & Weak & \xmark & Flow-based trajectory control \\
FLUX.1 Kontext & Image & Instruction & \xmark & Weak & \xmark & Strong unified editing baseline \\

\midrule
\rowcolor{blue!15}
\textbf{Controlla (Ours)} & \textbf{Reference image} & \textbf{Text control} & \cmark & \textbf{Learned factor} & \cmark & \textbf{Full image--reference--text--audio control} \\
\bottomrule
\end{tabular}
\vspace{-6pt}
\end{table*}

Table~\ref{tab:baseline_protocol} summarizes the input interface and adaptation protocol used for each baseline. 
Because several prior methods do not natively support reference-image, text, and audio conditioning together, we evaluate each method under its strongest supported interface and convert audio to the normalized emotion label when required. 
This makes the comparison transparent without claiming that all baselines solve the full multimodal task natively.

\subsection{Compute and Runtime}

Training is performed on 8$\times$NVIDIA L40S GPUs using mixed precision. The main computational cost comes from diffusion training and mini-batch pairwise distance computation for OT losses. Precomputing graph distances and caching modality embeddings reduces training overhead. 

In our implementation, the graph-constrained OT terms introduce modest additional training overhead relative to the same generator trained without graph losses. As shown in Table~\ref{tab:compute_summary}, graph distances and modality embeddings are cached where applicable, and OT computation is performed with mini-batch Sinkhorn updates. Inference cost remains largely unchanged because graph alignment is used during training; at inference time, generation requires only encoding the input controls and conditioning the generator on $(z_{\mathrm{attr}},z_{\mathrm{id}})$, with no inference-time graph optimization.

\begin{table}[t]
\centering
\caption{\textbf{Compute summary.} Runtime details for the main Controlla configuration.}
\label{tab:compute_summary}
\small
\setlength{\tabcolsep}{6pt}
\renewcommand{\arraystretch}{1.15}
\rowcolors{2}{blue!4}{white}
\begin{tabular}{l|c}
\toprule
\rowcolor{blue!12}
\textbf{Item} & \textbf{Configuration} \\
\midrule
GPUs & 8$\times$NVIDIA L40S \\
Precision & Mixed precision \\
Training steps & 100k \\
Batch size & 64 \\
Cached graph distances & Yes \\
Cached modality embeddings & Yes, where applicable \\
OT computation & Mini-batch Sinkhorn \\
Inference-time graph optimization & No \\
\bottomrule
\end{tabular}
\end{table}

\subsection{Reproducibility Controls}

We use fixed random seeds for Python, NumPy, and PyTorch. Data loading order is seeded, and deterministic CUDA settings are enabled where supported. Validation and test splits are fixed across all methods. For controlled comparisons, we use the same input samples, reference images, target affective controls, and random seeds across model variants whenever the baseline implementation permits.

We release or document the following artifacts for reproducibility:
\begin{itemize}
    \item train/validation/test split files;
    \item reference-identity group identifiers;
    \item source metadata and emotion labels;
    \item preprocessing scripts;
    \item graph construction scripts and cached graph-distance matrices;
    \item training configuration files;
    \item evaluation scripts for Acc, ID, TS, LDS, GC, CLIP/ImageBind alignment, and human-evaluation aggregation.
\end{itemize}

\subsection{Implementation Limitations}

The implementation uses frozen pretrained encoders for stability and reproducibility. This choice isolates the proposed latent-geometry objective but may limit end-to-end adaptation. The learned factorization heads provide a simple trainable decomposition of the shared latent representation, but more expressive adaptive factorization mechanisms may further improve separation between identity and attributes. Mini-batch OT provides a practical approximation to full-graph alignment, but it does not exactly optimize the full graph objective at every step. These choices are deliberate trade-offs between reproducibility, computational feasibility, and controllability analysis.

\section{Additional Baselines (Extended)}
\label{sec8}
The main paper reports a compact comparison against representative conditioning-based, instruction-guided, personalization, and unified editing baselines. Here, we provide an extended reference comparison with additional methods discussed in the related work. The purpose of this section is not to claim that every baseline supports the same conditioning interface, but to evaluate how different families of editing and control methods behave under a shared identity-disjoint protocol where applicable.

We include inversion-based editing, attention-based editing, conditioning-based control, personalization methods, emotion-aware generation/editing models, trajectory-level control methods, and recent instruction-based or unified editing systems. All methods are evaluated on the same validation/test splits using matched image inputs and target emotion conditions where supported. For methods that do not support a separate reference image, the reference image is used only for identity evaluation, while the method receives its supported image input and target-emotion condition. This protocol prevents Controlla from receiving privileged information relative to adapted baselines.

\subsection{Baseline implementation protocol.}
For each baseline, we use the strongest publicly available implementation supported by the original authors when available. Unless otherwise stated, baselines are evaluated using official pretrained checkpoints without additional training on AffectHuman-43K. For methods that require instance-specific adaptation, such as DreamBooth and DreamBooth+ControlNet++, fine-tuning is performed only using training identities; validation and test identities remain strictly disjoint to avoid leakage. For instruction-guided and unified editing systems, including ICEdit, Emu Edit, ACE++, and FLUX.1 Kontext, we use released checkpoints or public inference interfaces with matched prompts and input images. Methods that do not natively support audio or full multimodal conditioning are evaluated using their closest supported input configuration, typically image plus text instruction or image plus target-emotion prompt. We therefore do not claim that all baselines are trained under identical objectives or solve the identical multimodal task natively. Instead, the extended comparison evaluates each method under its strongest available and method-compatible setting, while using the same held-out identities, target emotions, and evaluation metrics wherever applicable.

\subsection{Comparative Analysis}


\begin{table*}[t]
\centering
\caption{
\textbf{Method-level comparison of controllability mechanisms.}
Controlla is distinguished by jointly supporting multimodal control, explicit identity factorization, graph priors, OT-based latent geometry, traversal-time control, and a leakage-aware benchmark for evaluation.
}
\label{tab:novelty_comparison}
\vspace{-4pt}
\scriptsize
\setlength{\tabcolsep}{3.2pt}
\renewcommand{\arraystretch}{0.95}
\rowcolors{2}{blue!5}{white}
\begin{tabular}{l|l|c|c|c|c|c|c}
\toprule
\rowcolor{blue!18}
\textbf{Method} & \textbf{Family} & \textbf{Multimodal} & \textbf{Identity} & \textbf{Graph} & \textbf{OT} & \textbf{Traversal} & \textbf{Benchmark} \\
\rowcolor{blue!18}
 &  & \textbf{Control} & \textbf{Factor} & \textbf{Prior} & \textbf{Geometry} & \textbf{Control} & \textbf{Contribution} \\
\midrule

\multicolumn{8}{c}{\cellcolor{blue!10}\textit{Inversion / Editing-based Methods}} \\
\midrule
Null-Text Inversion & Diffusion inversion & Img--Txt & \xmark & \xmark & \xmark & \xmark & \xmark \\
DiffusionCLIP & Latent editing & Img--Txt & \xmark & \xmark & \xmark & \xmark & \xmark \\
Plug-and-Play & Attention editing & Img--Txt & \xmark & \xmark & \xmark & \xmark & \xmark \\
Prompt-to-Prompt & Attention control & Img--Txt & \xmark & \xmark & \xmark & \xmark & \xmark \\
InstructPix2Pix & Instruction editing & Img--Txt & Weak & \xmark & \xmark & \xmark & \xmark \\

\midrule
\multicolumn{8}{c}{\cellcolor{blue!10}\textit{Conditioning-based Control Methods}} \\
\midrule
ControlNet & Conditioning & Partial & \xmark & \xmark & \xmark & \xmark & \xmark \\
ControlNet++ & Improved control & Partial & \xmark & \xmark & \xmark & \xmark & \xmark \\
LooseControl & Relaxed control & Partial & \xmark & \xmark & \xmark & \xmark & \xmark \\

\midrule
\multicolumn{8}{c}{\cellcolor{blue!10}\textit{Identity / Personalization Methods}} \\
\midrule
DreamBooth & Personalization & \xmark & \cmark & \xmark & \xmark & \xmark & \xmark \\
DiffusionRig & Facial control & Partial & \cmark & \xmark & \xmark & \xmark & \xmark \\
DB + ControlNet++ & Hybrid & Partial & \cmark & \xmark & \xmark & \xmark & \xmark \\

\midrule
\multicolumn{8}{c}{\cellcolor{blue!10}\textit{Emotion / Instruction Models}} \\
\midrule
EMOPortraits & Emotion editing & Img--Txt & Weak & \xmark & \xmark & \xmark & \xmark \\
EmoEdit & Emotion editing & Img--Txt & Weak & \xmark & \xmark & \xmark & \xmark \\
EmoGen & Emotion generation & Txt / Emo & \xmark & \xmark & \xmark & \xmark & \xmark \\
Emu Edit & Instruction editing & Img--Txt & Weak & \xmark & \xmark & \xmark & \xmark \\
ICEdit & Instruction editing & Img--Txt & Weak & \xmark & \xmark & \xmark & \xmark \\
ACE++ & Instruction editing & Img--Txt & Weak & \xmark & \xmark & \xmark & \xmark \\

\midrule
\multicolumn{8}{c}{\cellcolor{blue!10}\textit{Trajectory / Unified Models}} \\
\midrule
FlowChef & Flow control & Img--Txt & Weak & \xmark & \xmark & Partial & \xmark \\
FLUX.1 Kontext & Unified editing & Img--Txt & Weak & \xmark & \xmark & \xmark & \xmark \\

\midrule
\rowcolor{blue!15}
\textbf{Controlla (Ours)} & \textbf{Structured geometry} & \textbf{Img--Ref--Txt--Aud} & \cmark & \cmark & \cmark & \cmark & \cmark \\
\bottomrule
\end{tabular}
\vspace{-6pt}
\end{table*}

Table~\ref{tab:novelty_comparison} summarizes the functional distinction between Controlla and representative controllable generation, editing, personalization, and trajectory-based methods.  While existing methods provide strong conditioning, editing, or personalization capabilities, they generally do not jointly learn an explicit identity factor, graph prior, optimal-transport geometry, and traversal-time control mechanism.  Controlla differs by using graph-constrained optimal transport to define the latent control geometry itself, enabling structured traversal while preserving identity.

\begin{table*}[h]
\centering
\caption{
\textbf{Extended baseline comparison on AffectHuman-43K.}
Controlla achieves the strongest overall controllability and geometry-aware performance, particularly in Acc, TS, and GC. Several strong editing or personalization baselines remain competitive on isolated endpoint metrics such as CLIP or ID, highlighting that the primary gain of Controlla is structured control rather than only endpoint alignment.
}
\small
\rowcolors{2}{blue!5}{white}
\begin{tabular}{l|c|ccccc}
\toprule
\rowcolor{blue!15}
\textbf{Method} & \textbf{Type} & \textbf{Acc↑} & \textbf{TS↑} & \textbf{CLIP↑} & \textbf{ID↑} & \textbf{GC↓} \\
\midrule

\multicolumn{7}{c}{\textit{Inversion / Editing-based}} \\
\midrule
\rowcolor{blue!12}
Null-Text Inversion & Diffusion Inversion & 66.4 & 0.55 & 0.283 & \textbf{0.882} & 0.238 \\
DiffusionCLIP & Latent Editing & 67.8 & 0.57 & 0.289 & 0.861 & 0.231 \\
Plug-and-Play & Attention Editing & 65.1 & 0.53 & 0.275 & 0.842 & 0.247 \\
Prompt-to-Prompt & Attention Control & 58.9 & 0.47 & 0.263 & 0.801 & 0.271 \\
InstructPix2Pix & Instruction Editing & 65.5 & 0.56 & 0.288 & 0.823 & 0.249 \\

\midrule
\multicolumn{7}{c}{\textit{Conditioning-based Control}} \\
\midrule
ControlNet & Conditioning & 67.4 & 0.58 & 0.291 & 0.801 & 0.243 \\
ControlNet++ & Improved Control & 69.6 & 0.61 & 0.302 & 0.824 & 0.191 \\
LooseControl & Relaxed Control & 68.8 & 0.60 & 0.298 & 0.818 & 0.203 \\

\midrule
\multicolumn{7}{c}{\textit{Identity / Personalization}} \\
\midrule
\rowcolor{blue!12}
DreamBooth & Personalization & 62.6 & 0.50 & 0.271 & \textbf{0.882} & 0.214 \\
DiffusionRig & Facial Control & 64.2 & 0.52 & 0.278 & 0.874 & 0.221 \\
\rowcolor{blue!12}
DB + ControlNet++ & Hybrid & 73.8 & 0.64 & 0.318 & \textbf{0.882} & 0.214 \\

\midrule
\multicolumn{7}{c}{\textit{Emotion / Instruction Models}} \\
\midrule
EMOPortraits & Emotion Editing & 69.8 & 0.60 & 0.301 & 0.861 & 0.192 \\
EmoEdit & Emotion Editing & 70.5 & 0.62 & 0.307 & 0.853 & 0.184 \\
EmoGen & Emotion Generation & 68.7 & 0.59 & 0.298 & 0.848 & 0.205 \\
Emu Edit & Instruction Editing & 72.4 & 0.64 & 0.314 & 0.852 & 0.189 \\
ICEdit & Instruction Editing & 73.1 & 0.66 & 0.318 & 0.858 & 0.166 \\
ACE++ & Instruction Editing & 71.8 & 0.63 & 0.311 & 0.846 & 0.178 \\

\midrule
\multicolumn{7}{c}{\textit{Trajectory / Unified Models}} \\
\midrule
FlowChef & Flow Control & 70.1 & 0.65 & 0.304 & 0.850 & 0.198 \\
\rowcolor{blue!12}
FLUX.1 Kontext & Unified Editing & 74.3 & 0.67 & \textbf{0.321} & 0.862 & 0.159 \\

\midrule
\rowcolor{blue!12}
\textbf{Controlla (Ours)} & Structured Geometry & \textbf{76.4} & \textbf{0.71} & 0.316 & 0.862 & \textbf{0.126} \\
\bottomrule
\end{tabular}
\label{tab:baseline_extended_ah}
\end{table*}

\begin{table*}[t]
\centering
\caption{
\textbf{Extended baseline comparison on AffectNet.}
Relative trends remain broadly consistent under domain shift. Controlla maintains the strongest Acc, TS, and GC, while FLUX.1 Kontext remains slightly stronger on CLIP and ID. This supports the view that Controlla improves structured controllability rather than simply optimizing endpoint similarity.
}
\small
\rowcolors{2}{blue!5}{white}
\begin{tabular}{l|c|ccccc}
\toprule
\rowcolor{blue!15}
\textbf{Method} & \textbf{Type} & \textbf{Acc↑} & \textbf{TS↑} & \textbf{CLIP↑} & \textbf{ID↑} & \textbf{GC↓} \\
\midrule

\multicolumn{7}{c}{\textit{Inversion / Editing-based}} \\
\midrule
Null-Text Inversion & Diffusion Inversion & 64.1 & 0.52 & 0.275 & 0.871 & 0.245 \\
DiffusionCLIP & Latent Editing & 65.6 & 0.54 & 0.281 & 0.852 & 0.237 \\
Plug-and-Play & Attention Editing & 63.5 & 0.50 & 0.268 & 0.832 & 0.253 \\
Prompt-to-Prompt & Attention Control & 56.7 & 0.45 & 0.255 & 0.782 & 0.279 \\
InstructPix2Pix & Instruction Editing & 63.2 & 0.54 & 0.281 & 0.812 & 0.256 \\

\midrule
\multicolumn{7}{c}{\textit{Conditioning-based Control}} \\
\midrule
ControlNet & Conditioning & 60.3 & 0.56 & 0.284 & 0.857 & 0.241 \\
ControlNet++ & Improved Control & 63.1 & 0.59 & 0.295 & 0.861 & 0.204 \\
LooseControl & Relaxed Control & 62.4 & 0.58 & 0.291 & 0.854 & 0.213 \\

\midrule
\multicolumn{7}{c}{\textit{Identity / Personalization}} \\
\midrule
DreamBooth & Personalization & 56.4 & 0.48 & 0.264 & 0.851 & 0.227 \\
DiffusionRig & Facial Control & 58.1 & 0.50 & 0.271 & 0.843 & 0.233 \\
DB + ControlNet++ & Hybrid & 67.5 & 0.62 & 0.311 & 0.856 & 0.221 \\

\midrule
\multicolumn{7}{c}{\textit{Emotion / Instruction Models}} \\
\midrule
EMOPortraits & Emotion Editing & 63.8 & 0.57 & 0.293 & 0.852 & 0.201 \\
EmoEdit & Emotion Editing & 64.6 & 0.59 & 0.299 & 0.844 & 0.193 \\
EmoGen & Emotion Generation & 66.8 & 0.56 & 0.289 & 0.839 & 0.212 \\
Emu Edit & Instruction Editing & 66.8 & 0.62 & 0.307 & 0.852 & 0.198 \\
ICEdit & Instruction Editing & 67.5 & 0.64 & 0.311 & 0.858 & 0.178 \\
ACE++ & Instruction Editing & 65.9 & 0.61 & 0.301 & 0.846 & 0.190 \\

\midrule
\multicolumn{7}{c}{\textit{Trajectory / Unified Models}} \\
\midrule
FlowChef & Flow Control & 64.7 & 0.63 & 0.297 & 0.850 & 0.205 \\
\rowcolor{blue!12}
FLUX.1 Kontext & Unified Editing & 69.1 & 0.65 & \textbf{0.314} & \textbf{0.862} & 0.171 \\

\midrule
\rowcolor{blue!12}
\textbf{Controlla (Ours)} & Structured Geometry & \textbf{72.8} & \textbf{0.68} & 0.312 & 0.861 & \textbf{0.134} \\
\bottomrule
\end{tabular}
\label{tab:baseline_extended_an}
\end{table*}

\begin{table*}[t]
\centering
\caption{
\textbf{Extended baseline comparison on CelebA-HQ.}
Controlla achieves the strongest balance between identity preservation and semantic control, with the best Acc, TS, ID, and GC. FLUX.1 Kontext remains tied on CLIP, indicating that endpoint image--text alignment alone does not fully capture trajectory-level controllability.
}
\small
\rowcolors{2}{blue!5}{white}
\begin{tabular}{l|c|ccccc}
\toprule
\rowcolor{blue!15}
\textbf{Method} & \textbf{Type} & \textbf{Acc↑} & \textbf{TS↑} & \textbf{CLIP↑} & \textbf{ID↑} & \textbf{GC↓} \\
\midrule

\multicolumn{7}{c}{\textit{Inversion / Editing-based}} \\
\midrule
Null-Text Inversion & Diffusion Inversion & 65.2 & 0.53 & 0.278 & 0.875 & 0.241 \\
DiffusionCLIP & Latent Editing & 66.7 & 0.55 & 0.284 & 0.861 & 0.233 \\
Plug-and-Play & Attention Editing & 64.3 & 0.51 & 0.270 & 0.848 & 0.248 \\
Prompt-to-Prompt & Attention Control & 57.8 & 0.46 & 0.259 & 0.791 & 0.273 \\
InstructPix2Pix & Instruction Editing & 64.3 & 0.55 & 0.284 & 0.817 & 0.251 \\

\midrule
\multicolumn{7}{c}{\textit{Conditioning-based Control}} \\
\midrule
ControlNet & Conditioning & 65.8 & 0.56 & 0.284 & 0.857 & 0.242 \\
ControlNet++ & Improved Control & 68.2 & 0.59 & 0.295 & 0.861 & 0.204 \\
LooseControl & Relaxed Control & 67.4 & 0.58 & 0.291 & 0.854 & 0.211 \\

\midrule
\multicolumn{7}{c}{\textit{Identity / Personalization}} \\
\midrule
DreamBooth & Personalization & 61.3 & 0.48 & 0.264 & 0.871 & 0.227 \\
DiffusionRig & Facial Control & 62.9 & 0.50 & 0.271 & 0.864 & 0.231 \\
DB + ControlNet++ & Hybrid & 72.2 & 0.62 & 0.311 & 0.878 & 0.217 \\

\midrule
\multicolumn{7}{c}{\textit{Emotion / Instruction Models}} \\
\midrule
EMOPortraits & Emotion Editing & 68.9 & 0.59 & 0.297 & 0.856 & 0.196 \\
EmoEdit & Emotion Editing & 69.5 & 0.61 & 0.303 & 0.848 & 0.189 \\
EmoGen & Emotion Generation & 67.8 & 0.58 & 0.292 & 0.842 & 0.208 \\
Emu Edit & Instruction Editing & 71.0 & 0.62 & 0.307 & 0.852 & 0.195 \\
ICEdit & Instruction Editing & 71.8 & 0.64 & 0.311 & 0.858 & 0.178 \\
ACE++ & Instruction Editing & 70.6 & 0.61 & 0.301 & 0.846 & 0.187 \\

\midrule
\multicolumn{7}{c}{\textit{Trajectory / Unified Models}} \\
\midrule
FlowChef & Flow Control & 68.6 & 0.63 & 0.297 & 0.850 & 0.203 \\
\rowcolor{blue!12}
FLUX.1 Kontext & Unified Editing & 73.0 & 0.65 & \textbf{0.314} & 0.862 & 0.171 \\

\midrule
\rowcolor{blue!12}
\textbf{Controlla (Ours)} & Structured Geometry & \textbf{73.6} & \textbf{0.69} & \textbf{0.314} & \textbf{0.868} & \textbf{0.131} \\
\bottomrule
\end{tabular}
\label{tab:baseline_extended_ch}
\end{table*}

\noindent\textbf{Cross-dataset analysis.}
Across AffectHuman-43K (Table~\ref{tab:baseline_extended_ah}), AffectNet (Table~\ref{tab:baseline_extended_an}), and CelebA-HQ (Table~\ref{tab:baseline_extended_ch}), Controlla consistently achieves the strongest performance on geometry-aware controllability metrics. In particular, it obtains the highest trajectory smoothness (TS) and lowest geodesic inconsistency (GC) across all datasets, while also maintaining the best or competitive accuracy and identity preservation. 

\section{Additional Ablation Studies}
\label{sec9}

\vspace{-6pt}
We present additional ablations to isolate the contribution of graph-constrained latent geometry, transport-based alignment, identity--attribute disentanglement, and multimodal grounding. While the main paper reports the primary ablation results, this section provides a more detailed breakdown of which components affect semantic accuracy, trajectory smoothness, identity preservation, latent disentanglement, and geodesic consistency arcoss datasets. The results are intended to support the main claim that Controlla improves structured controllability through latent geometry rather than only through stronger conditioning.

\subsection{Cross-Dataset Ablation Trends}

Table~\ref{tab:supp_cross_dataset_ablation} reports representative ablation trends across AffectHuman-43K, AffectNet, and CelebA-HQ. The goal is not to claim identical absolute performance across datasets, but to show that removing graph structure or multimodal inputs leads to similar relative degradation under domain shift. This supports the main-paper cross-dataset analysis and helps rule out the possibility that the gains arise only from the primary benchmark.

\begin{table*}[h]
\centering
\caption{
\textbf{Cross-dataset ablation trends.}
Representative ablations show similar relative trends across AffectHuman-43K, AffectNet, and CelebA-HQ. Full Controlla provides the strongest overall tradeoff, while removing graph structure or multimodal grounding reduces performance.
}
\scriptsize
\setlength{\tabcolsep}{4pt}
\renewcommand{\arraystretch}{1.08}
\rowcolors{2}{blue!4}{white}
\begin{tabular}{l|ccc|c}
\toprule
\rowcolor{blue!12}
\textbf{Setting} 
& \textbf{AffectHuman-43K Acc$\uparrow$} 
& \textbf{AffectNet Acc$\uparrow$} 
& \textbf{CelebA-HQ Acc$\uparrow$} 
& \textbf{GC$\downarrow$} \\
\midrule
Euclidean alignment        & 74.3 & 70.6 & 71.8 & 0.161 \\
w/o FGW                    & 74.8 & 71.9 & 72.9 & 0.141 \\
w/o GW                     & 75.1 & 72.1 & 73.0 & 0.134 \\
Image only                 & 69.1 & 67.3 & 68.0 & -- \\
Image + Text               & 74.2 & 71.7 & 72.8 & -- \\
\rowcolor{blue!12}
\textbf{Full Controlla}    & \textbf{76.4} & \textbf{72.8} & \textbf{73.6} & \textbf{0.126} \\
\bottomrule
\end{tabular}
\label{tab:supp_cross_dataset_ablation}
\end{table*}

\subsection{Hyperparameter Sensitivity}

Table~\ref{tab:supp_hyperparam_sensitivity} studies the most important hyperparameters: attribute graph weight, identity graph weight, graph connectivity, Sinkhorn iterations, entropy regularization, and orthogonality strength. We keep the table compact by reporting only the metrics most directly tied to the claims: Acc, TS, LDS, ID, and GC. Across settings, performance varies smoothly and remains close to the selected operating point, suggesting that Controlla does not rely on a brittle hyperparameter choice.

\begin{table*}[h]
\centering
\caption{
\textbf{Hyperparameter sensitivity on AffectHuman-43K.}
Performance varies smoothly across graph weights, graph connectivity, OT solver settings, entropy regularization, and orthogonality strength. Slightly stronger values may improve isolated metrics, but the selected configuration provides a stable tradeoff across Acc, TS, LDS, ID, and GC.
}
\scriptsize
\setlength{\tabcolsep}{3.5pt}
\renewcommand{\arraystretch}{1.08}
\rowcolors{2}{blue!4}{white}
\begin{tabular}{l|c|ccccc}
\toprule
\rowcolor{blue!12}
\textbf{Hyperparameter} 
& \textbf{Value} 
& \textbf{Acc$\uparrow$} 
& \textbf{TS$\uparrow$} 
& \textbf{LDS$\uparrow$} 
& \textbf{ID$\uparrow$} 
& \textbf{GC$\downarrow$} \\
\midrule

\multicolumn{7}{c}{\textit{Attribute graph weight}} \\
\midrule
$\lambda_e$ & 0.1 & 74.9 & 0.64 & 0.61 & \textbf{0.863} & 0.151 \\
$\lambda_e$ & 0.3 & 75.8 & 0.67 & 0.65 & 0.862 & 0.139 \\
$\lambda_e$ & 0.5 & 76.1 & 0.69 & 0.67 & 0.862 & 0.132 \\
\rowcolor{blue!12}
$\lambda_e$ & 1.0 & \textbf{76.4} & \textbf{0.71} & \textbf{0.69} & 0.862 & \textbf{0.126} \\
$\lambda_e$ & 2.0 & 76.2 & 0.70 & 0.68 & 0.859 & 0.129 \\

\midrule
\multicolumn{7}{c}{\textit{Identity graph weight}} \\
\midrule
$\lambda_i$ & 0.1 & 75.6 & 0.675 & 0.630 & 0.846 & 0.137 \\
$\lambda_i$ & 0.3 & 75.9 & 0.690 & 0.650 & 0.853 & 0.133 \\
$\lambda_i$ & 0.5 & 76.2 & 0.700 & 0.670 & 0.858 & 0.130 \\
\rowcolor{blue!12}
$\lambda_i$ & 1.0 & \textbf{76.4} & \textbf{0.710} & 0.690 & 0.862 & \textbf{0.126} \\
$\lambda_i$ & 2.0 & 76.1 & 0.704 & \textbf{0.695} & \textbf{0.864} & 0.129 \\

\midrule
\multicolumn{7}{c}{\textit{Graph connectivity}} \\
\midrule
$k$ & 5  & 76.0 & 0.690 & 0.670 & 0.861 & 0.131 \\
\rowcolor{blue!12}
$k$ & 10 & \textbf{76.4} & \textbf{0.710} & \textbf{0.690} & \textbf{0.862} & \textbf{0.126} \\
$k$ & 20 & 76.1 & 0.695 & 0.681 & 0.859 & 0.133 \\

\midrule
\multicolumn{7}{c}{\textit{OT solver and entropy regularization}} \\
\midrule
Sinkhorn iters. & 20  & 75.9 & 0.695 & 0.675 & -- & 0.132 \\
\rowcolor{blue!12}
Sinkhorn iters. & 50  & 76.4 & 0.710 & 0.690 & -- & 0.126 \\
Sinkhorn iters. & 100 & \textbf{76.5} & \textbf{0.712} & \textbf{0.691} & -- & \textbf{0.125} \\
$\epsilon$ & 0.01 & 75.8 & 0.700 & 0.682 & -- & 0.130 \\
\rowcolor{blue!12}
$\epsilon$ & 0.05 & \textbf{76.4} & \textbf{0.710} & \textbf{0.690} & -- & \textbf{0.126} \\
$\epsilon$ & 0.10 & 76.1 & 0.704 & 0.684 & -- & 0.129 \\

\midrule
\multicolumn{7}{c}{\textit{Orthogonality strength}} \\
\midrule
$\lambda_{\perp}$ & 0.0 & 73.9 & 0.620 & 0.590 & 0.853 & 0.152 \\
$\lambda_{\perp}$ & 0.1 & 75.2 & 0.660 & 0.630 & 0.857 & 0.140 \\
$\lambda_{\perp}$ & 0.5 & 76.0 & 0.695 & 0.670 & 0.860 & 0.130 \\
\rowcolor{blue!12}
$\lambda_{\perp}$ & 1.0 & \textbf{76.4} & \textbf{0.710} & \textbf{0.690} & \textbf{0.862} & \textbf{0.126} \\

\bottomrule
\end{tabular}
\label{tab:supp_hyperparam_sensitivity}
\end{table*}

\subsection{Latent Decomposition and Disentanglement Analysis}

\paragraph{Fixed Latent Decomposition.}
In Controlla, the latent representation is decomposed into attribute and identity subspaces using a fixed dimensional split. This design choice is motivated by stability considerations when optimizing graph-constrained objectives (FGW/GW), which require consistent semantic roles across subspaces. A fixed partition reduces early-stage mixing between identity and attribute factors, which can otherwise destabilize optimal transport alignment.

\paragraph{Is Disentanglement Engineered or Learned?}
While the dimensional split provides an architectural bias, disentanglement is not solely induced by this partition. Instead, it is enforced through:
(i) graph-constrained optimal transport objectives (FGW/GW), which align attribute and identity subspaces with distinct relational structures,
(ii) orthogonality constraints that explicitly reduce cross-subspace leakage, and
(iii) Lipschitz regularization that stabilizes transformations under latent traversal.
Together, these mechanisms ensure that disentanglement emerges from geometry-aware learning rather than from dimensional partitioning alone.

\paragraph{Sensitivity to Latent Split.}
We evaluate robustness to different dimensional allocations between attribute and identity subspaces. 

\begin{table*}[h]
\centering
\caption{
\textbf{Sensitivity to latent split.}
Performance remains stable across different attribute/identity dimensional allocations, indicating that disentanglement is not solely induced by dimensional partitioning but emerges from the proposed geometry-aware objectives.
}
\scriptsize
\setlength{\tabcolsep}{6pt}
\renewcommand{\arraystretch}{1.08}
\rowcolors{2}{blue!4}{white}
\begin{tabular}{l|ccc}
\toprule
\rowcolor{blue!12}
\textbf{Split (attr/id)} 
& \textbf{Acc$\uparrow$} 
& \textbf{TS$\uparrow$} 
& \textbf{LDS$\uparrow$} \\
\midrule
128 / 128 & 76.9 & 0.70 & 0.68 \\
\rowcolor{blue!12}
\textbf{192 / 64} & \textbf{77.6} & \textbf{0.73} & \textbf{0.71} \\
224 / 32 & 77.1 & 0.71 & 0.69 \\
\bottomrule
\end{tabular}
\label{tab:supp_latent_split}
\end{table*}

As shown in Table~\ref{tab:supp_latent_split}, performance varies only marginally across different splits, with a slight advantage for the 192/64 attribute--identity allocation.
If disentanglement were purely a consequence of dimensional partitioning, performance would vary substantially across configurations.
Instead, the observed stability suggests that disentanglement is driven primarily by structural constraints, specifically graph alignment through FGW/GW and orthogonality regularization, rather than by architectural separation alone.

\paragraph{Complexity vs. Performance Trade-off.}
Controlla introduces additional components, including graph priors and graph-constrained optimal transport (FGW/GW), which incur modest computational overhead. However, the primary goal of Controlla is not to maximize marginal improvements in task-level accuracy, but to enable structured and controllable transformations in latent space.

While improvements in accuracy (Acc) are moderate (approximately 2–3\%), we observe substantially larger gains in geometry-aware and controllability-specific metrics, including trajectory smoothness (TS), latent disentanglement (LDS), and geodesic consistency (GC). These metrics directly reflect the ability of the model to perform stable, identity-preserving transformations—capabilities not captured by standard accuracy measures.

Furthermore, qualitative results (Sec. \hyperref[sec10]{J}, Sec. \hyperref[sec11]{K}) demonstrate that Controlla produces smoother and more coherent intermediate states, whereas conditioning-based baselines exhibit identity drift and inconsistent transitions.

From a computational perspective, the additional cost is limited to training: graph-constrained optimal transport introduces approximately $\sim$8\% overhead (Sec. 4.3), while inference cost remains largely unchanged. This makes Controlla practical for deployment while providing improved controllability.

Overall, the added complexity is justified by a qualitative shift from heuristic, prompt-driven control to geometry-driven, structured controllability, which is the central objective of this work.

\section{Additional Qualitative Analysis}
\label{sec10}

We provide additional qualitative results to complement the quantitative evaluation in the main paper. The goal of this section is to examine model behavior across datasets, baseline families, fine-grained control settings, and challenging scenarios. Examples are selected to reflect representative outcomes across multiple samples rather than best-case outputs.

\paragraph{Selection protocol.}
For each setting, multiple generations are produced per method, and representative examples are shown. The evaluation spans different datasets, identities, and emotion categories to reduce bias toward specific instances.

\paragraph{Cross-dataset comparison.}
Figure~\ref{fig:cross_dataset} compares model behavior across AffectHuman-43K, CelebA-HQ, and AffectNet. Across datasets with varying distributions and visual characteristics, we observe differences in how methods maintain identity and follow target expressions. While all methods demonstrate varying levels of performance depending on the dataset, Controlla exhibits relatively consistent behavior across datasets in the presented examples, particularly in maintaining recognizable identity while adapting expressions.

\begin{figure*}[h]
\centering
\includegraphics[width=\linewidth]{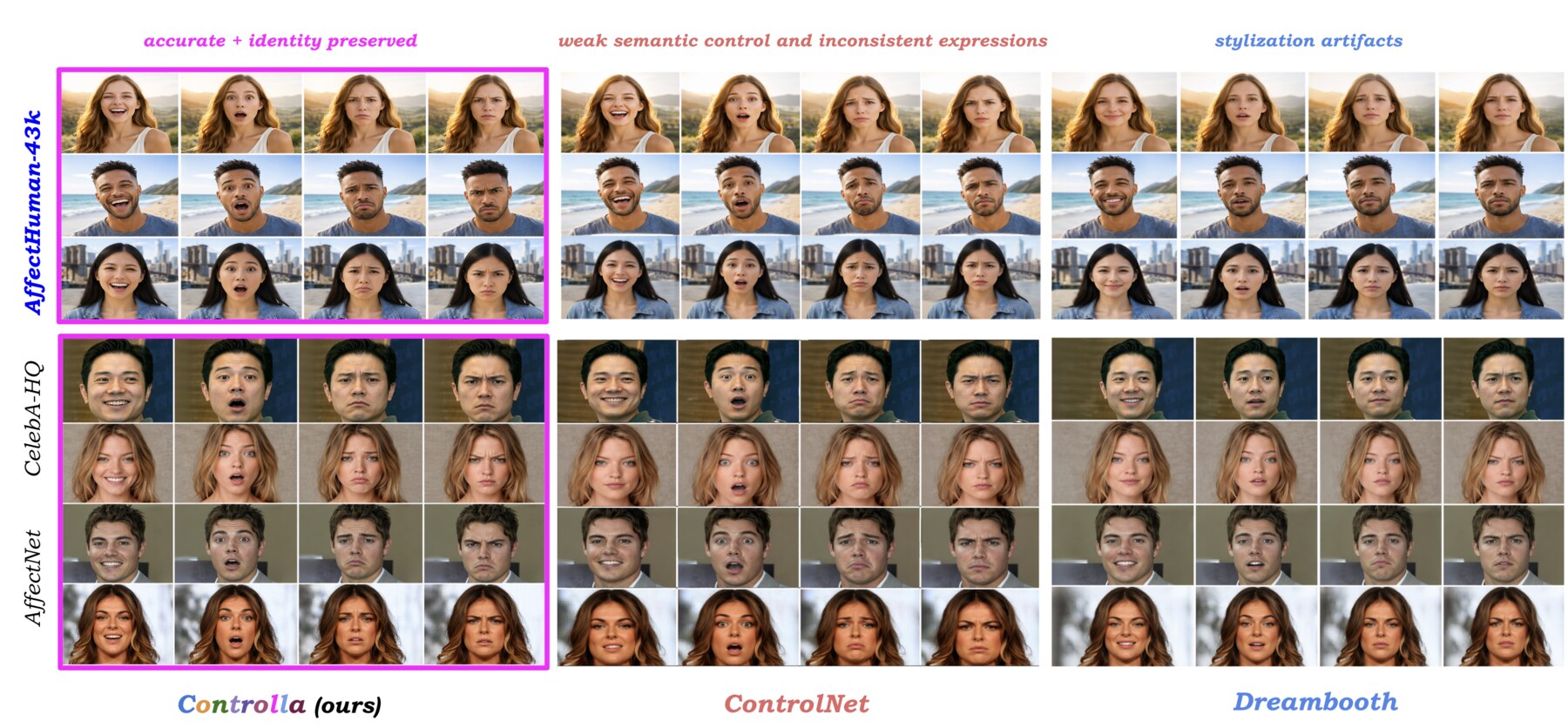}
\caption{
\textbf{Cross-dataset comparison across AffectHuman-43K, CelebA-HQ, and AffectNet.}
}
\label{fig:cross_dataset}
\end{figure*}

\paragraph{Baseline comparison.}
Figure~\ref{fig:baseline_comparison} presents comparisons across representative baseline methods spanning conditioning-based control, instruction-guided editing, and unified generation approaches. Across these method families, we observe recurring patterns such as variability in identity preservation, inconsistent expression intensity, and occasional visual artifacts. Controlla demonstrates comparatively stable behavior in these examples, although variations remain across different inputs and conditions.

\begin{figure*}[h]
\centering
\includegraphics[width=\linewidth]{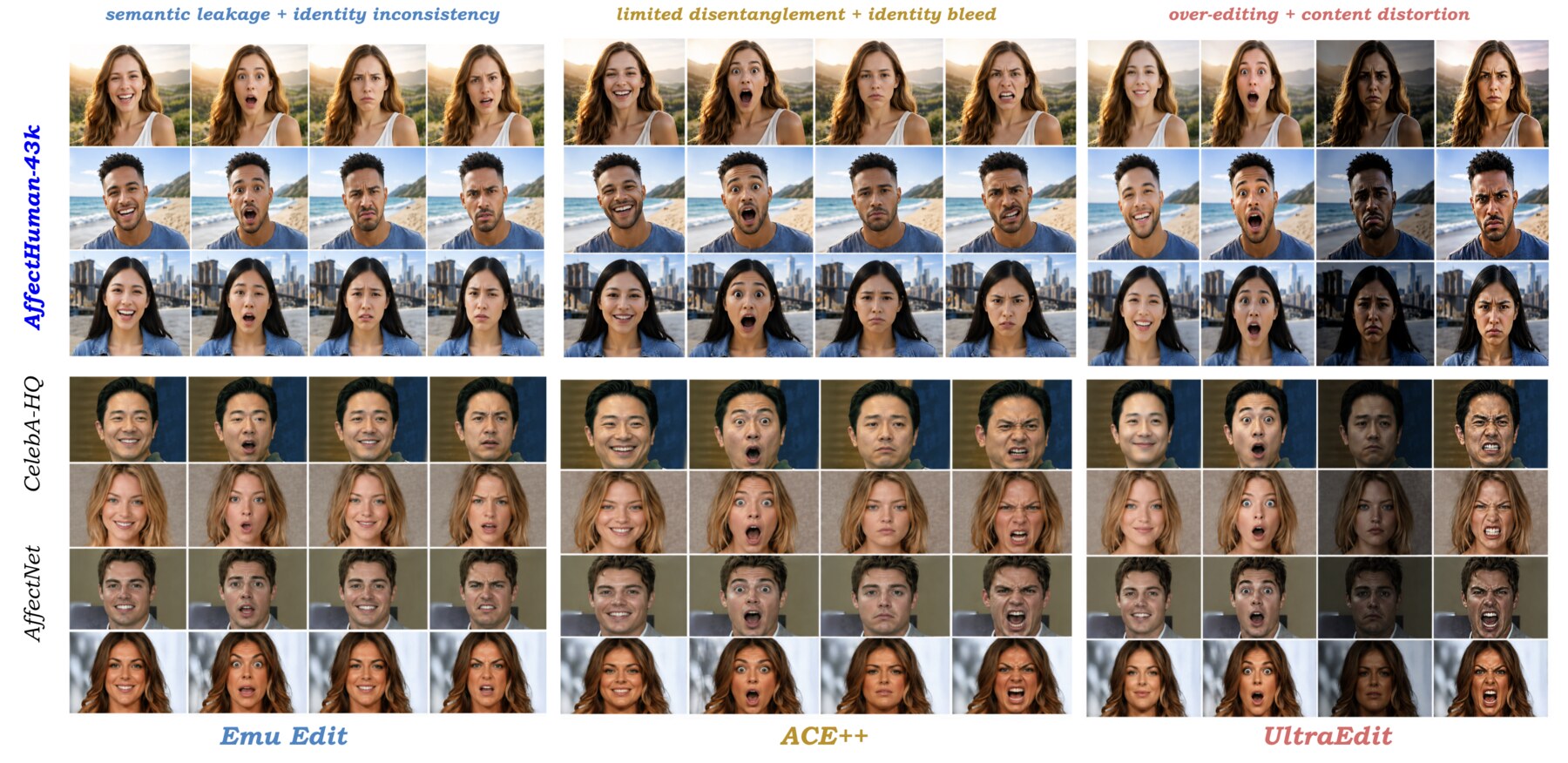}
\caption{
\textbf{Comparison across representative baseline methods.}
}
\label{fig:baseline_comparison}
\end{figure*}

\paragraph{Multi-identity evaluation.}
Figure~\ref{fig:multi_identity} evaluates performance across multiple identities under the same target conditions. This setting highlights how methods generalize across diverse facial structures and appearances. We observe that some methods exhibit identity drift or inconsistent attribute application across identities. In contrast, Controlla maintains more consistent identity features across the shown examples while adapting expressions.

\begin{figure*}[h]
\centering
\includegraphics[width=\linewidth]{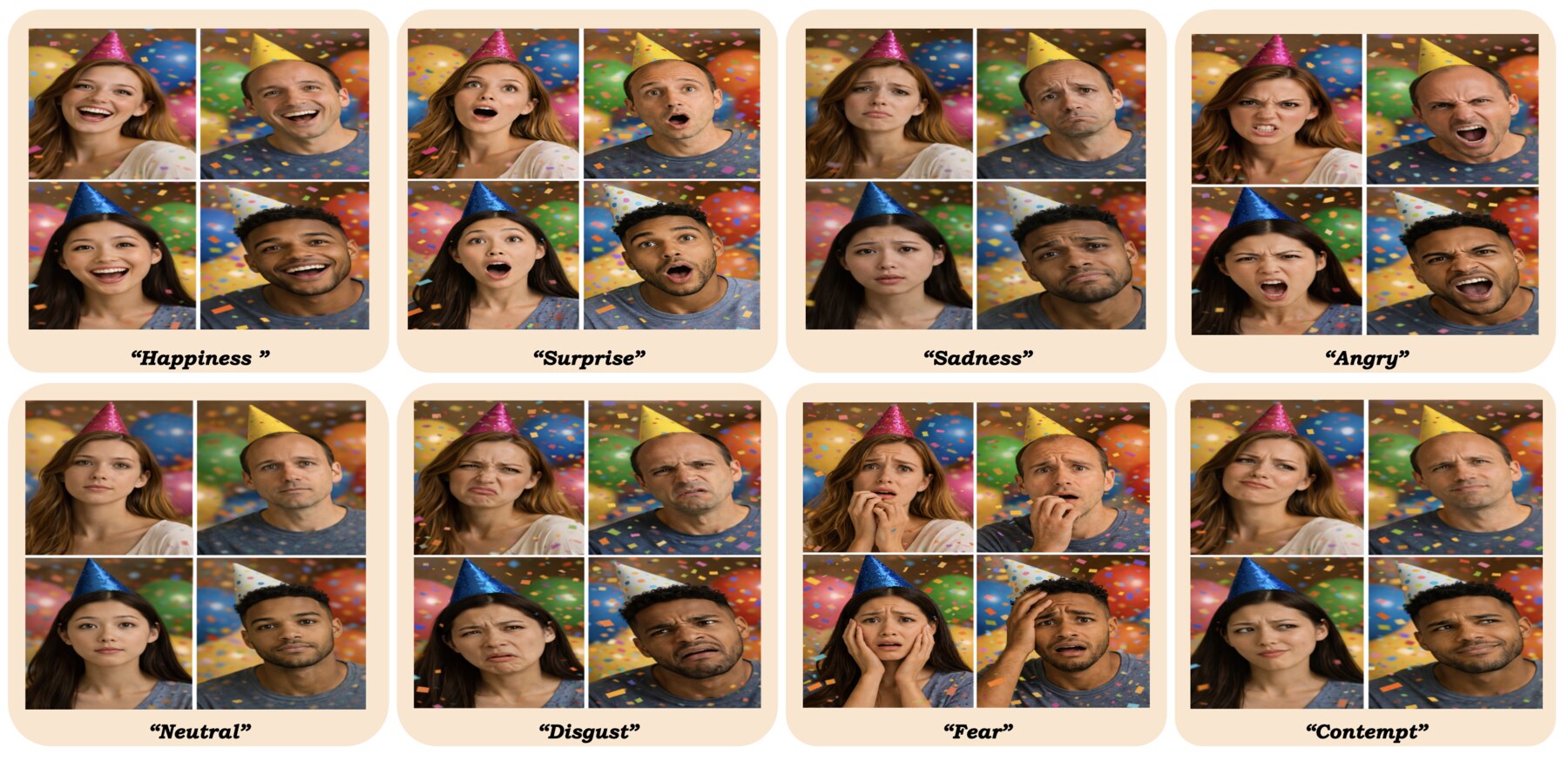}
\caption{
\textbf{Multi-identity evaluation under shared target conditions.}
}
\label{fig:multi_identity}
\end{figure*}

\paragraph{Fine-grained emotion control.} 
Figure \ref{fig:fine_grained_emotion} illustrates fine-grained variations within emotion categories, including transitions between closely related affective states. These examples demonstrate the ability to represent subtle differences such as varying intensity or nuanced expressions within the same emotion class. While all methods capture coarse emotional changes, differences emerge in the consistency and smoothness of intermediate variations. Controlla shows smoother transitions across related expressions in the presented sequences.

\begin{figure*}[h]
\centering
\includegraphics[width=0.72\linewidth]{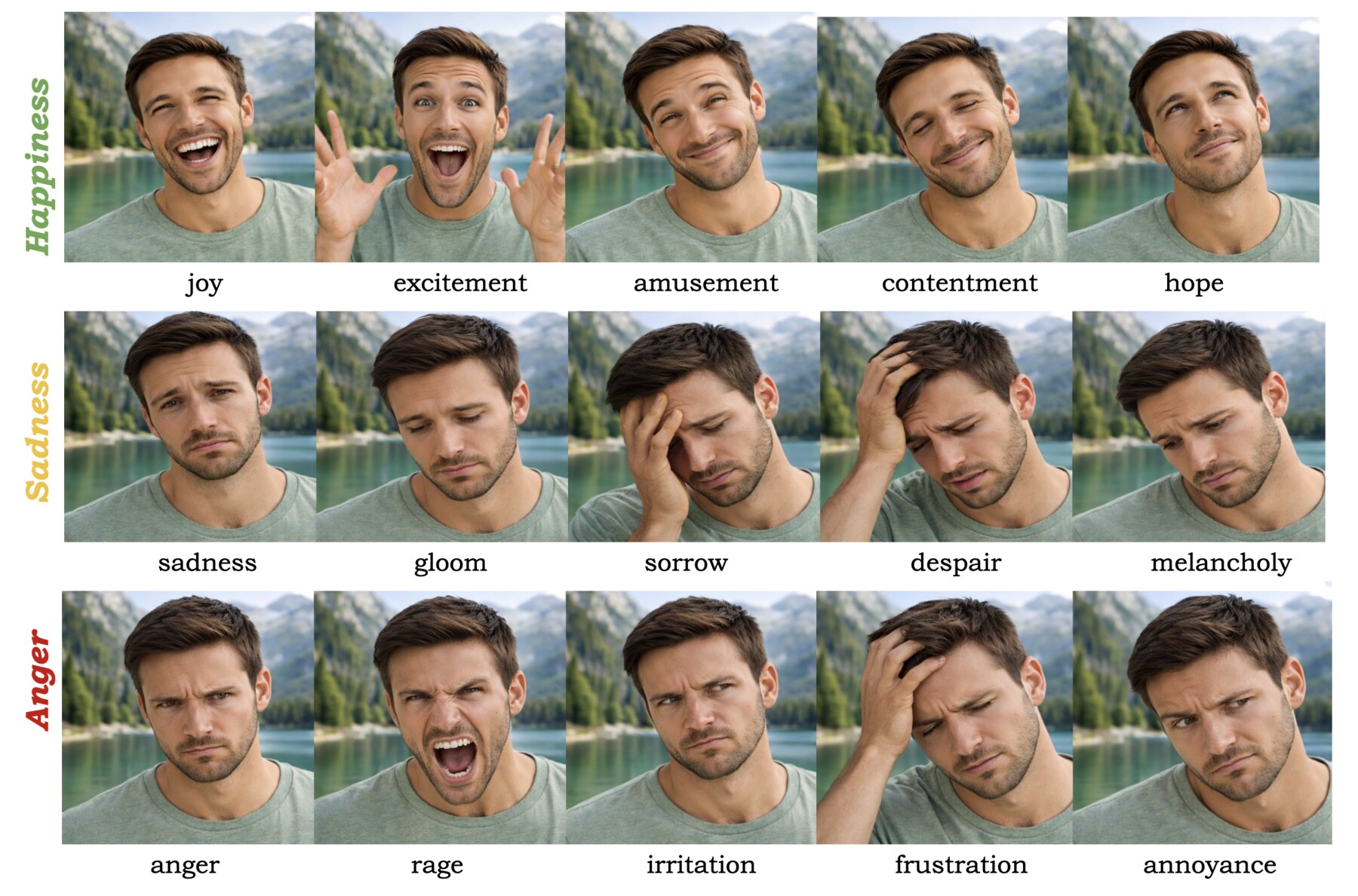}
\includegraphics[width=0.72\linewidth]{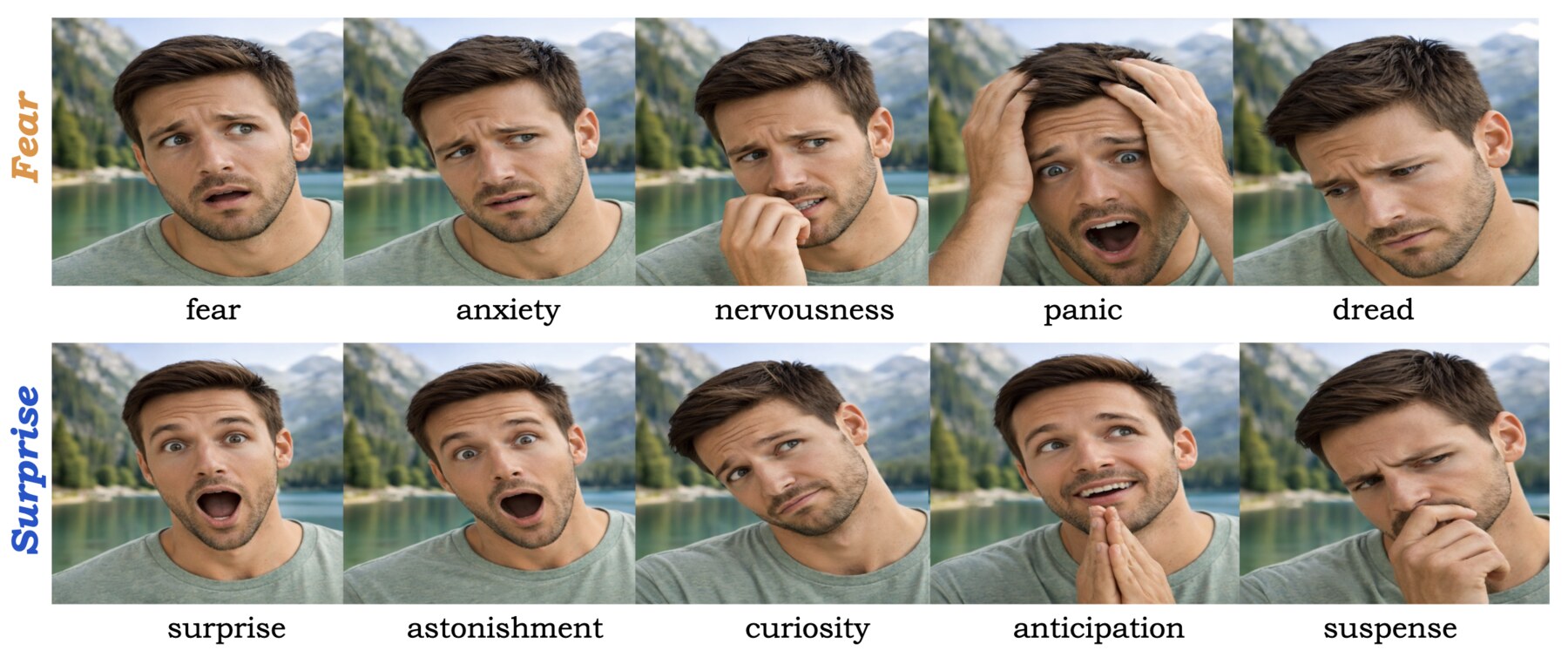}
\includegraphics[width=0.72\linewidth]{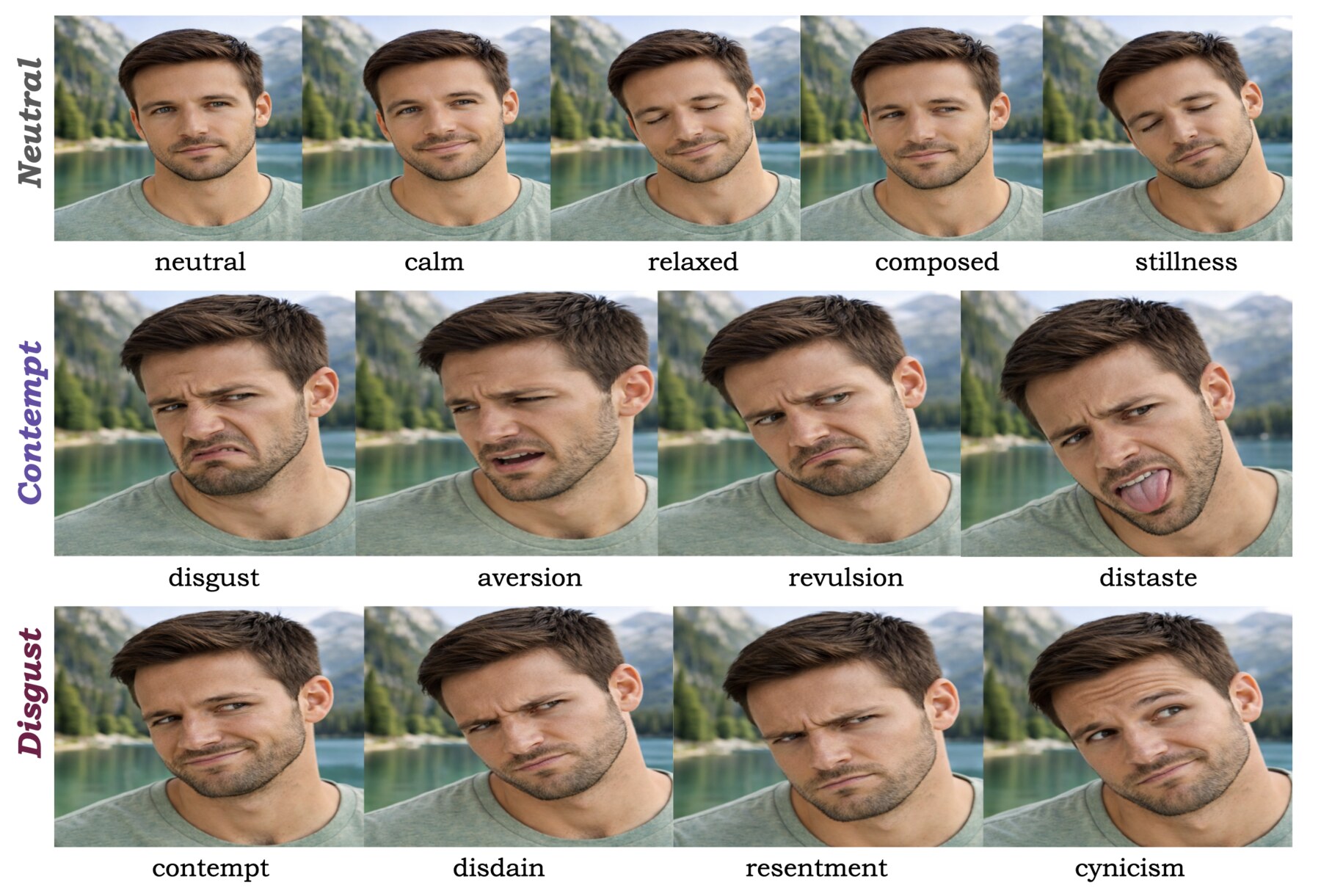}
\caption{
\textbf{Fine-grained variations within high-level emotion categories.}
}
\label{fig:fine_grained_emotion}
\end{figure*}

\paragraph{Challenging subtle cases.}
Figures ~\ref{fig:challenging_1} and ~\ref{fig:challenging_2} present challenging scenarios involving visually similar or overlapping expressions. These cases require distinguishing fine-grained semantic cues that are often difficult to separate. Across methods, we observe variability in handling these distinctions, with some outputs collapsing to more generic expressions. Controlla demonstrates more consistent differentiation in these examples, though performance varies depending on the specific case.

\begin{figure*}[h]
\centering
\includegraphics[width=0.90\linewidth]{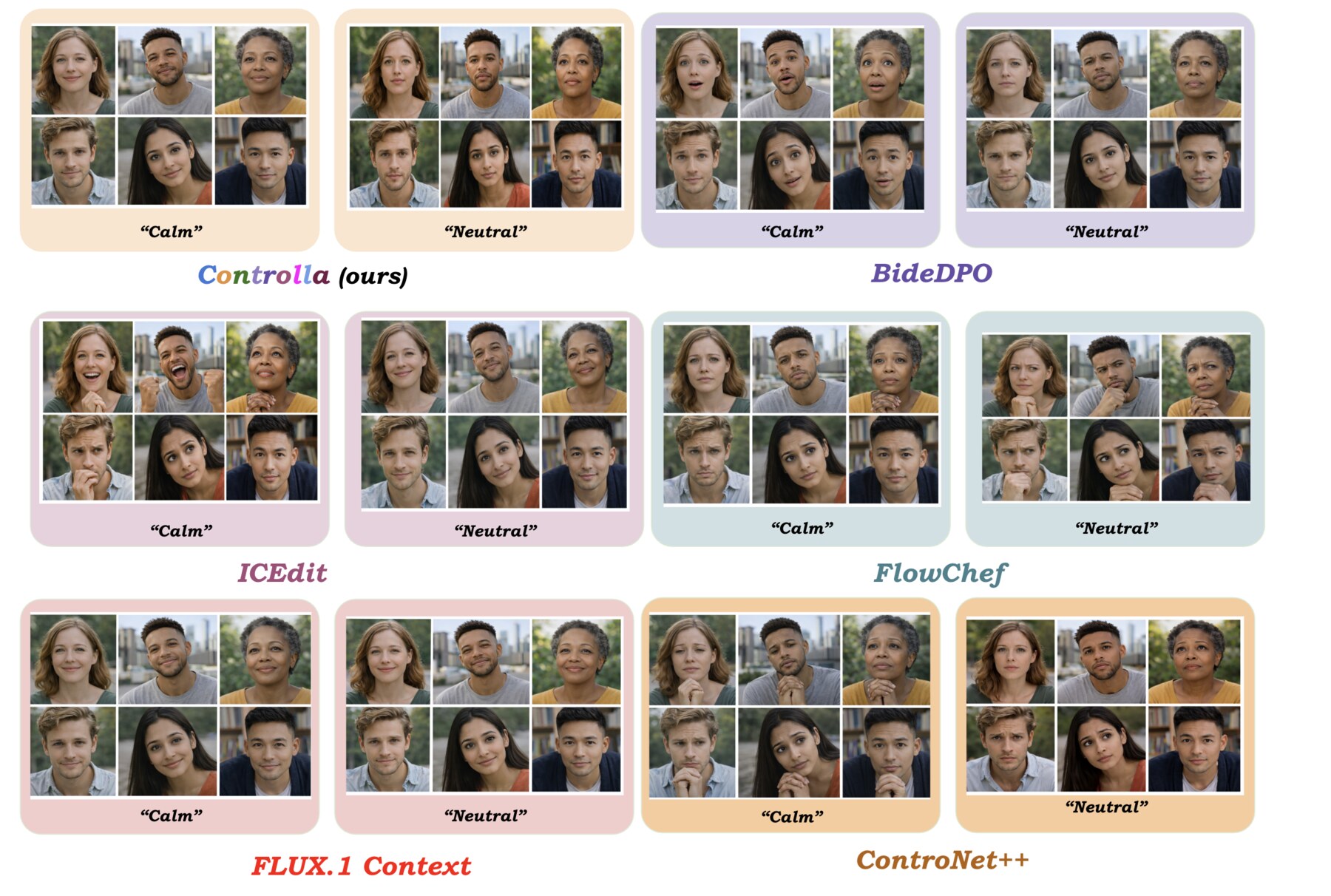}
\caption{
\textbf{Challenging cases involving visually similar expressions.}
}
\label{fig:challenging_1}
\includegraphics[width=0.90\linewidth]{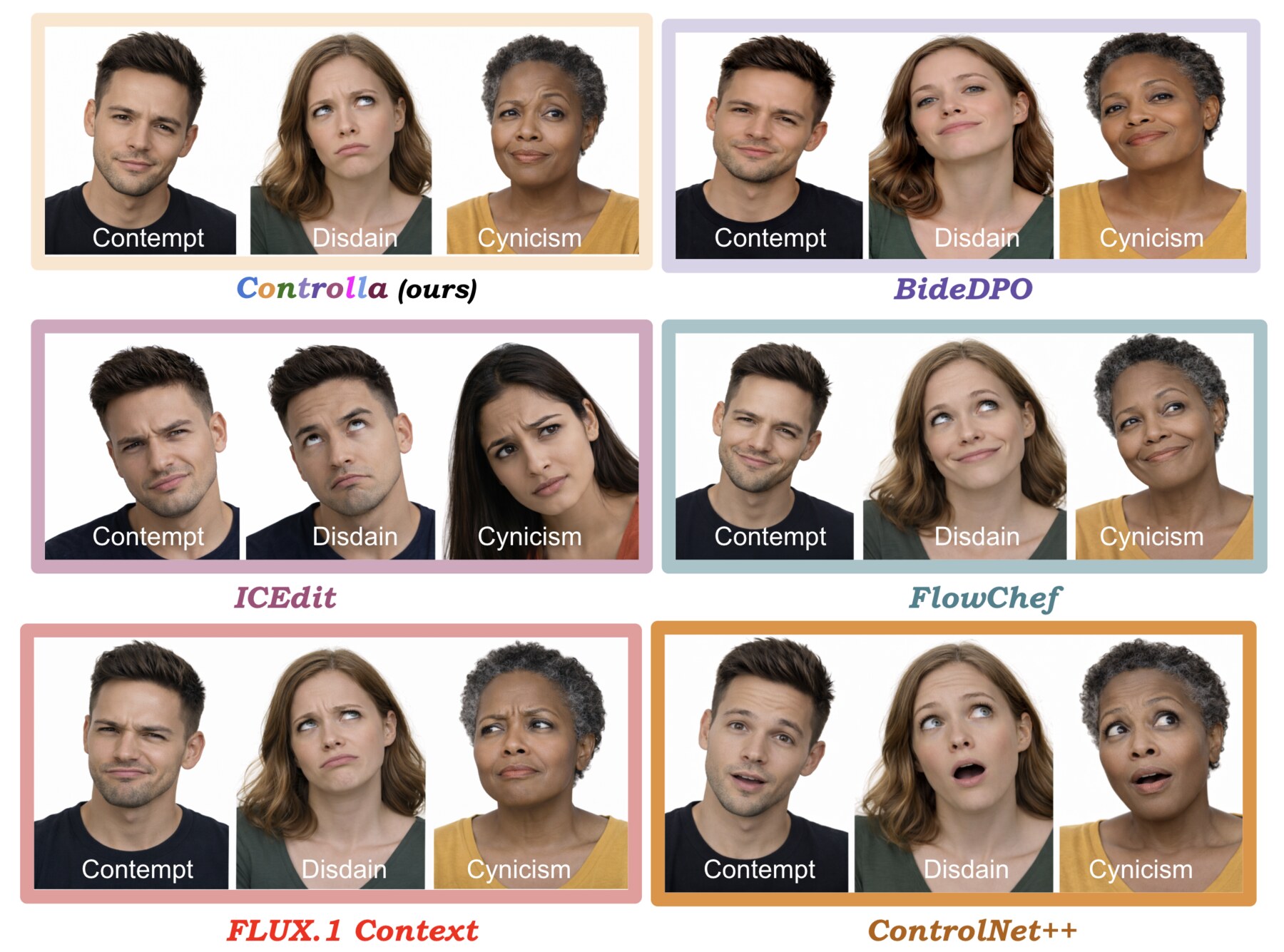}
\caption{
\textbf{Challenging cases with subtle semantic differences.}
}
\label{fig:challenging_2}
\end{figure*}

\paragraph{Multimodal ambiguity.}
Figures~\ref{fig:ambiguity_1} and~\ref{fig:ambiguity_2} examine scenarios with potentially conflicting or ambiguous multimodal inputs. These cases test how models integrate multiple signals when cues are not fully aligned. Across methods, outputs may vary in how different signals are prioritized. Controlla produces more consistent outputs across the presented examples, suggesting improved integration of multimodal information under ambiguity.

\begin{figure*}[h]
\centering
\includegraphics[width=\linewidth]{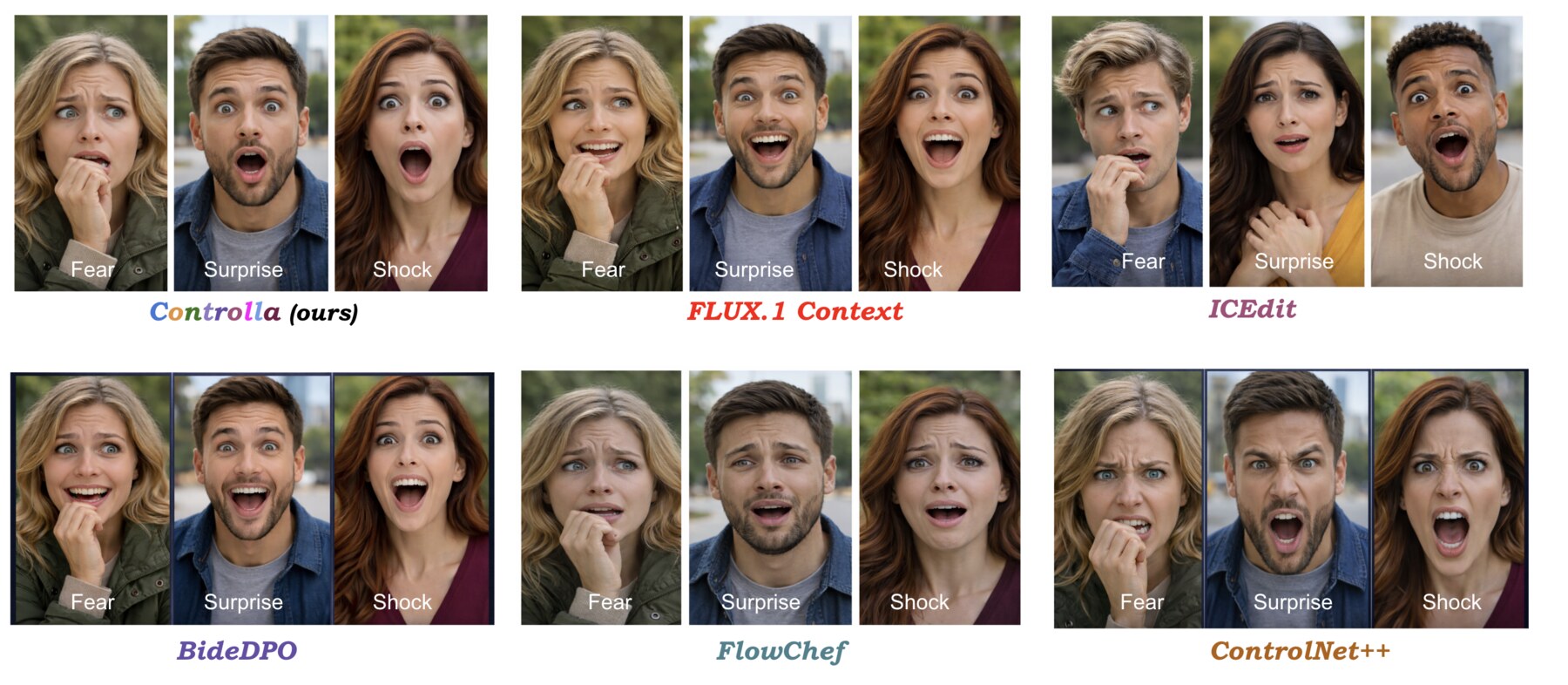}
\caption{
\textbf{Multimodal ambiguity with partially conflicting cues.}
}
\label{fig:ambiguity_1}
\end{figure*}

\begin{figure*}[h]
\centering
\includegraphics[width=\linewidth]{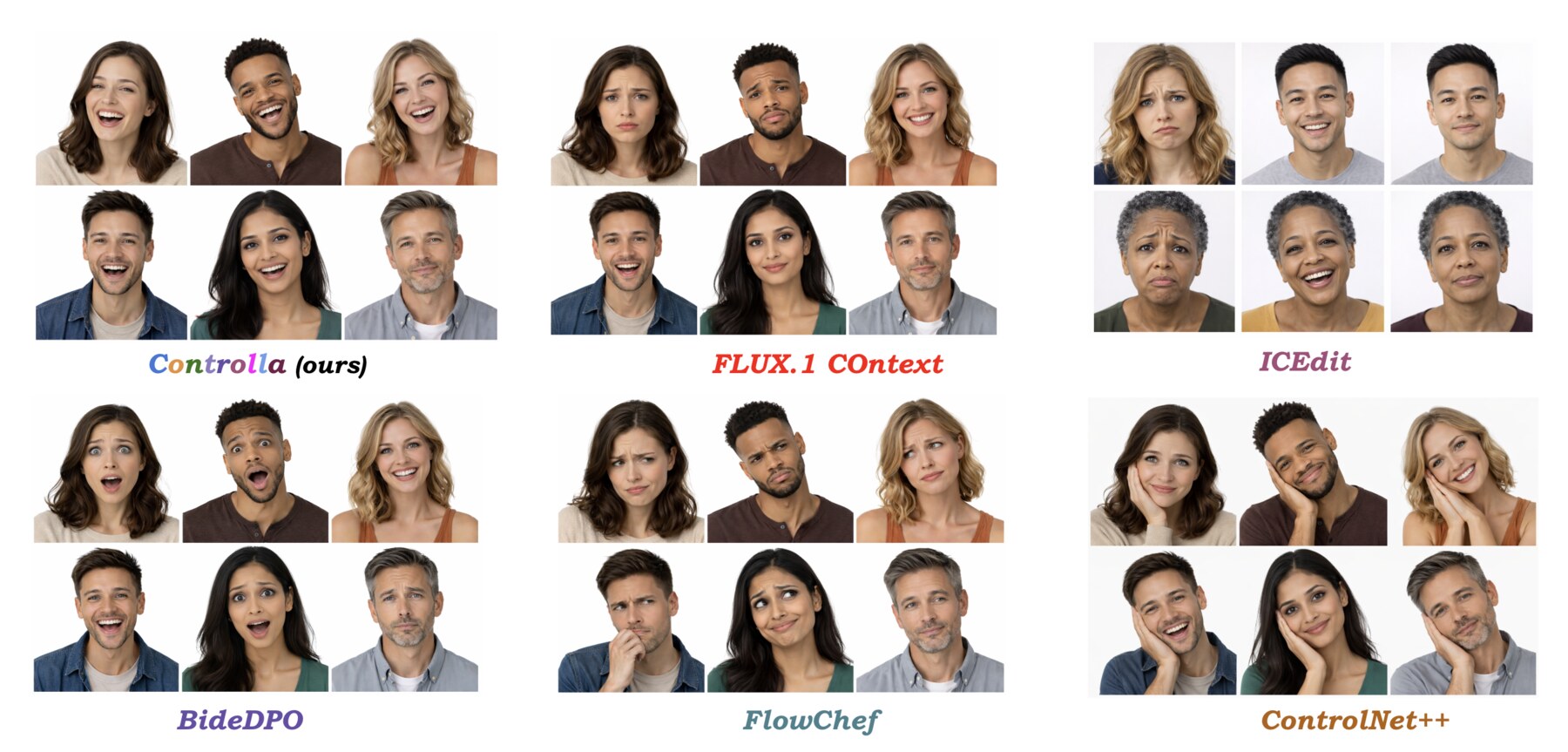}
\caption{
\textbf{Multimodal ambiguity under diverse identity and expression settings.}
}
\label{fig:ambiguity_2}
\end{figure*}

\section{Behavior Under Challenging Conditions}
\label{sec11}

We analyze model behavior under challenging conditions to better understand how structured priors, multimodal interactions, and optimization constraints influence generation. 
Rather than focusing on best-case outputs, this section examines representative scenarios where inputs are ambiguous, partially conflicting, or where model design choices introduce trade-offs.
These cases reveal how Controlla balances attribute control with identity preservation when the requested transformation becomes difficult or under-specified.
They also help distinguish controllability failures caused by weak graph alignment from failures caused by insufficient visual or cross-modal evidence. Overall, the analysis provides a more realistic view of robustness beyond aggregate metrics and curated qualitative examples.

\vspace{-6pt}

\subsection{Sensitivity to Structured Priors}
Controlla relies on predefined emotion and identity graphs to guide latent alignment. As a result, model behavior depends on how well these priors reflect underlying semantic relationships.

Figure~\ref{fig:fail_wrong_emotion_prior} illustrates a scenario with conflicting multimodal inputs (smiling image, calm audio, and strongly negative text). In such cases, the model produces intermediate expressions that reflect competing signals rather than a single dominant target.

This behavior indicates that outputs are influenced by the interaction between structured priors and multimodal inputs, with the model interpolating between conflicting cues rather than collapsing to one modality.

\begin{figure*}[h]
\centering
\includegraphics[width=0.92\textwidth]{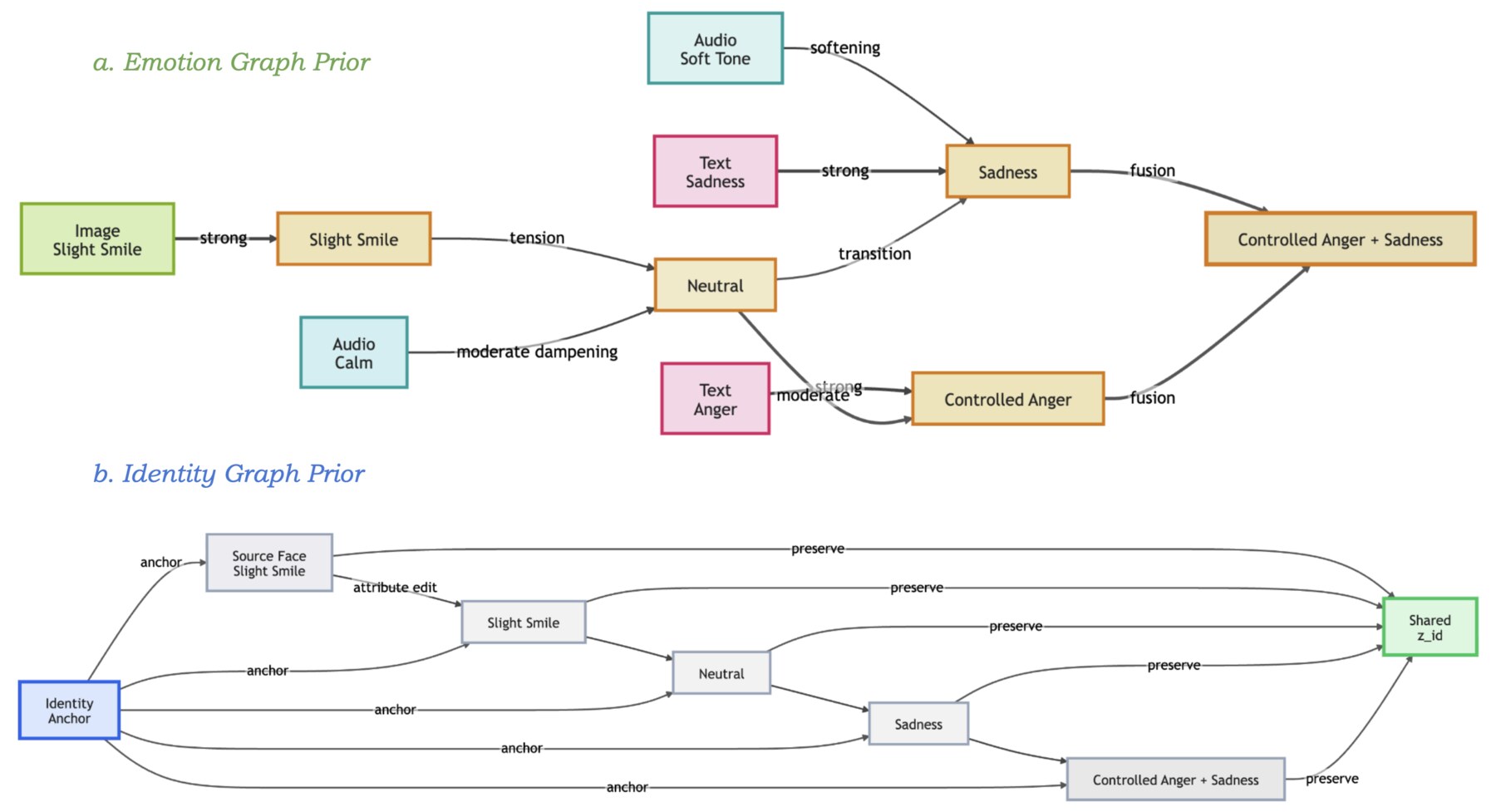}
\includegraphics[width=0.92\textwidth]{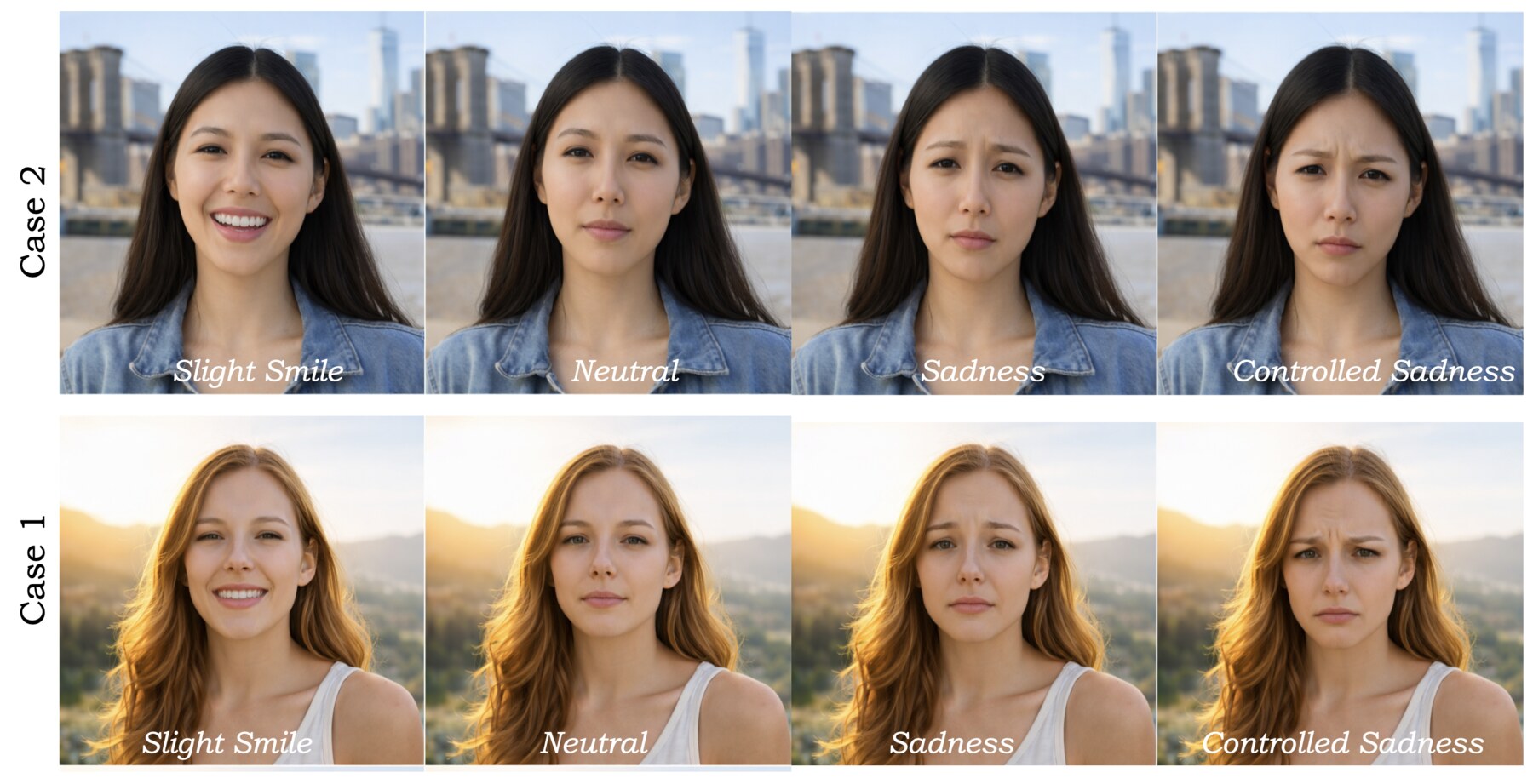}
\caption{
\textbf{Behavior under conflicting multimodal signals.}
Inputs include a smiling image, calm audio, and strongly negative text.
The resulting outputs reflect a combination of these cues, illustrating how structured priors guide latent transitions when signals are not fully aligned.
}
\label{fig:fail_wrong_emotion_prior}
\end{figure*}

\vspace{-6pt}

\subsection{Controllability--Diversity Trade-off}
Graph regularization introduces a trade-off between semantic consistency and expressive diversity. Increasing $\lambda$ strengthens alignment with target structure, while reducing variability in generated outputs.

As shown in Fig.~\ref{fig:lambda_tradeoff} and Table~\ref{tab:lambda_regularization}, lower regularization produces more diverse but less consistent outputs, whereas higher regularization yields more stable but less varied expressions. Intermediate values provide a balance between these factors.

This behavior reflects how graph constraints shape the latent space and influence the trade-off between controllability and diversity.

\begin{figure*}[h]
\centering
\includegraphics[width=\linewidth]{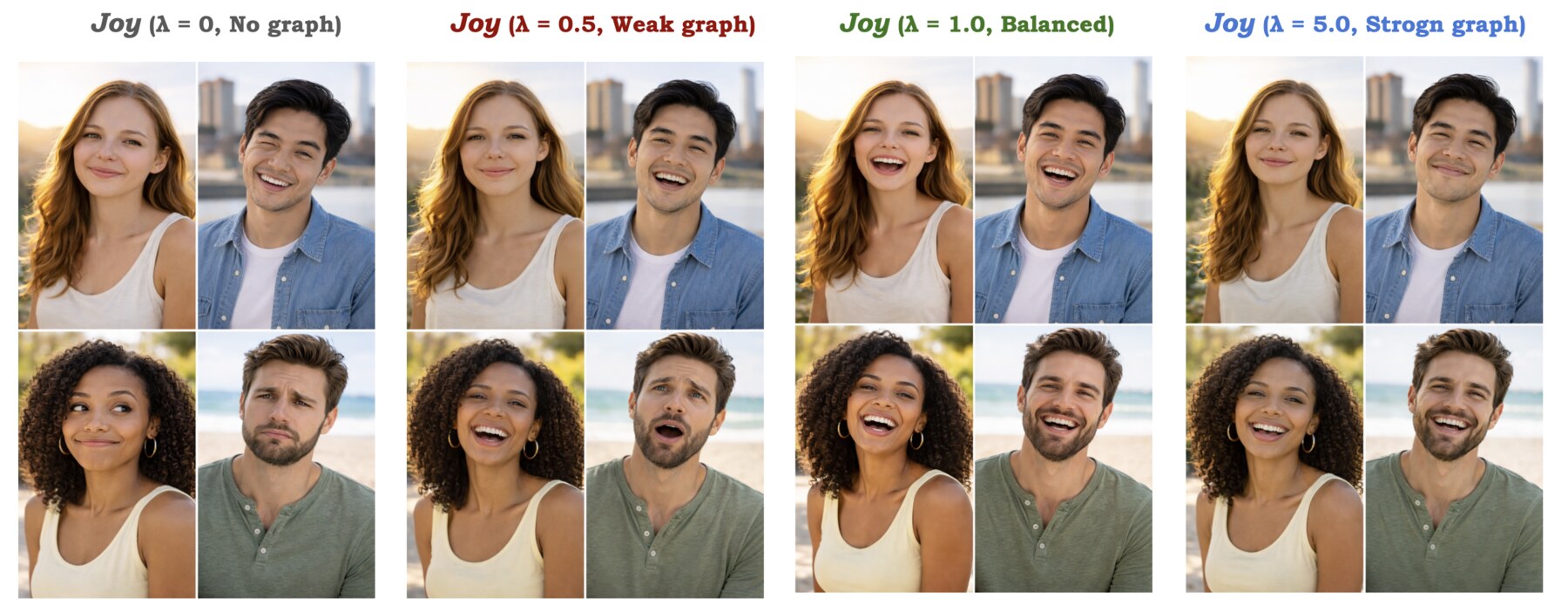}
\caption{
\textbf{Effect of graph regularization strength.}
Varying $\lambda$ changes the balance between semantic consistency and expressive variability.
Lower values produce more diverse outputs, while higher values yield more consistent but less varied expressions.
}
\label{fig:lambda_tradeoff}
\end{figure*}

\vspace{-6pt}

\subsection{Behavior Under Multimodal Ambiguity.}

Controlla is designed for settings in which image, text, and audio provide complementary affective cues, but in practice these signals may be partially ambiguous or weakly conflicting. For example, the visual input may suggest a neutral or mildly positive expression, while the text requests a stronger affective state, or the audio may contain noisy prosodic evidence that only partially matches the text. In such cases, there may not be a single unambiguous target emotion.

Fig.~\ref{fig:multimodal_ambiguity_failure} illustrates this behavior qualitatively, and Table~\ref{tab:multimodal_ambiguity} reports the corresponding quantitative trends. When only one modality is available, performance drops because the model has limited evidence for the intended affective transformation. When text or audio is noisy, Controlla still maintains relatively stable identity and graph consistency, suggesting that the factorized representation does not collapse under imperfect multimodal input. The full clean setting performs best, while the full setting with one noisy modality shows only a moderate degradation, indicating that the model can partially compensate using the remaining aligned modalities.

This behavior suggests that Controlla integrates multimodal signals through the learned attribute factor rather than relying on a single dominant cue. However, the model does not include an explicit conflict-resolution module. Therefore, when modalities provide strongly contradictory affective instructions, the output may reflect a compromise between signals or follow the modality with the strongest learned affective evidence. We view this as a realistic limitation of the current formulation: structured latent geometry improves stability under ambiguity, but explicit reasoning about modality reliability remains an important direction for future work.

\begin{figure*}[h]
\centering
\includegraphics[width=\linewidth]{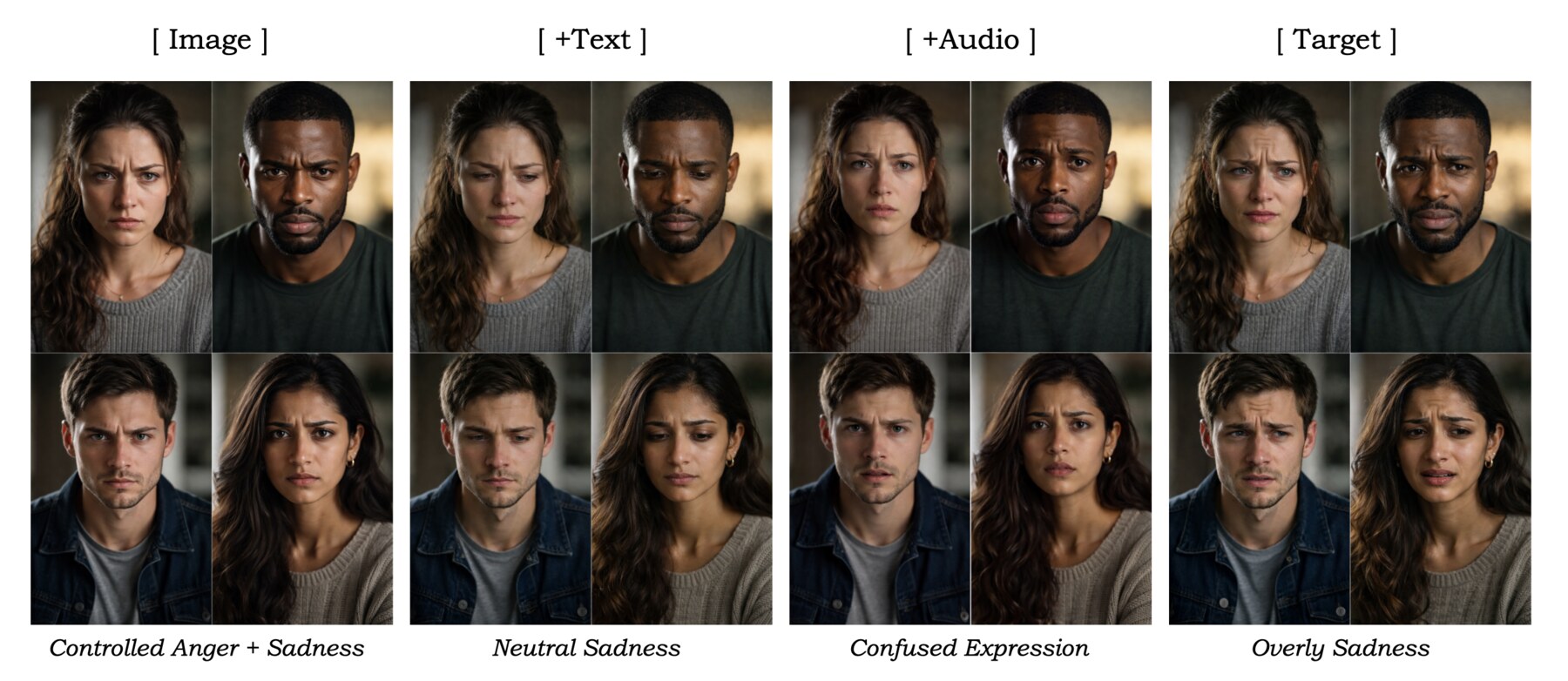}
\caption{
\textbf{Multimodal ambiguity.}
Partially conflicting image, text, and audio cues produce outputs that combine or prioritize affective evidence across modalities.
}
\label{fig:multimodal_ambiguity_failure}
\end{figure*}

\vspace{6pt}

\subsection{Computational Considerations.}

The main computational overhead in Controlla comes from structured alignment with graph-constrained optimal transport. In addition to the frozen diffusion backbone and trainable conditioning modules, Controlla computes pairwise distances among latent factors and aligns these distances with the emotion and identity graph distance matrices. This adds cost during training because GW and FGW objectives require relational comparisons and iterative transport-plan optimization.

Table~\ref{tab:computational_cost} summarizes this overhead. Relative to the baseline diffusion setup, adding graph construction increases cost modestly, while GW and FGW alignment introduce the larger overhead due to pairwise distance computation and transport optimization. In the reported setting, the full structured alignment remains practical, increasing relative cost while preserving manageable memory usage. Importantly, most of this cost is concentrated during training. At inference time, graph distances and transport structure can be cached or avoided, since generation only requires the learned factorization and conditioning adapters.

Several implementation choices help keep the method efficient. First, graph construction is performed offline or cached, so it does not need to be recomputed at every inference step. Second, mini-batch optimal transport reduces the cost of computing alignment over the full dataset. Third, entropic regularization stabilizes transport optimization and improves numerical efficiency. Finally, the base diffusion generator remains frozen, so training focuses on the projection adapters, factorization heads, graph-alignment module, and conditioning adapters rather than the entire generator.

Scaling Controlla to larger datasets, denser graphs, or higher-dimensional latent factors may require further optimization. Possible improvements include approximate nearest-neighbor graph construction, sparse transport solvers, lower-rank distance approximations, and more aggressive caching of pairwise distances. Overall, the current setup provides a practical trade-off: structured latent geometry introduces additional training cost, but this cost is moderate relative to the gains in controllability, identity preservation, and graph-consistent traversal.

\begin{table*}[t]
\centering
\caption{\textbf{Analysis of model behavior under structured regularization, multimodal ambiguity, and computational constraints.}
The results illustrate how graph regularization influences consistency and variability, how multimodal inputs are integrated under ambiguity, and the associated computational overhead.}
\label{tab:structured_regularization_analysis}

\begin{subtable}[t]{0.42\textwidth}
\centering
\caption{Effect of graph regularization strength.}
\label{tab:lambda_regularization}
\begin{tabular}{c|cccc}
\toprule
\rowcolor{gray!25}
$\lambda$ & \textbf{Acc} & \textbf{CLIP} & \textbf{LPIPS} & \textbf{GC$\downarrow$} \\
\midrule
0.1 & 72.3 & 0.298 & 0.341 & 0.181 \\
\rowcolor{gray!10}
0.3 & 75.8 & 0.312 & 0.318 & 0.149 \\
0.5 & 78.1 & 0.329 & 0.287 & 0.121 \\
\rowcolor{blue!15}
1.0 & \textbf{79.9} & \textbf{0.349} & 0.254 & \textbf{0.093} \\
2.0 & 79.6 & 0.345 & \textbf{0.241} & 0.095 \\
\bottomrule
\end{tabular}
\end{subtable}
\hfill
\begin{subtable}[t]{0.50\textwidth}
\centering
\caption{Behavior under multimodal ambiguity.}
\label{tab:multimodal_ambiguity}
\begin{tabular}{l|cccc}
\toprule
\rowcolor{gray!25}
\textbf{Setting} & \textbf{Acc} & \textbf{CLIP} & \textbf{ID} & \textbf{GC$\downarrow$} \\
\midrule
Image only & 69.1 & 0.281 & 0.862 & 0.182 \\
\rowcolor{gray!10}
Noisy Text & 73.2 & 0.296 & 0.874 & 0.149 \\
Noisy Audio & 72.4 & 0.292 & 0.871 & 0.153 \\
\rowcolor{blue!15}
Full (clean) & \textbf{79.9} & \textbf{0.349} & \textbf{0.918} & \textbf{0.093} \\
Full (1 noisy) & 78.3 & 0.334 & 0.907 & 0.107 \\
\bottomrule
\end{tabular}
\end{subtable}

\vspace{1em}

\begin{subtable}[t]{0.72\textwidth}
\centering
\caption{Computational cost of structured alignment.}
\label{tab:computational_cost}
\begin{tabular}{l|ccc}
\toprule
\rowcolor{gray!25}
\textbf{Component} & \textbf{Time (ms)} & \textbf{Memory (GB)} & \textbf{Relative Cost} \\
\midrule
Baseline Diffusion & 120 & 6.2 & $1.0\times$ \\
\rowcolor{gray!10}
+ Graph Construction & 145 & 6.8 & $1.2\times$ \\
+ GW Alignment & 198 & 7.5 & $1.6\times$ \\
\rowcolor{gray!10}
+ FGW Alignment & 224 & 7.9 & $1.8\times$ \\
\bottomrule
\end{tabular}
\end{subtable}
\end{table*}
\section{Extensibility Beyond Affective Control}
\label{sec12}

\subsection{Framework Generality}

Controlla is formulated as a graph-conditioned latent control framework, rather than as an emotion-specific architecture. 
The method requires two abstract ingredients: 
(i) an \emph{attribute graph} $\mathcal{G}_a$, which defines the semantic factor to be controlled, and 
(ii) a \emph{preserved-factor graph} $\mathcal{G}_p$, which defines the factor that should remain stable during traversal. 
Graph-constrained optimal transport then aligns the controllable latent factor with $\mathcal{G}_a$ and the preserved latent factor with $\mathcal{G}_p$.

In the main paper, $\mathcal{G}_a$ is instantiated as an affective graph over emotion categories and fine-grained valence--arousal prototypes, while $\mathcal{G}_p$ is instantiated as a reference-identity graph. 
This section evaluates whether the same formulation can be reused when the controlled attribute is changed. 
Specifically, we replace the affective graph with pose and lighting graphs, while keeping the architecture, factorization mechanism, graph-OT objective, and generator-conditioning design unchanged. 
This isolates the effect of changing the control graph and tests whether Controlla behaves as a graph-modular framework rather than an emotion-only method.

Table~\ref{tab:extensibility_mapping} summarizes how different control factors can be represented within the same abstraction. 
The experiments below focus on pose and lighting because they provide two complementary cases: pose is an ordered, discrete control factor, while lighting is a smoother visual attribute with more continuous variation.

\begin{table}[h]
\centering
\caption{
\textbf{Graph-modular control abstraction.}
Different control tasks can be instantiated by replacing the attribute graph while preserving the same factorized latent-control formulation.
}
\small
\rowcolors{2}{blue!5}{white}
\begin{tabular}{l|l|l}
\toprule
\rowcolor{blue!15}
\textbf{Control Factor} & \textbf{Attribute Graph $\mathcal{G}_a$} & \textbf{Preserved Graph $\mathcal{G}_p$} \\
\midrule
Emotion & Valence--arousal graph & Identity graph \\
Pose & Yaw-based pose graph & Identity graph \\
Lighting & Illumination direction graph & Identity / content graph \\
Style & Style embedding graph & Content graph \\
Viewpoint & Camera pose graph & Object graph \\
Scene Layout & Object relation graph & Scene graph \\
\bottomrule
\end{tabular}
\label{tab:extensibility_mapping}
\end{table}

\subsection{Plug-in Graph Evaluation}

We evaluate extensibility using a plug-in graph protocol. 
For each new control factor, only the attribute graph is replaced; the remaining model design is unchanged. 
This means that the same shared representation, identity--attribute factorization, graph-constrained OT alignment, and generator-conditioning adapters are used across emotion, pose, and lighting control. 
Therefore, differences in behavior can be attributed to the control graph and its associated supervision rather than to task-specific architectural changes.

This evaluation is intended to test controlled generalization of the formulation, not to claim that a single graph construction is optimal for every control factor. 
For each factor, the graph should encode the semantic neighborhood structure of that attribute: adjacent nodes should correspond to small semantic changes, while distant nodes should correspond to larger transformations. 
Under this assumption, graph-consistent traversal can be reused across different control dimensions.

\subsubsection{Pose Control via Plug-in Graphs}

For pose control, we replace the affective graph with an ordered yaw-state graph. 
Each node corresponds to a coarse head-pose state, and edges connect neighboring pose levels. 
The graph therefore encodes an ordinal left-to-right structure: moving from \emph{Strong Left} to \emph{Frontal} to \emph{Strong Right} should produce gradual pose changes rather than abrupt jumps.

Importantly, these pose nodes represent ordered pose categories, not calibrated 3D angles. 
Thus, the evaluation measures whether Controlla supports structured ordinal pose traversal while preserving identity, rather than whether it performs precise geometric head-pose estimation. 
This distinction keeps the claim conservative and avoids overstating the role of the pose graph.

\begin{table}[h]
\centering
\caption{
\textbf{Pose graph.}
The pose graph defines an ordinal left-to-right traversal structure over coarse yaw states.
}
\small
\rowcolors{2}{blue!4}{white}
\begin{tabular}{c|c|c}
\toprule
\rowcolor{blue!12}
\textbf{Node} & \textbf{Pose Level} & \textbf{Description} \\
\midrule
$P_1$ & Strong Left & Large leftward pose \\
$P_2$ & Left & Moderate leftward pose \\
$P_3$ & Slight Left & Small leftward pose \\
$P_4$ & Frontal & Center-facing pose \\
$P_5$ & Slight Right & Small rightward pose \\
$P_6$ & Right & Moderate rightward pose \\
$P_7$ & Strong Right & Large rightward pose \\
\bottomrule
\end{tabular}
\label{tab:pose_graph}
\end{table}

\subsubsection{Lighting Control via Plug-in Graphs}

For lighting control, we replace the affective graph with a graph over illumination direction. 
Each node corresponds to a coarse lighting condition, and edges connect semantically neighboring illumination states. 
Unlike pose, lighting is a smoother visual attribute: transitions between neighboring nodes may change shading and highlights without requiring large geometric deformation.

This setting tests whether Controlla can support graph-structured traversal over a non-affective visual factor. 
The preserved graph is defined over identity or content structure, so the intended behavior is to vary illumination while keeping identity and facial content stable.

\begin{table}[h]
\centering
\caption{
\textbf{Lighting graph.}
The lighting graph defines structured traversal over coarse illumination directions.
}
\small
\rowcolors{2}{blue!4}{white}
\begin{tabular}{c|c|c}
\toprule
\rowcolor{blue!12}
\textbf{Node} & \textbf{Lighting Condition} & \textbf{Description} \\
\midrule
$L_1$ & Left Light & Illumination from the left side \\
$L_2$ & Soft Left & Diffused left-side illumination \\
$L_3$ & Frontal Light & Approximately even frontal illumination \\
$L_4$ & Soft Right & Diffused right-side illumination \\
$L_5$ & Right Light & Illumination from the right side \\
\bottomrule
\end{tabular}
\label{tab:lighting_graph}
\end{table}

\subsection{Quantitative Results}

Table~\ref{tab:plugin_results} reports results for emotion, pose, and lighting under the same Controlla formulation. 
Emotion remains the strongest setting because it is the primary task studied in the main paper and benefits from the full AffectHuman-43K affective supervision. 
However, pose and lighting retain similar identity scores and only moderately higher GC values, indicating that the learned factorized-control structure remains stable when the attribute graph is replaced.

These results support the intended generalization claim: Controlla is not tied to the semantics of emotion labels alone. 
When a control factor can be represented as a graph over semantically ordered states, the same graph-constrained latent geometry objective can organize traversal for that factor. 
The slightly lower accuracy and higher GC for pose and lighting are expected because these tasks use simpler plug-in graphs and less specialized supervision than the primary affective setting.

\begin{table}[h]
\centering
\caption{
\textbf{Plug-in graph evaluation.}
Replacing the emotion graph with pose or lighting graphs preserves stable controllability, identity consistency, and graph-consistent traversal under the same Controlla framework.
}
\small
\rowcolors{2}{blue!4}{white}
\begin{tabular}{l|ccc}
\toprule
\rowcolor{blue!12}
\textbf{Control Factor} & Acc $\uparrow$ & ID $\uparrow$ & GC $\downarrow$ \\
\midrule
Emotion & \textbf{76.4} & \textbf{0.862} & \textbf{0.126} \\
Pose & 74.2 & 0.858 & 0.134 \\
Lighting & 73.3 & 0.856 & 0.139 \\
\bottomrule
\end{tabular}
\label{tab:plugin_results}
\end{table}

\subsection{Qualitative Analysis}

\paragraph{Pose control via plug-in graphs.}
Figure~\ref{fig:pose_extensibility} shows qualitative results using the pose graph from Table~\ref{tab:pose_graph}. 
The generated sequence follows the intended ordinal pose structure, moving gradually across left-facing, frontal, and right-facing states. 
Across the traversal, the same reference identity is largely preserved, indicating that the identity factor remains stable while the attribute factor changes. 
This supports the claim that Controlla can reuse graph-consistent traversal for non-affective attributes without modifying the architecture.

\begin{figure}[h]
\centering
\includegraphics[width=\linewidth]{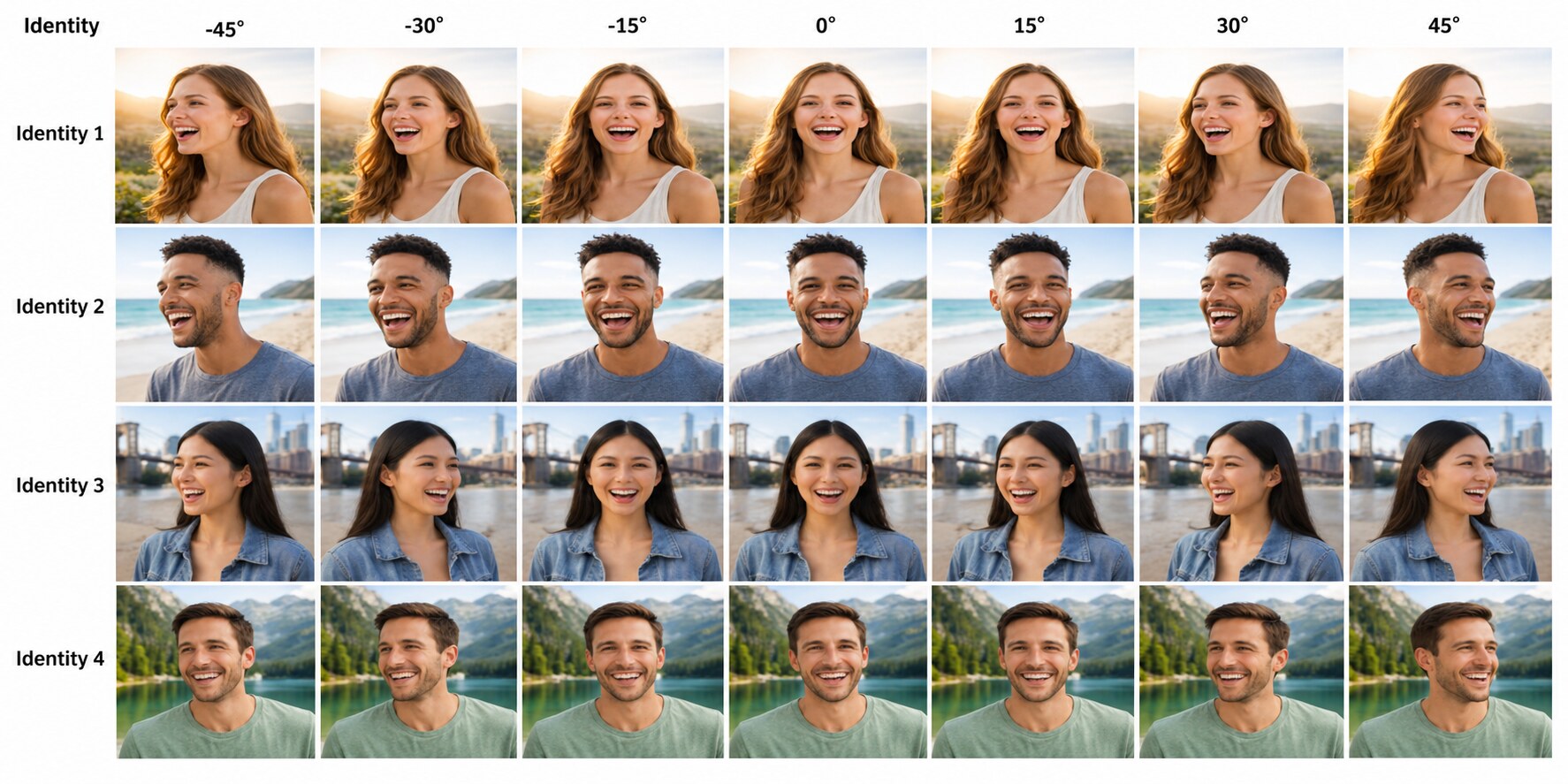}
\caption{
\textbf{Pose control.}
A plug-in pose graph enables ordered pose traversal while preserving reference identity.
}
\label{fig:pose_extensibility}
\end{figure}

\paragraph{Lighting control via plug-in graphs.}
Figure~\ref{fig:lighting_extensibility} presents qualitative results using the lighting graph from Table~\ref{tab:lighting_graph}. 
The outputs show gradual changes in illumination direction and shading while maintaining identity and facial structure. 
Compared with pose control, lighting traversal involves smaller geometric changes but clearer appearance-level variation. 
This demonstrates that the graph-latent formulation can support both discrete ordinal control factors and smoother visual attributes.

\begin{figure}[h]
\centering
\includegraphics[width=\linewidth]{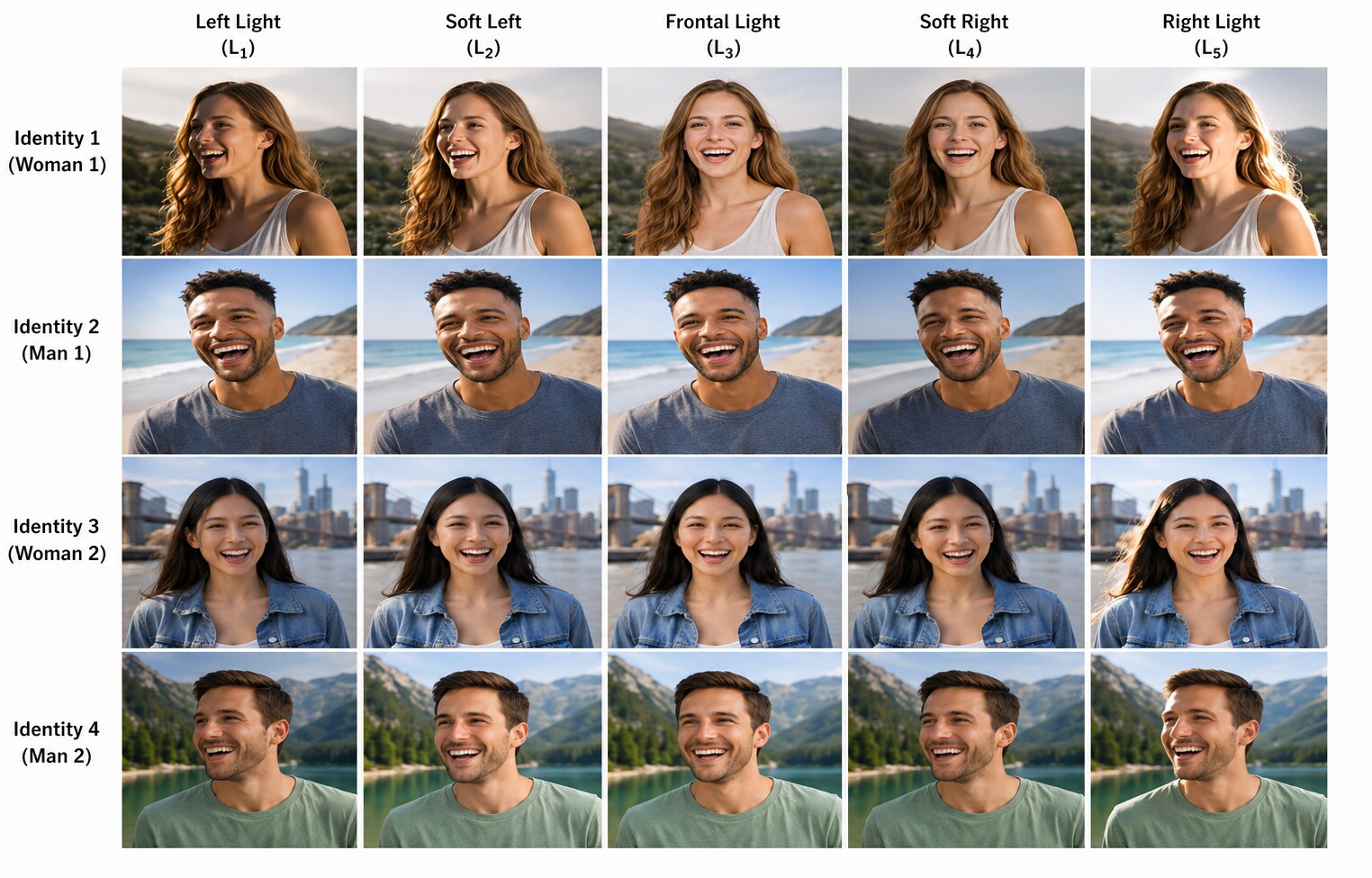}
\caption{
\textbf{Lighting control.}
A plug-in lighting graph enables structured illumination changes while preserving identity and content.
}
\label{fig:lighting_extensibility}
\end{figure}

\subsection{Scope of the Generalization Claim}

The plug-in graph results show that Controlla can be instantiated beyond affective control when the new control factor admits a meaningful graph structure. 
The claim is therefore not that a single universal graph solves all control tasks, but that the Controlla formulation provides a reusable mechanism: define an attribute graph, align the corresponding latent factor to that graph, and traverse the factor while preserving identity or content.

This scope is important for interpreting the results. 
Emotion, pose, and lighting differ in semantic type, supervision, and visual manifestation, yet all can be expressed as graph-defined control factors. 
The observed performance across these settings supports the graph-modular design of Controlla and suggests that future extensions can instantiate other controllable factors, such as style, viewpoint, object layout, or expression intensity, by defining suitable attribute and preserved-factor graphs.

\section{Robustness and Behavior Analysis}
\label{sec13}

We evaluate Controlla under controlled perturbations and challenging conditions to understand when structured latent geometry remains stable and when it begins to fail.
The goal is not to exhaustively cover all distribution shifts, but to isolate representative failure modes involving degraded visual identity cues, missing or noisy modalities, conflicting controls, and severe input mismatch.
These tests complement the main-paper results by examining whether Controlla's factorized representation preserves identity and graph-consistent traversal under imperfect inputs.

\vspace{-6pt}

\subsection{Robustness to Image Degradation}

We first evaluate robustness to common visual perturbations applied to the image/reference input, including Gaussian noise, motion blur, JPEG compression, and low-light degradation.
These perturbations primarily test the stability of the visual identity anchor and whether degraded visual evidence disrupts graph-consistent attribute traversal.

\begin{table}[h]
\centering
\caption{
\textbf{Robustness to image degradation.}
Performance decreases gradually under visual perturbations, while identity similarity remains comparatively stable.
}
\small
\rowcolors{2}{blue!5}{white}
\begin{tabular}{l|cccc}
\toprule
\rowcolor{blue!15}
\textbf{Perturbation} & Acc $\uparrow$ & TS $\uparrow$ & ID $\uparrow$ & GC $\downarrow$ \\
\midrule
Clean input & \textbf{76.4} & \textbf{0.71} & \textbf{0.862} & \textbf{0.126} \\
Gaussian noise & 74.8 & 0.69 & 0.851 & 0.137 \\
Motion blur & 74.1 & 0.68 & 0.846 & 0.142 \\
JPEG compression & 75.2 & 0.69 & 0.854 & 0.135 \\
Low light & 73.6 & 0.67 & 0.842 & 0.148 \\
\bottomrule
\end{tabular}
\label{tab:robust_image}
\end{table}

As shown in Table~\ref{tab:robust_image}, performance degrades smoothly rather than collapsing under moderate image corruption.
Low light and motion blur cause the largest drops, which is expected because they reduce the reliability of facial and identity-relevant visual cues.
However, ID remains comparatively stable across perturbations, indicating that the reference-grounded identity factor is not highly brittle under moderate visual degradation.
GC increases only moderately, suggesting that the attribute traversal remains graph-structured even when the visual input becomes noisier.
Thus, image degradation mainly affects endpoint accuracy and smoothness, while the learned factorization continues to preserve identity and traversal structure within a reasonable perturbation range.

\vspace{-6pt}

\subsection{Robustness to Multimodal Perturbations}

We next evaluate robustness when auxiliary modalities are missing, noisy, or semantically inconsistent.
This setting tests whether Controlla relies on a single modality for control or can integrate partially available multimodal evidence while keeping identity stable.

\begin{table}[h]
\centering
\caption{
\textbf{Robustness to multimodal perturbations.}
The model degrades gracefully when modalities are missing or noisy, while conflicting inputs primarily affect controllability.
}
\small
\rowcolors{2}{blue!4}{white}
\begin{tabular}{l|cccc}
\toprule
\rowcolor{blue!12}
\textbf{Conditioning} & Acc $\uparrow$ & TS $\uparrow$ & ID $\uparrow$ & GC $\downarrow$ \\
\midrule
Image + Text + Audio & \textbf{76.4} & \textbf{0.71} & 0.862 & \textbf{0.126} \\
Image + Text only & 74.2 & 0.66 & 0.860 & 0.139 \\
Image + Audio only & 71.7 & 0.59 & 0.857 & 0.154 \\
Noisy text & 73.8 & 0.65 & 0.858 & 0.145 \\
Noisy audio & 72.5 & 0.62 & 0.856 & 0.151 \\
Conflicting inputs & 70.9 & 0.58 & \textbf{0.864} & 0.166 \\
\bottomrule
\end{tabular}
\label{tab:robust_modalities}
\end{table}

Table~\ref{tab:robust_modalities} shows that removing or corrupting modalities primarily affects controllability, as reflected by reduced Acc and TS and increased GC.
Identity similarity remains relatively stable across all multimodal perturbations, which supports the intended separation between the identity factor and the attribute-control factor.
The drop from Image+Text to Image+Audio indicates that text provides a stronger and more direct affective control signal in this setting, while audio contributes complementary but noisier prosodic evidence.
The largest degradation occurs under conflicting inputs, where modalities provide inconsistent affective cues.
This suggests that Controlla can tolerate missing or noisy modalities, but does not fully resolve semantic conflicts when modalities disagree.
In such cases, the model tends to preserve identity while attribute control becomes less reliable.

\vspace{-6pt}

\subsection{Stress Tests and Failure Patterns}

Finally, we evaluate more severe stress-test conditions, including random emotion labels, mismatched audio, contradictory prompts, occluded faces, and extreme crops.
These cases are designed to expose the limits of the framework rather than represent normal operating conditions.

\begin{table}[h]
\centering
\caption{
\textbf{Stress-test evaluation.}
Performance drops under contradictory signals and severe corruption of the identity anchor, revealing limits of controllability and identity preservation.
}
\small
\rowcolors{2}{blue!4}{white}
\begin{tabular}{l|cccc}
\toprule
\rowcolor{blue!15}
\textbf{Condition} & Acc $\uparrow$ & TS $\uparrow$ & ID $\uparrow$ & GC $\downarrow$ \\
\midrule
Clean input & \textbf{76.4} & \textbf{0.71} & \textbf{0.862} & \textbf{0.126} \\
Random emotion label & 62.8 & 0.49 & 0.860 & 0.214 \\
Mismatched audio & 67.3 & 0.55 & 0.858 & 0.188 \\
Contradictory prompt & 68.1 & 0.56 & 0.857 & 0.181 \\
Occluded face & 69.4 & 0.57 & 0.831 & 0.176 \\
Extreme crop & 66.9 & 0.53 & 0.815 & 0.194 \\
\bottomrule
\end{tabular}
\label{tab:robust_stress}
\end{table}

As shown in Table~\ref{tab:robust_stress}, the strongest failures occur under two different conditions.
First, contradictory or mismatched control signals reduce Acc, TS, and GC because the attribute factor receives inconsistent semantic evidence.
Second, severe degradation of the identity anchor, such as occlusion or extreme cropping, reduces ID because the reference identity itself becomes unreliable.
This separation is important: random labels and mismatched controls mainly damage controllability, whereas occlusion and cropping directly damage identity preservation.
Overall, these stress tests show that Controlla is robust to moderate corruption and missing modalities, but remains limited when the requested control signal is semantically contradictory or when the visual identity anchor is severely corrupted.
These limitations are consistent with the framework: structured latent geometry can stabilize traversal, but it cannot fully recover information that is absent, contradictory, or visually destroyed.


\end{document}